\documentclass[11pt]{article}
\pdfoutput=1

\usepackage{arxiv}
\usepackage{amsmath,amssymb,bm}
\usepackage{amsthm}
\usepackage{graphicx}
\graphicspath{{images/}{./}}   
\usepackage{float}
\usepackage{ragged2e}
\usepackage[nodisplayskipstretch]{setspace}
\usepackage{setspace}
\usepackage{titlesec}
\usepackage[utf8]{inputenc}
\usepackage[T1]{fontenc}
\usepackage{xcolor}
\usepackage{hyperref}
\usepackage{url}
\usepackage{booktabs}
\usepackage{amsfonts}
\usepackage{nicefrac}
\numberwithin{equation}{section} 
\usepackage{microtype}
\usepackage{graphicx}
\usepackage[authoryear,round]{natbib}
\usepackage{doi}
\usepackage{amsmath}
\usepackage[english]{babel}
\usepackage{amssymb}
\usepackage{txfonts}
\usepackage{mathdots}
\usepackage{algorithm}
\usepackage{algpseudocode}
\usepackage{algorithmicx}
\usepackage{varwidth}
\usepackage{mathptmx}
\usepackage{caption}
\usepackage{geometry}
\usepackage{placeins}

\usepackage{subcaption}
\usepackage{float}
\usepackage{bm}

\usepackage{tikz}
\usetikzlibrary{shapes.geometric,arrows.meta,positioning,fit,backgrounds,calc,decorations.pathmorphing}
\tikzset{
  startstop/.style={rectangle,rounded corners,draw=black,fill=blue!8,
    minimum height=0.7cm,minimum width=1.4cm,align=center,inner sep=3pt},
  proc/.style={rectangle,draw=black,fill=gray!6,
    minimum height=0.7cm,minimum width=1.4cm,align=center,inner sep=3pt},
  io/.style={trapezium,trapezium left angle=70,trapezium right angle=110,
    draw=black,fill=green!8,minimum height=0.7cm,align=center,inner sep=3pt},
  dec/.style={diamond,aspect=2,draw=black,fill=orange!12,
    align=center,inner sep=1pt},
  arrow/.style={-{Latex[length=2mm]},thick},
}
\usepackage{stmaryrd}          
\providecommand{\dd}{\mathrm{d}}
\definecolor{RefBlue}{RGB}{0,84,166}     
\definecolor{EqRed}{RGB}{176,26,26}      
\definecolor{ThmGreen}{RGB}{0,112,60}    
\hypersetup{
  colorlinks=true,
  citecolor=RefBlue,
  urlcolor=RefBlue,
  filecolor=RefBlue,
  linkcolor=RefBlue,
  linktocpage=false,
  breaklinks=true,
  bookmarksnumbered=true,
  pdftitle={Autoregressive Physics-Informed Neural-Operator Surrogates and Joint
    4-D Seismic-Production Bayesian History Matching}
}
\DeclareRobustCommand{\eqr}[1]{{\hypersetup{linkcolor=EqRed}\eqref{#1}}}

\DeclareRobustCommand{\thmr}[1]{{\hypersetup{linkcolor=ThmGreen}\ref{#1}}}

\usepackage{amsmath,amssymb}
\newcommand{\norm}[1]{\left\lVert #1 \right\rVert}
\usepackage{amsmath,amssymb}
\newcommand{\E}{\mathbb{E}}

\usepackage[nodisplayskipstretch]{setspace}
\usepackage{setspace}
\usepackage{ragged2e}
\usepackage{titlesec}

\newtheorem{theorem}{Theorem}[section]

\newtheorem{proposition}[theorem]{Proposition}
\newtheorem{remark}[theorem]{Remark}

\titlespacing{\section}{0.5pt}{\parskip}{-\parskip}
\titlespacing{\subsection}{0.5pt}{\parskip}{-\parskip}
\titlespacing{\subsubsection}{0.5pt}{\parskip}{-\parskip}

\newcommand{\cA}{\mathcal{A}}
\newcommand{\cN}{\mathcal{N}}
\newcommand{\cL}{\mathcal{L}}
\newcommand{\cO}{\mathcal{O}}

\newcommand{\cB}{\mathcal{B}}
\newcommand{\cF}{\mathcal{F}}
\newcommand{\cU}{\mathcal{U}}

\newcommand{\cK}{\mathcal{K}}
\newcommand{\cG}{\mathcal{G}}
\newcommand{\cX}{\mathcal{X}}
\newcommand{\cD}{\mathcal{D}}
\newcommand{\cW}{\mathcal{W}}
\newcommand{\cM}{\mathcal{M}}
\newcommand{\cH}{\mathcal{H}}
\newcommand{\cR}{\mathcal{R}}
\newcommand{\cT}{\mathcal{T}}

\newcommand{\R} {\mathbb{R}}

\newcommand{\cC}{\mathcal{C}}

\newcommand{\grad}{\nabla}

\titlespacing{\section}{0.5pt}{\parskip}{-\parskip}
\titlespacing{\subsection}{0.5pt}{\parskip}{-\parskip}
\titlespacing{\subsubsection}{0.5pt}{\parskip}{-\parskip}

\title{\vspace{1.0cm}%
Autoregressive Physics-Informed Neural-Operator Surrogates and
Joint 4-D Seismic--Production Bayesian History Matching}

\author{%
  Clement Etienam$^{1,}$\thanks{%
    $^{1}$NVIDIA Corporation.\quad
    $^{2}$Shell Information Technology International B.V.\quad
    $^{3}$Shell India Markets Private Limited.\\
    $^{4}$Shell Information Technology International Inc.\quad
    $^{5}$Shell Global Solutions International B.V.\\
    $^{6}$Shell International Exploration and Production Inc.\quad
    $^{*}$\,\textbf{Corresponding author.}}
  \And David Blom$^{2}$
  \And Hari Pattathil$^{3}$
  \And Tsubasa Onishi$^{1}$
  \And Harpreet Sethi$^{1}$
  \And Lalson Vincent$^{3}$
  \And Suraj Pawar$^{4}$
  \And Kishan Ramesh Kumar$^{5}$
  \And Pandu Devarakota$^{4}$
  \And Faruk O. Alpak$^{6}$
}

\date{}
\renewcommand{\headeright}{}
\renewcommand{\undertitle}{}
\newtheorem{assumption}{Assumption}[section]
\begin{document}
\maketitle

\begin{abstract}
We present a framework for surrogate-accelerated Bayesian history matching and
demonstrate it on a realistic synthetic model, here referred to as the
\textbf{Tau reservoir} ($46\!\times\!99\!\times\!22$, $N=100{,}188$ active cells;
two producers and one water injector). The forward operator
is replaced by a multi-head Fourier Neural Operator trained in
\emph{autoregressive} (closed-loop) mode with a discrete physics residual
(PINO), $K$-step truncated backpropagation through time ($K{=}10$, $T{=}20$),
and Gaussian covariate-shift injection, distributed over four GB200 GPUs by
PyTorch DistributedDataParallel (DDP) with NVIDIA Collective Communications
Library (NCCL) all-reduce. The inverse engine
assimilates production data \emph{and} 4-D seismic through an adaptive
regularised Ensemble Kalman Inversion (aREKI), with the seismic observable
synthesised inside the surrogate rollout by a closed-form, GPU-batched Gassmann
petro-elastic transform and compressed as a raw
$|\Delta V_p|$ map or a Parseval-exact truncated 2-D discrete cosine transform
(DCT). The petrophysical properties $(\log K,\phi)$ may
be inverted in grid space or in a reduced subspace via DCT or a variational
convolutional auto-encoder (VCAE), with optional Furrer covariance
localisation. A two-stage schedule (production-data aREKI, inter-phase
inflation $\gamma{=}1.01$, then a joint production- and seismic-data ensemble
smoother with multiple data assimilation, ES-MDA) is selected by a
dimensionless two-modality score. We develop the full theory: the Bayesian
posterior as a Radon--Nikodym density on a Borel $\sigma$-algebra over a
function space; the autoregressive exposure-bias bound and its PINO-induced
contraction; the discretised black-oil system residual underlying the physics
loss; the Peaceman completion model; the petro-elastic forward map; and an
elementary Löwner-order argument for why, in the linearised Gaussian regime,
joint seismic--production assimilation is expected to be at least as
informative as production-only matching, with the effect demonstrated
empirically ($s_{\mathrm{P2}}{=}0.876$ vs.\ $s_{\mathrm{P1}}{=}2.556$). On a held-out test
set the surrogate attains $R^2>0.999$ on the pressure and saturation fields and
$R^2\!\gtrsim\!0.9999$ on WOPR, and rolls out in $0.579$~s, a $\sim\!2500\times$
per-call speed-up (single digits end-to-end once the training and reference-run
budget is included, \S\ref{sec:results}) that brings a 190-member,
10-iteration joint inversion within one GPU-hour.
\end{abstract}


\newpage

\section{Introduction}
\label{sec:intro}

\subsection{Engineering application}
\label{sec:application}
Reservoir management rests on a recurring computation: conditioning a
high-dimensional reservoir model on observed data so that forecasts of pressure,
saturation and well rates carry calibrated uncertainty. Mathematically this is the
\emph{inverse problem} of inferring the spatial distribution of the petrophysical
properties, absolute permeability $K(x)$ and porosity $\phi(x)$, on a
three-dimensional reservoir domain $\Omega$, from sparse, noisy and indirect
measurements of the flow they induce \citep{oliver2008inverse}.

We study this problem on a realistic synthetic reservoir model, referred to in
this paper as the Tau reservoir model (see Reservoir~A in
\citet{alpak2023lbfgs}). The Tau model represents a fault-dissected clastic
reservoir stratigraphy; it is produced by two wells, \texttt{Tau\_1} and
\texttt{Tau\_2}, supported by one water injector, \texttt{Inj1}, and it is
discretised on its native model grid of $46\times99\times22$ cells
($N=100{,}188$ active) over $T=20$ report steps. The unknown therefore has
dimension $2N\approx2\times10^5$, whereas the production record has dimension
$\mathcal{O}(10^2)$. The problem is severely under-determined, and a Bayesian
treatment, which delivers a posterior \emph{distribution} rather than a single
calibrated model, is the appropriate one.

Two costs have historically made rigorous Bayesian history matching intractable
at this scale. The first is the \emph{forward cost}: each evaluation of the
parameter-to-data map $\cG$ requires a fully implicit finite-volume (FV) solve of
the coupled black-oil equations, taking $\mathcal{O}(4\text{--}5)$ minutes on the
native model grid, and an ensemble inversion needs
$N_e\times N_{\mathrm{iter}}\approx1900$ such solves. The second is the
\emph{information cost}: production data sampled at three wells cannot, even
without noise, resolve the inter-well permeability connectivity that governs
sweep and breakthrough. The work described in this paper attacks both types of
costs. The FV operator is replaced by an \emph{autoregressive physics-informed
neural operator} (PINO) surrogate that evaluates in $0.579$~s, and the
observation is enriched with a spatially distributed \emph{4-D seismic} channel
synthesised inside the surrogate by a petro-elastic transform.

\subsection{Related work}
\label{sec:related}

The contribution of this paper sits at the confluence of four research
threads that have, until recently, been pursued largely in isolation:
(i) operator learning and physics-informed surrogates for parametric PDEs;
(ii) the stabilisation of \emph{autoregressive} learned simulators against
exposure bias; (iii) the function-space Bayesian theory of inverse problems and
its ensemble realisations; and (iv) the use of 4-D (time-lapse)
seismic data for reservoir characterisation, forecasting and development
optimisation. We review each in turn, with an emphasis
on the most recent literature, and use the review to motivate the specific
design choices of our pipeline.

\subsubsection{From response surfaces to neural operators}
\label{sec:rw-operators}
Classical surrogate reservoir modelling replaced the full-physics simulator
with cheap statistical emulators: polynomial and radial-basis response
surfaces, Kriging/Gaussian-process proxies, and reduced-order models built by
proper orthogonal decomposition or trajectory piecewise-linearisation
\citep{oliver2008inverse}. These emulators are accurate only near the design
ensemble on which they are fitted, degrade combinatorially with parameter
dimension, and are tied to a fixed grid. The deep-learning era first replaced
them with convolutional encoder--decoders and physics-informed neural networks
(PINNs) \citep{raissi2019physics}, which learn a \emph{single} solution field
but remain resolution-locked: a network trained at one discretisation cannot be
queried at another without retraining.

The decisive conceptual shift was \emph{operator learning}: instead of
approximating a function $u:\Omega\to\R$, one approximates the solution
\emph{operator} $\cG^\dagger:a\mapsto u$ between infinite-dimensional Banach
spaces of coefficients and states. DeepONet \citep{lu2021deeponet} realised
this through a branch--trunk factorisation with a universal-approximation
guarantee for operators, while the Fourier Neural Operator (FNO)
\citep{li2020fourier,kovachki2021neural} parametrised the operator by global
spectral convolutions, achieving discretisation invariance and quasi-linear
$\cO(N\log N)$ cost. The approximation theory has since matured: quantitative
error bounds for DeepONets and FNOs \citep{lanthaler2022error,
kovachki2021neural} establish algebraic convergence in the spectral-mode count and
parameter budget for sufficiently regular operators, a result we invoke in
Theorem~\thmr{thm:approx}. A rapidly growing family of variants extends the
construction to complex geometries and multi-scale physics, Geo-FNO for
deformed domains \citep{li2023geofno}, the U-shaped UNO
\citep{rahman2023uno}, factorised and tensorised FNOs for three-dimensional
fields \citep{tran2023factorized}, wavelet neural operators
\citep{tripura2023wavelet}, and transformer-based operators
\citep{cao2021choose,hao2023gnot}, and operator learning has been surveyed
comprehensively for scientific machine learning
\citep{azizzadenesheli2024neuralop}. In the subsurface domain specifically,
FNO and its physics-informed extensions have been applied to two-phase flow,
$\text{CO}_2$ storage and multiphase reservoir dynamics
\citep{wen2022ufno,wen2023nested,zhang2022fourier,chandra2025co2}, consistently
reporting two-to-three orders of magnitude speed-up over full-physics solvers;
in particular, \citet{chandra2025co2} demonstrate FNO surrogates for
$\text{CO}_2$ plume migration on realistic geologies at $\mathcal{O}(10^5)$
speed-up. Our
composite eleven-head FNO (\S\ref{sec:forward}) is a member of this family,
specialised to the coupled black-oil state and the multi-layer Peaceman well
observables.

\paragraph{Physics-informed operators.}
Pure data-driven operators can violate conservation over long horizons. The
physics-informed neural operator (PINO) \citep{li2021physics} augments the data
loss with a PDE residual evaluated on the operator's own output, combining the
generalisation of operator learning with the constraint satisfaction of PINNs.
When the residual is taken in its \emph{discretised}
(finite-volume/finite-difference) form, the penalty costs one $\cO(N)$ assembly
per step and requires no symbolic differentiation of the PDE, the economy we
exploit in \eqr{eq:pino} by reusing the residual of the simulator's own
discretised black-oil system.

\paragraph{The operator-learning problem, formally.}
Stating the target precisely frames the approximation theory invoked later.
Given a probability measure $\mu$ on the input space
$\cA$, operator learning seeks parameters $\theta$ minimising the population risk
\begin{equation}
\cR(\theta)=\E_{a\sim\mu}\,\norm{\cG_\theta(a)-\cG^\dagger(a)}_{\cU}^2,
\label{eq:oprisk}
\end{equation}
approximated in practice by its empirical average over the training ensemble.
The universal-approximation theorems for DeepONet \citep{lu2021deeponet} and for
neural operators \citep{kovachki2021neural} guarantee that for every
$\varepsilon>0$ and every compact $\cK\subset\cA$ there exists a network with
$\sup_{a\in\cK}\norm{\cG_\theta(a)-\cG^\dagger(a)}_{\cU}<\varepsilon$; the
quantitative refinements \citep{lanthaler2022error} sharpen this into algebraic
rates in the spectral-mode count and parameter budget (our
Theorem~\thmr{thm:approx}), at the price of a curse-of-dimensionality constant
that physics-informed training tempers by injecting the PDE as an inductive bias
\citep{wang2021theory}. In the subsurface setting these ideas were carried into
data assimilation by \citet{tang2020deep}, whose deep convolutional surrogate
conditions the dynamic flow response on the geological realisation for
ensemble-based inversion; the most recent work even pretrains a single operator
across multiple governing equations \citep{mccabe2023multiple}. Our surrogate
departs from all of these in being \emph{autoregressive and physics-constrained
end-to-end through the rollout}, a $T$-step map rather than a one-shot
regressor, which is exactly the object the ensemble inversion of
\S\ref{sec:inverse} evaluates $N_e\!\times\!N_a$ times.

\subsubsection{Autoregressive surrogates and exposure bias}
\label{sec:rw-autoreg}
A surrogate for a time-dependent PDE must ultimately be rolled out in closed
loop: its own prediction at step $n$ becomes the input at step $n{+}1$. Trained
by one-step \emph{teacher forcing}, the operator only ever sees ground-truth
inputs, so at deployment it is queried off the training distribution, the
\emph{exposure bias} or covariate-shift problem long known in sequence modelling
\citep{bengio2015scheduled,ranzato2016sequence} and imitation learning, where
DAgger-style dataset aggregation \citep{ross2011dagger,venkatraman2015improving}
corrects the train/test mismatch by training on states the policy actually
visits. In learned physics simulation the same pathology manifests as error
accumulation and eventual blow-up of long rollouts, and a range of remedies has
emerged: the noise-injection training of graph-network simulators
\citep{sanchez2020learning}, the message-passing PDE solvers of
\citet{brandstetter2022message} with their explicit push-forward trick,
temporal-bundling and the Markov neural operator with dissipativity constraints
\citep{li2022markov}, and stability analyses of autoregressive neural emulators
in weather and climate \citep{pathak2022fourcastnet,keisler2022forecasting}.
The training of recurrent maps by truncated backpropagation through time (TBPTT)
and its bias/variance trade-off are classical \citep{williams1990efficient,
pascanu2013difficulty}, with unbiasing schemes analysed by
\citet{tallec2017unbiased}; we use these results to certify the near-unbiasedness
of our $K$-step truncation in Proposition~\thmr{prop:tbptt}.

\paragraph{Controlling the long-horizon Jacobian.}
The stabilisation of long rollouts has since become a research theme in its own
right, and the proposed remedies are unified by a single principle: control the
one-step Jacobian norm $\E\norm{D_x\cG_\theta}$ that governs the geometric error
amplification of Proposition~\thmr{prop:rollout}. Mesh-based graph-network
simulators \citep{pfaff2021learning} and learned coarse models
\citep{stachenfeld2022learned} inject training noise so the network learns to
contract its own perturbations back toward the data manifold;
\citet{lippe2023pderefiner} instead treat each step as a short denoising
diffusion (PDE-Refiner), explicitly suppressing the high-wavenumber error that
seeds blow-up; and dissipativity-constrained Markov operators
\citep{li2022markov} bake a contractive attractor into the architecture. Our
combination of closed-loop (autoregressive) training, covariate-shift noise
injection~\eqr{eq:noise}, and a discrete physics residual targets the same
Jacobian (Proposition~\thmr{prop:arstab}), with the additional guarantee,
absent from the purely data-driven schemes, that the PINO penalty keeps the
operator on the discrete conservation manifold $\{R(x,\cdot)=0\}$
(Theorem~\thmr{thm:contract}). To the extent that the penalty is driven towards
zero, the resulting contraction is inherited from the dissipative (monotone,
upwind) finite-volume discretisation rather than being a purely learned
property; the argument is the proof of Theorem~\thmr{thm:contract}, which rests on
the M-matrix structure of the two-point-flux pressure update and on the
diffusivity hierarchy of \S\ref{sec:pvt}. Because the penalty is not driven
exactly to zero in practice, this is a motivating rather than a certified
property of the trained operator (\S\ref{sec:limitations}).

Our own prior work developed this thread for reservoir simulation: the
cluster-classify-regress construction for learning discontinuous maps
\citep{bernholdt2019ccr} and its deep mixture-of-experts form
\citep{etienam2024gpexperts}, the coupled mixture-of-experts / PINO
history-matching forward model on the Norne field
\citep{etienam2024norne,etienam2024ecmor}, and most recently the
\emph{sequential} (autoregressive) PINO surrogate
\citep{etienam2025seqpino,etienam2026seqecmor}, which trains the operator in closed loop with a
discrete physics residual and demonstrates stable long-horizon rollouts for
Norne history matching. The present paper builds directly on
\citep{etienam2025seqpino,etienam2026seqecmor}: we retain its autoregressive-PINO training
principle, rollout with truncated BPTT, discrete black-oil residual, and
covariate-shift noise injection, and extend it in three directions: a
multi-head operator that also emits the Peaceman well responses, a
petro-elastic 4-D seismic channel synthesised inside the rollout, and a
joint two-stage seismic--production inversion, all demonstrated on the
Tau reservoir model. Sections~\ref{sec:forward}--\ref{sec:inverse}
make this inheritance precise, referencing the sequential-PINO construction at
each step.

\subsubsection{Bayesian inverse problems and ensemble data assimilation}
\label{sec:rw-bayes}
History matching is a Bayesian inverse problem: recover high-dimensional
petrophysical properties from sparse, noisy, indirect data. The mathematically
rigorous formulation is the \emph{function-space} Bayesian framework of
\citet{stuart2010inverse} and \citet{dashti2017bayesian}, in which the posterior
is defined as a measure absolutely continuous with respect to a Gaussian prior,
with a Radon--Nikodym density given by the data-misfit potential, and is shown
to be well-posed (Lipschitz in the Hellinger metric) under mild continuity and
integrability assumptions. This is the framework we adopt and expand in
\S\ref{sec:bayes}, following Stuart's development closely because it is the only
one that survives the mesh-refinement limit $d\to\infty$ intrinsic to a
function-valued unknown. The requisite Gaussian-measure theory, Cameron--Martin
spaces, absolute continuity, and the Feldman--H\'ajek dichotomy, is standard
\citep{bogachev1998gaussian,daprato2014stochastic}.

Because sampling the posterior in $d\!\sim\!10^5$ dimensions by MCMC is
prohibitive even with dimension-robust samplers such as preconditioned Crank--
Nicolson \citep{cotter2013mcmc}, practitioners use ensemble Kalman methods:
ensemble Kalman inversion (EKI) \citep{iglesias2013ensemble}, its
regularising and adaptive variants \citep{iglesias2016regularising,
iglesias2021adaptive}, and the
ensemble smoother with multiple data assimilation (ES-MDA)
\citep{emerick2013esmda}. Their mathematical foundation, the continuous-time
mean-field limit, the preconditioned-gradient-flow structure, and the
subspace property, was established by \citet{schillings2017convergence,
schillings2018convergence} and extended to a Tikhonov--Kalman variant with
provable regularisation \citep{chada2020tikhonov}; the deterministic
square-root formulations and their well-known ensemble-collapse behaviour are
analysed by \citet{blomker2019strongly}. These methods are derivative-free,
embarrassingly parallel, and require only ensemble forward evaluations, a
natural match for a fast surrogate. We use the adaptive regularised EKI
(aREKI) of \citet{iglesias2021adaptive}, which we have previously applied to
surrogate-accelerated history matching \citep{etienam2024modulus,etienam2024norne}, in the
production-data inversion phase and ES-MDA in the joint production- and
seismic-data inversion phase, and \S\ref{sec:inverse} makes
the reason for switching algorithms precise.

\paragraph{Iterative smoothers, transport, and the Gaussian ansatz.}
Two lineages sharpen this picture. Iterative ensemble smoothers
\citep{emerick2013esmda,evensen2018analysis} recast history matching as a
sequence of tempered Gauss--Newton steps, the view underlying our ES-MDA phase
and its exact-tempering identity~\eqr{eq:esmdaexact}; they inherit the
robustness of a gradually annealed likelihood while remaining derivative-free.
A complementary, fully non-Gaussian lineage replaces the Kalman/Gaussian ansatz
by \emph{measure transport} \citep{marzouk2016transport,reich2015probabilistic},
seeking a deterministic map $T_\sharp\mu_0=\mu^y$ that pushes the prior forward
onto the posterior. Ensemble Kalman methods are exactly the affine special case
of such a transport, and the ensemble nonlinearity (Jensen) bias
$\delta_{\mathrm{NL}}$ we isolate in~\eqr{eq:jensen} is precisely the error
incurred by that affine restriction. We adopt the Kalman route here for its
derivative-free scalability to $N_e{=}190$ full-state rollouts, and identify the
transport/score-based generalisation (\S\ref{sec:discussion}) as the principled
way to remove the residual Gaussianity assumption without abandoning the
surrogate.

\paragraph{Learned parametrisations and generative priors.}
A finite ensemble of $N_e\ll d$ members yields a rank-deficient sample
covariance, and a purely Gaussian prior cannot represent complex geological
structure, for example connected high-permeability geobodies or sharp facies
contacts.
Two complementary remedies appear in the literature: covariance \emph{localisation}, which tapers spurious long-range
correlations \citep{gaspari1999,furrer2007}; and dimension reduction into a
geologically meaningful latent space. The latter has moved from linear
truncated bases (DCT, PCA/KL) to deep generative priors, variational
auto-encoders and GANs that learn the manifold of plausible geologies
\citep{laloy2017vae,laloy2018gan,mosser2020stochastic,canchumuni2019esmda},
so that the Kalman update in latent coordinates corresponds to a strongly
non-Gaussian, prior-consistent move in grid space. Our DCT and variational
convolutional auto-encoder (VCAE) charts (\S\ref{sec:param}) instantiate both
ends of this spectrum. Most recently, score-based diffusion models have been
proposed as expressive priors and even as full posterior samplers for inverse
problems \citep{song2021sde,chung2023dps,rozet2023scoreda,baldassari2024conditional},
a direction we return to as future work in \S\ref{sec:discussion}.

\subsubsection{4-D seismic history matching}
\label{sec:rw-seismic}
Production data sampled at a handful of wells cannot resolve inter-well
connectivity. Time-lapse (4-D) seismic can image the spatial change in elastic
properties caused by production under favourable conditions of rock stiffness,
fluid contrast and survey repeatability, and can then supply the areal
information that fills this null space \citep{landro2001discrimination,
lumley2001timelapse}. Its value is nonetheless case dependent: time-lapse
signals can be weak, ambiguous, or in apparent conflict with production data.
Assimilating seismic requires a petro-elastic forward model mapping reservoir state $(P,S)$ to seismic attributes: Gassmann
fluid substitution \citep{gassmann1951,smith2003gassmann} for the saturated
moduli, Batzle--Wang correlations \citep{batzle1992} for the fluid properties, and a
pressure-sensitivity law for the dry frame \citep{macbeth2004}. A substantial
literature couples such petro-elastic models to ensemble assimilation of
seismic-derived attributes, impedance, saturation/pressure changes, or raw
amplitude maps \citep{emerick2016review,fahimuddin2010,trani2013seismic,
luo2018seismic}, and consistently reports improved inter-well resolution over
production-only matching. What distinguishes the present work is that the
petro-elastic map is implemented as a closed-form, GPU-batched
transform of the surrogate's \emph{own} predicted state, so the seismic
observable is synthesised inside the autoregressive rollout at negligible
marginal cost, and the expected information gain is motivated by a linearised
Fisher-information (L\"owner-order) argument (\S\ref{sec:whyjoint}).

\subsection{Research gap and approach}
\label{sec:rw-position}
Learned surrogates and time-lapse seismic have already been brought together in
several settings. \citet{etienam2019seismic} combined 4-D seismic history matching
with unsupervised learning; \citet{millevoi2025seismic} developed a deep-learning
surrogate for seismic data assimilation in fault-activation modelling; and
\citet{wang2024co2} built a deep-learning history-matching framework for
CO$_2$ storage that assimilates 4-D seismic and monitoring-well data within
ES-MDA. Ensemble assimilation of seismic-derived attributes with full-physics
simulators is well established
\citep{emerick2016review,fahimuddin2010,trani2013seismic,luo2018seismic}, as are
deep-learning surrogates for subsurface data assimilation \citep{tang2020deep}.
We therefore claim no priority for combining surrogates with seismic data. What
we contribute is a combination that, to our knowledge, has not been assembled in
a single framework for a three-phase hydrocarbon reservoir with a long production
history: (i) a surrogate that remains stable when rolled out autoregressively over
the entire history; (ii) a petro-elastic channel evaluated from the surrogate's
\emph{own} rolled-out state, so that production and seismic are two consistent
observables of one forward model; (iii) well responses predicted by the same
operator through the Peaceman completion model; and (iv) an inversion whose
ensemble updates are traceable to a well-posed function-space posterior. The
framework builds on our sequential-PINO training \citep{etienam2025seqpino,etienam2026seqecmor} and
on our coupled-surrogate history-matching studies
\citep{etienam2024norne,etienam2024ecmor}, and it rests on two ideas that are
usually studied separately.

\paragraph{Autoregressive operator learning.}
The forward map is a discrete-time dynamical system
$x_{n+1}=\cF(x_n;u,c_n,\tau_n)$ on the state $x_n=(p_n,S_{w,n},S_{o,n})$. A
surrogate trained by one-step \emph{teacher forcing} sees only true states at its
input; deployed in closed loop it is driven by its own (erroneous) predictions,
whose distribution drifts away from the training distribution. This
\emph{covariate shift} (exposure bias) is the dominant failure mode of learned
simulators. We train the operator the way it will be used, feeding back its own
output through all $T$ steps, and we make the training stable with $K$-step
truncated backpropagation through time, a discrete physics residual, and noise
injection. Section~\ref{sec:forward} develops the theory and the closed-loop
training procedure that realises it.

\paragraph{Joint multi-physics Bayesian inversion.}
Because the surrogate rolls out the full three-dimensional state, the predicted
pressure and oil-saturation cubes can be pushed through a closed-form Gassmann
fluid-substitution map \citep{gassmann1951,smith2003gassmann,mavko2009rock} to a
4-D seismic velocity-change field $\Delta V_p(x,y)$. Seismic thereby becomes a
\emph{first-class observable} on the same footing as production, and the two are
assimilated jointly by ensemble Kalman inversion in a stacked Mahalanobis cost
\citep{mahalanobis1936,oliver2008inverse}.
Sections~\ref{sec:bayes}--\ref{sec:inverse} build the inversion from its
measure-theoretic foundations up to the practical two-stage schedule.

Figure~\ref{fig:cartoon} summarises the resulting workflow, from the reservoir
model and its full-physics ensemble, through the autoregressive surrogate and the
petro-elastic seismic channel, to the Bayesian inversion that returns a posterior
reservoir model. The full space--time evolution of the reservoir state, and the
surrogate's fidelity to the finite-volume reference across it, are shown in
Figs.~\ref{fig:dyn}--\ref{fig:evo1}; the software orchestration is described in
\S\ref{sec:hpc}.

\begin{figure}[t]\centering
\resizebox{\textwidth}{!}{%
\begin{tikzpicture}[font=\sffamily,>=Latex,line join=round]
\begin{scope}[shift={(0,0)}]
  \draw[thick,fill=orange!18] (0,0) rectangle (3.1,2.1);
  \foreach \yy in {0.42,0.84,1.26,1.68} \draw[gray!55] (0,\yy)--(3.1,\yy);
  \draw[red!70,very thick] (1.7,0)--(2.05,2.1);
  \draw[thick] (0.65,2.55)--(0.65,0.15); \fill (0.65,0.15) circle(1.6pt);
  \draw[thick] (2.55,2.55)--(2.55,0.15); \fill (2.55,0.15) circle(1.6pt);
  \draw[very thick,blue!70] (1.25,2.55)--(1.25,0.15);
  \node[font=\scriptsize] at (0.65,2.78){Tau\_1};
  \node[font=\scriptsize] at (2.55,2.78){Tau\_2};
  \node[font=\scriptsize,blue!70] at (1.25,2.80){Inj1};
  \node[align=center,font=\bfseries\small] at (1.55,-0.5){Reservoir model};
  \node[align=center,font=\scriptsize] at (1.55,-1.0){synthetic, fault-dissected\\$(\log K,\phi)$ \emph{unknown}};
\end{scope}
\begin{scope}[shift={(4.9,0)}]
  \foreach \xx in {0,...,4} \draw[gray!50] (\xx*0.62,0)--(\xx*0.62,2.1);
  \foreach \yy in {0,...,3} \draw[gray!50] (0,\yy*0.7)--(2.48,\yy*0.7);
  \draw[thick] (0,0) rectangle (2.48,2.1);
  \node[align=center,font=\scriptsize] at (1.24,1.05){$\partial_t\!\big(\tfrac{\varphi S_\alpha}{B_\alpha}\big)$\\[1pt]$=\nabla\!\cdot(T_\alpha\nabla P)+Q_\alpha$};
  \node[align=center,font=\bfseries\small] at (1.24,-0.5){Full-physics FV};
  \node[align=center,font=\scriptsize] at (1.24,-1.0){$200$ simulator runs\\$\sim$5\,min each};
\end{scope}
\begin{scope}[shift={(9.3,0)}]
  \draw[thick,fill=blue!8,rounded corners] (-0.1,-0.05) rectangle (2.75,2.15);
  \foreach \i in {0,1,2} \fill[blue!60] (0.4,0.45+\i*0.6) circle(2.3pt);
  \foreach \i in {0,1,2,3} \fill[blue!45] (1.35,0.3+\i*0.5) circle(2.3pt);
  \foreach \i in {0,1,2} \fill[blue!60] (2.3,0.45+\i*0.6) circle(2.3pt);
  \foreach \a in {0,1,2}\foreach \b in {0,1,2,3} \draw[blue!22](0.4,0.45+\a*0.6)--(1.35,0.3+\b*0.5);
  \foreach \a in {0,1,2,3}\foreach \b in {0,1,2} \draw[blue!22](1.35,0.3+\a*0.5)--(2.3,0.45+\b*0.6);
  \draw[->,thick,red!70] (2.3,2.02) to[out=95,in=85,looseness=1.7] (0.4,2.02);
  \node[font=\scriptsize,red!70] at (1.35,2.5){$x_n\!\to\!x_{n+1}$};
  \node[align=center,font=\bfseries\small] at (1.35,-0.5){AR-PINO (FNO)};
  \node[align=center,font=\scriptsize] at (1.35,-1.0){closed-loop, $K$-TBPTT\\$+$ physics residual};
\end{scope}
\begin{scope}[shift={(13.6,0)}]
  \draw[thick,fill=green!8,rounded corners](0,0) rectangle (2.5,2.1);
  \draw[green!55!black,thick] plot[domain=0.22:2.28,samples=48]({\x},{1.58+0.12*sin(\x*540)});
  \draw[green!55!black,thick] plot[domain=0.22:2.28,samples=48]({\x},{1.0+0.17*sin(\x*540)});
  \draw[green!55!black,thick] plot[domain=0.22:2.28,samples=48]({\x},{0.42+0.10*sin(\x*540)});
  \node[align=center,font=\bfseries\small] at (1.25,-0.5){4-D seismic};
  \node[align=center,font=\scriptsize] at (1.25,-1.0){Gassmann $\Delta V_p$\\(closed form)};
\end{scope}
\begin{scope}[shift={(17.6,0)}]
  \draw[thick,fill=violet!7,rounded corners](0,0) rectangle (2.75,2.1);
  \foreach \o in {-0.55,-0.28,0.28,0.55}
    \draw[gray!55] plot[smooth,domain=0.22:2.55,samples=22]({\x},{1.05+\o*(1-\x/2.75)});
  \draw[violet!80,very thick] plot[domain=0.22:2.55,samples=6]({\x},{1.05});
  \node[font=\scriptsize,gray!55!black] at (0.55,1.95){prior};
  \node[font=\scriptsize,violet!80] at (2.25,0.78){post.};
  \node[align=center,font=\bfseries\small] at (1.38,-0.5){Bayesian inversion};
  \node[align=center,font=\scriptsize] at (1.38,-1.0){aREKI $\to$ ES-MDA\\posterior $\mu^y$};
\end{scope}
\draw[->,very thick] (3.2,1.05)--(4.85,1.05) node[midway,above,font=\scriptsize]{simulate};
\draw[->,very thick] (7.45,1.05)--(9.15,1.05) node[midway,above,font=\scriptsize]{learn};
\draw[->,very thick] (12.1,1.05)--(13.55,1.05) node[midway,above,font=\scriptsize]{$g_s$};
\draw[->,very thick] (16.15,1.05)--(17.55,1.05) node[midway,above,font=\scriptsize]{assimilate};
\draw[->,thick] (10.65,-0.1) to[out=-90,in=-90,looseness=0.5] (18.9,-0.1);
\node[font=\scriptsize] at (14.8,-1.55){observed production data $g_p$ (Peaceman heads)};
\draw[->,thick,violet!70,dashed] (18.95,2.2) to[out=90,in=90,looseness=0.32] (1.55,2.95);
\node[font=\scriptsize,violet!70] at (10.2,3.75){posterior $(\log K,\phi)$ updates the reservoir model};
\end{tikzpicture}}
\caption{Cartoon of the end-to-end workflow. A synthetic, fault-dissected
reservoir model with unknown petrophysical properties $(\log K,\phi)$ (left) is
sampled by an ensemble of full-physics
finite-volume runs; these train an \emph{autoregressive} PINO--FNO surrogate that
rolls out the black-oil state in closed loop; the surrogate's own state is pushed
through a closed-form Gassmann map to a 4-D seismic observable and through the
Peaceman heads to production observables; and a two-stage ensemble Bayesian
inversion (aREKI\,$\to$\,ES-MDA) assimilates both to return a posterior
reservoir model (dashed feedback). Each block corresponds to a stage of the
pipeline and a subsection of this paper.}
\label{fig:cartoon}
\end{figure}
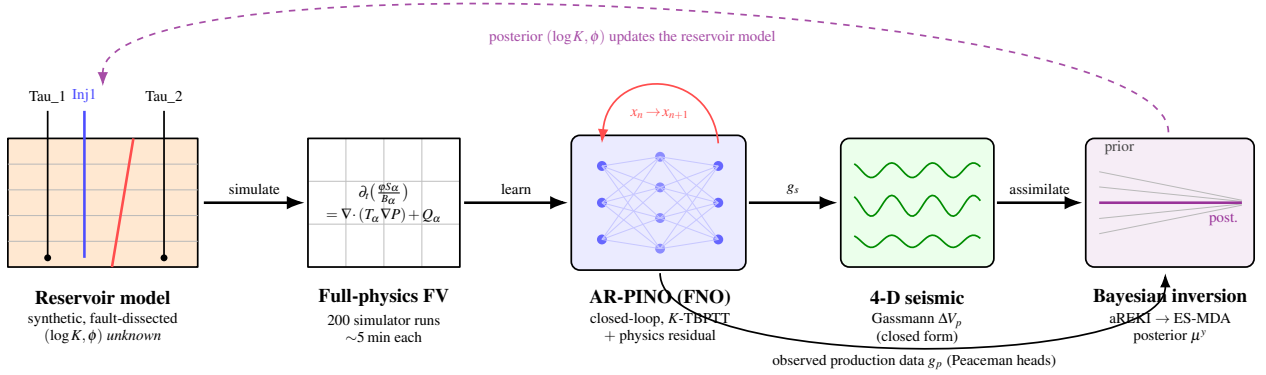

\subsection{Contributions and outline}
\label{sec:contrib}
The contributions of the paper, in the order in which they appear, are:
\begin{enumerate}
\item \textbf{(\S\ref{sec:forward})} A theory and algorithm for autoregressive
PINO operator learning: an exposure-bias total-variation bound, PINO-induced
contractivity, the geometric bias decay of truncated backpropagation through time
(TBPTT), and the discretised black-oil system residual used as the physics
penalty.
\item \textbf{(\S\ref{sec:peaceman})} A derivation of the multi-layer Peaceman
completion model that closes the flow equations at the wells and defines the
production observables.
\item \textbf{(\S\ref{sec:bayes})} A measure-theoretic construction of the
Bayesian posterior, following \citet{stuart2010inverse}: the Borel
$\sigma$-algebra on the parameter space, the Gaussian prior measure, the
likelihood and Radon--Nikodym density, well-posedness, and the ensemble
(particle) approximation.
\item \textbf{(\S\ref{sec:seismic})} A closed-form, GPU-batched Gassmann
petro-elastic forward channel fused inside the rollout, together with
Parseval-exact discrete-cosine-transform (DCT) and raw encodings of the 4-D
seismic observable.
\item \textbf{(\S\ref{sec:param})} DCT and variational convolutional auto-encoder
(VCAE) reparametrisations of the petrophysical properties $(\log K,\phi)$, with a
precise statement of why dimension reduction is \emph{necessary} for a finite
ensemble.
\item \textbf{(\S\ref{sec:inverse})} An adaptive regularised ensemble Kalman
inversion (aREKI) \citep{iglesias2021adaptive,etienam2024modulus,etienam2024norne} with discrepancy-principle step selection, localisation and a
two-stage schedule, motivated by an elementary Löwner-order argument that, in the
linearised Gaussian regime, joint assimilation should be no less informative than
production-only assimilation.
\item \textbf{(\S\ref{sec:hpc})} A high-performance realisation using distributed
data parallelism (DDP), NVIDIA Collective Communications Library (NCCL)
all-reduce and CUDA graphs.
\item \textbf{(\S\ref{sec:results})} Inverse-modelling results on the Tau model
with a blind forecast window, together with an explicit statement of scope
(\S\ref{sec:limitations}).
\end{enumerate}
Section~\ref{sec:discussion} discusses the results and concludes.

\section{The Forward Problem and Its Autoregressive Neural-Operator Surrogate}
\label{sec:notation}

\subsection{Reservoir physics, discretisation, and problem setup}
\label{sec:blackoil}

\subsubsection{Continuum model}
On the reservoir domain $\Omega\subset\R^3$, discretised on a Cartesian grid
$(n_x,n_y,n_z)=(46,99,22)$ with $N=n_xn_yn_z=100{,}188$ active cells, the
black-oil model describes three-phase (water/oil/gas) flow with gas dissolving
into oil at solution ratio $R_{so}$. The primary unknowns are the oil-phase
pressure $P$, water saturation $S_w$ and gas saturation $S_g$, with
$S_o=1-S_w-S_g$. With phase mobility $\lambda_\alpha=K_{r\alpha}/\mu_\alpha$,
transmissibility $T_\alpha=K(x)\lambda_\alpha$ and formation-volume factors
$B_\alpha$, mass conservation for water and oil reads
\begin{align}
\frac{\partial}{\partial t}\!\Big(\frac{\varphi S_w}{B_w}\Big)
-\grad\!\cdot\!\Big[\tfrac{T_w}{B_w}\big(\grad P-\rho_w g\grad Z\big)\Big]
&=\frac{Q_w}{B_w},
\label{eq:water}\\
\frac{\partial}{\partial t}\!\Big(\frac{\varphi S_o}{B_o}\Big)
-\grad\!\cdot\!\Big[\tfrac{T_o}{B_o}\big(\grad P-\rho_o g\grad Z\big)\Big]
&=\frac{Q_o}{B_o},
\label{eq:oil}
\end{align}
with an analogous gas equation carrying the free-plus-dissolved term
$\partial_t(\varphi[S_g/B_g+R_{so}S_o/B_o])$. Here $\varphi$ is porosity,
$Z$ depth, $g$ gravity, and $Q_\alpha$ the well source/sink terms of
Section~\ref{sec:peaceman}. Under the small-capillary approximation
$P_w\!\approx\!P_o\!\approx\!P_g\!=:\!P$ the system reduces to an elliptic
global pressure equation coupled to advection-dominated saturation transport.

\subsubsection{Constitutive relations}
The Stone~II relative permeabilities (from the run configuration) are,
in Corey form,
\begin{equation}
K_{rw}=k_{rw}^{\max}\!\Big(\tfrac{S_w-S_{wc}}{1-S_{wc}-S_{or}}\Big)^{n_w},\;\;
K_{ro}=k_{ro}^{\max}\!\Big(\tfrac{1-S_w-S_{or}}{1-S_{wc}-S_{or}}\Big)^{n_o},\;\;
K_{rg}=k_{rg}^{\max}\!\Big(\tfrac{S_g-S_{gc}}{1-S_{wc}-S_{gc}}\Big)^{n_g},
\label{eq:relperm}
\end{equation}
with $k_{rw}^{\max}{=}0.3$, $k_{ro}^{\max}{=}0.9$, $k_{rg}^{\max}{=}0.8$,
$n_\bullet{=}2$, $S_{wc}{=}0.1$, $S_{or}{=}0.2$, $S_{gc}{=}0.01$. Fault
transmissibility multipliers $f(x)\in[0,1]$ scale inter-cell fluxes,
$T_{ij}^{\mathrm{eff}}=f_{ij}\,T_{ij}^{\mathrm{matrix}}$, with $f_{ij}{=}0$ a
sealing fault; $f$ is a static input channel to the surrogate, allowing it to
learn fault-controlled \texttt{Inj1}$\to$producer flow paths.

\subsubsection{PVT, compressibility and the equilibration timescale}
\label{sec:pvt}
The formation-volume factors and viscosities encode the pressure-dependent
phase behaviour. For the undersaturated hydrocarbon fluid ($P>P_b$), we use
linearised compressible formation-volume factors (FVFs)
\begin{equation}
B_o(P)\approx B_o^{\mathrm{ref}}\big(1-C_{fo}(P-P_b)\big),\quad
B_w(P)\approx B_w^{\mathrm{ref}}\big(1-C_{fw}(P-P_b)\big),\quad
\mu_o\approx2.5,\ \mu_w\approx1.0\ \text{cP},
\label{eq:pvt}
\end{equation}
with $B_o^{\mathrm{ref}}{=}1.2$, $B_w^{\mathrm{ref}}{=}1.0$~rm$^3$/sm$^3$ and
$C_{fo}{=}C_{fw}{\approx}1.5\times10^{-6}$~kPa$^{-1}$. The total compressibility
$c_t=\varphi(S_wc_w+S_oc_o+S_gc_g)+c_r\approx2.9\times10^{-6}$~kPa$^{-1}$ sets the
pressure diffusivity $\kappa=K/(\varphi\mu c_t)$ and hence the equilibration
timescale $\tau_P=L^2/\kappa$. With the high-permeability swept sand
($K\!\sim\!10^3$~mD, $\varphi\!\approx\!0.2$, $\mu\!\approx\!1$~cP) this gives
$\tau_P\approx115$~days over the $\sim\!4$~km reservoir, which is why the
report-step spacing grows from $1$~day initially to $183$~days once pressure has
equilibrated. This separation $\kappa_P\gg\kappa_S$ is the diffusivity hierarchy
invoked in
Theorem~\thmr{thm:contract}: pressure relaxes globally and instantaneously
(elliptic, $\alpha_P\ll1$), saturation advects locally (hyperbolic,
$\alpha_O\to1$).

\subsubsection{Finite-volume discretisation and the discrete residual}
\label{sec:fvresidual}
A cell-centred two-point-flux (TPFA) discretisation with upwind mobilities
gives, for cell $i$, phase $\alpha$, at time level $n{+}1$,
\begin{equation}
\frac{V_i}{\Delta t_n}\Big[\Big(\tfrac{\varphi S_\alpha}{B_\alpha}\Big)_i^{n+1}
-\Big(\tfrac{\varphi S_\alpha}{B_\alpha}\Big)_i^{n}\Big]
-\!\!\sum_{j\in\cN(i)}\!\!T_{\alpha,ij}^{n+1}\big(P_j^{n+1}-P_i^{n+1}\big)
=Q_{\alpha,i}^{n+1},
\label{eq:fv}
\end{equation}
with $T_{\alpha,ij}^{n+1}=T_{ij}^{\mathrm{geo}}\lambda_\alpha^{\mathrm{up}}$,
the geometric transmissibility
$T_{ij}^{\mathrm{geo}}=2K_iK_jA_{ij}/(K_id_j+K_jd_i)$, and the upwind mobility
$\lambda_{\alpha,ij}^{\mathrm{up}}=\lambda_{\alpha,i}$ if $P_i\ge P_j$ else
$\lambda_{\alpha,j}$. Stacking~\eqr{eq:fv} over all cells and phases defines
the \emph{discrete residual}
\begin{equation}
R(x_{n+1},x_n;u,c_n,\tau_n)=0,\qquad R:\R^{4N}\!\times\R^{4N}\!\times\R^{3N}
\!\times\R^{3N}\!\times\R^2\to\R^{4N},
\label{eq:residual}
\end{equation}
which the FV simulator drives to $\norm{R}<10^{-6}$ by Newton--Krylov, and
which we will reuse, \emph{evaluated, never inverted}, as the PINO
physics penalty in~\eqr{eq:pino}. This is the crucial economy of PINO over
PINNs: the penalty uses the \emph{discretised} residual~\eqr{eq:residual}
through the discrete black-oil residual operator, requiring one
$\mathcal{O}(N)$ assembly per step and no symbolic differentiation of the PDE.

\subsubsection{Fractional flow and front-speed: the source of saturation error}
\label{sec:frontspeed}
The mixed elliptic--hyperbolic character of the system explains where the
surrogate's residual error concentrates. With $P$ fixed (elliptic), the water
saturation obeys an advection equation whose flux is the fractional-flow
function $f_w=T_w/(T_w+T_o+T_g)$, and sharp saturation fronts propagate at the
Rankine--Hugoniot speed
\begin{equation}
v_{\mathrm{front}}=\frac{\llbracket f_w\rrbracket}{\llbracket S_w\rrbracket}\,
\frac{K\lambda_t}{\phi}\,\norm{\grad P},
\label{eq:rh}
\end{equation}
i.e.\ the chord slope $\llbracket f_w\rrbracket/\llbracket S_w\rrbracket$ times
the total Darcy velocity $u_t=K\lambda_t\norm{\grad P}$ divided by porosity
$\phi$, with $\lambda_t=\sum_\alpha\lambda_\alpha$ the total mobility and
$\llbracket\cdot\rrbracket$ the jump across the front. A small
surrogate error in $\grad P$ therefore produces a \emph{positional} error in
the front that grows linearly in time, the concrete realisation of the
Lipschitz amplification of Proposition~\thmr{prop:rollout} along
saturation directions, and the reason the oil $R^2$ declines mildly at long
horizons while pressure $R^2$ does not (Theorem~\thmr{thm:contract}). The
up-weighting $w_W=2>w_O=1$ in the loss, and the increment terms below,
deliberately bias the optimiser toward resolving the front position.

\subsubsection{The composite loss in detail}
\label{sec:loss}
Each field term uses a composite loss blending an absolute and an
increment ($p$-norm) penalty,
\begin{equation}
\ell\big(\hat y,y\big)=\alpha\,\norm{\hat y-y}_p^p
+\beta\,\norm{\nabla_t\hat y-\nabla_t y}_p^p,
\label{eq:combined}
\end{equation}
with $p=2$, $\alpha=\beta=1$, where $\nabla_t$ is the temporal increment. The
per-field weights are $w_P{=}2$ (pressure), $w_W{=}2$ (water),
$w_O{=}1$ (oil), $w_Y{=}10$ (Peaceman, large because the wells are sparse
but carry the assimilable information), plus increment weights
$w_{\Delta p}{=}1,\ w_{\Delta S_w}{=}w_{\Delta S_o}{=}0.1$ that penalise
$|\Delta p|,|\Delta S|$ at one-tenth their absolute counterparts to avoid
double-counting magnitude. The seven PINO residual weights are each $10^{-6}$.
The total objective at epoch $e$ is
$\cL^{(e)}=\sum_{k}w_k\,\ell_k^{(e)}$ over the $11$ heads, monitored
component-wise through MLflow (Fig.~\ref{fig:metrics}).

\subsubsection{FNO complexity and parameter budget}
\label{sec:complexity}
Each of the $L=4$ spectral layers costs $\mathcal{O}(N\log N\,d_v^2)$ for the FFT
pair plus $\mathcal{O}(k_{\max}^3 d_v^2)$ for the truncated spectral multiply,
with $N=100{,}188$, $d_v=32$ (the latent width) and $k_{\max}=16$ retained
Fourier modes per spatial dimension (the FFT padding of $8$ cells keeps $16$
modes admissible even on the $22$-layer $z$-axis), so
$k_{\max}^3=16^3=4096$. The complex spectral tensor
$R^{(\ell)}\in\mathbb C^{d_v\times d_v\times k_{\max}^3}$ plus the pointwise map
$W^{(\ell)}\in\R^{d_v\times d_v}$ give
$2\!\cdot\!32^2\!\cdot\!4096+32^2\approx8.4\times10^6$ parameters per spectral
layer; with the four spectral layers and the Sym-style Conv1d decoder
(Table~\ref{tab:arch}) each $3$-D head is an $\mathcal{O}(3\!\times\!10^7)$-parameter
operator that fits comfortably in the $192$~GB HBM3e and keeps the $20$-step
rollout activation footprint modest. The
contrast with a 3-D CNN, whose cost
$\mathcal{O}(N\,C_{\mathrm{in}}C_{\mathrm{out}}k^3)$ is resolution-locked, is
the practical reason for the operator choice (Theorem~\thmr{thm:approx}'s
resolution invariance).

\subsubsection{The reference Newton--Krylov solve}
\label{sec:newton}
The FV reference simulator advances the state by solving~\eqr{eq:residual}
fully implicitly. At each report step the nonlinear system $R(x_{n+1},\cdot)=0$
is solved by Newton's method,
\begin{equation}
x_{n+1}^{(\nu+1)}=x_{n+1}^{(\nu)}-\Big[\tfrac{\partial R}{\partial x_{n+1}}
\Big|_{x_{n+1}^{(\nu)}}\Big]^{-1}R\big(x_{n+1}^{(\nu)},x_n;u,c_n,\tau_n\big),
\label{eq:newtit}
\end{equation}
where the Jacobian $J^{(\nu)}=\partial R/\partial x_{n+1}\in\R^{4N\times4N}$
($\sim\!4\times10^5$ square for the investigated model) is assembled analytically and inverted
approximately by preconditioned GMRES/BiCGSTAB. Each Newton iteration thus
costs a sparse linear solve of dimension $\sim\!4\times10^5$, and the full
$20$-step history needs $\mathcal{O}(10^3$--$10^4)$ Newton iterations per
realisation, the cost the surrogate eliminates. The \emph{same}
residual $R$ is what PINO penalises: the surrogate's job is to predict the
Newton \emph{fixed point} $x_{n+1}$ directly, in one forward pass, while
$\eqr{eq:pino}$ checks that the prediction nearly satisfies
$R(\hat x_{n+1},\hat x_n)\approx0$, without ever forming or inverting
$J$. This replaces an $\mathcal{O}(N^{1.5}$--$N^2)$ implicit solve by an
$\mathcal{O}(N\log N)$ spectral forward pass.

\subsubsection{Data extraction and the training tensor}
\label{sec:extract}
The first stage parses the $200$ reservoir-simulation output files used for the
surrogate (drawn from a $1400$-realisation pool of the reservoir model;
\S\ref{sec:results}), each
holding $T=20$ report steps on the $46\times99\times22$ grid, into the dense
tensors
$\{P,S_w,S_o\}\in\R^{200\times20\times46\times99\times22}$ (shape order
$(N_e,T,n_x,n_y,n_z)$) plus the static $(\log K,\phi,f)$ and the per-well
completion geometry. The full dynamic dataset is
$200\times20\times3\times100{,}188\approx1.2\times10^9$ FP32 values
($\sim\!4.5$~GB). Normalisation maxima ($\max K,\max P,\dots$) and the Peaceman
feature scalings are computed once and cached, as are the static fields and the
prior ensemble. The well completions parsed here, $20$ layers for each
producer (shale gaps at $6,12$), $15$ for the injector, are the
$\cC$ summed in~\eqr{eq:bhp}.

\subsubsection{State, parameters, controls; the augmented input}
\begin{equation}
\begin{gathered}
x_n=(p_n,S_{w,n},S_{o,n})\in\R^{3N},\qquad
u=(\log K,\phi,f)\in\R^{3N},\\
c_n=(Q_n,Q_{w,n},Q_{g,n})\in\R^{3N},\qquad
\tau_n=(t_n,\Delta t_n).
\end{gathered}
\end{equation}
The hydrocarbon fluid is undersaturated throughout the history ($P>P_b$,
\S\ref{sec:pvt}), so the free-gas saturation is identically zero,
$S_g\equiv0$, and $S_o=1-S_w$. The general four-component black-oil template
$\R^{4N}$ used in the residual \eqr{eq:residual} therefore reduces to the
three active fields $(p,S_w,S_o)$ that the surrogate advances with its three
dynamic heads; the gas row of $R$ is carried for generality but is trivially
satisfied, and the tabulated $k_{rg}^{\max}$ is inactive for this model.
Permeability is log-scaled to condition the inverse problem ($K\in[0,2\times
10^4]$ mD spans three decades). The $T=20$ report times (days) are fixed by the
acquisition schedule,
\begin{equation*}
\{t_n\}=\{1,31,183,366,548,595,731,732,913,1096,1097,1279,1462,1463,1644,1827,
1828,2009,2192,2193\},
\end{equation*}
with the matched increments $\Delta t_n$; the two 4-D seismic surveys are
acquired at steps $2$ and $10$ ($183$ and $1097$~days). The surrogate learns the
time-advance map $\cG_\theta:\xi_n=(u,x_n,c_n,\tau_n)\mapsto x_{n+1}$, with the
$n$-fold composition $\cG_\theta^{(n)}=\cG_\theta\circ\cdots\circ\cG_\theta$
realising the closed-loop rollout. Table~\ref{tab:notation} collects the
principal symbols used throughout.

\begin{table}[h]\centering\small
\caption{Principal symbols.}\label{tab:notation}
\begin{tabular}{p{0.17\linewidth}p{0.75\linewidth}}
\toprule
$\Omega,N$ & reservoir domain; active cells $46\times99\times22=100{,}188$\\
$x_n,\hat x_n$ & true / surrogate state $(p,S_w,S_o)$ at step $n$ ($S_g\!\equiv\!0$); $T=20$\\
$u=(\log K,\phi,f)$ & static parameters; inverted on $(\log K,\phi)$, $d=2N$\\
$c_n,\tau_n$ & well controls $(Q,Q_w,Q_g)$; time pair $(t_n,\Delta t_n)$\\
$\cG_\theta,\cG_\theta^{(n)}$ & learned time-advance operator; $n$-fold rollout\\
$R,R_P,R_S$ & discrete FV residual~\eqr{eq:residual}; pressure/saturation parts\\
$K,T$ & TBPTT window ($10$); horizon ($20$)\\
$\mu_n,\hat\mu_n^\theta$ & true vs.\ closed-loop state law; $L,\tilde L,\rho$ Lipschitz/contraction\\
$\mu_0,\mu^y,\Phi$ & prior / posterior measure; data-misfit potential~\eqr{eq:phi}\\
$\cB(\cX),C_0$ & Borel $\sigma$-algebra; prior covariance\\
$N_e,N_a$ & ensemble size ($190$); ES-MDA iterations ($10$)\\
$\alpha_k,\gamma$ & aREKI inflation; inter-phase inflation ($1.01$)\\
$C_{up},C_{pp},K_k$ & cross-/data-covariance; Kalman gain~\eqr{eq:gain}\\
$d=(d_p,d_s)$ & stacked production / seismic data, $k_p$, $k_s{=}4554$\\
$\Gamma=\mathrm{blkdiag}(\Gamma_p,\Gamma_s)$ & noise covariance; $25\%$ / $30\%$\\
$V_p,\Delta V_p,\cD$ & Gassmann velocity; 4-D delta; DCT operator\\
$\mathcal I_p,\mathcal I_s$ & production / seismic Fisher information~\eqr{eq:fisher}\\
\bottomrule
\end{tabular}
\end{table}

\subsection{Peaceman well modelling}
\label{sec:peaceman}

The flow equations~\eqr{eq:water}--\eqr{eq:oil} are closed at the wells by
the Peaceman model, which is also the definition of the production observables
that the inverse problem matches.

\subsubsection{The well-block pressure problem}
A well perforating grid block $i$ is a sub-grid singularity: the radial
pressure profile around a line source cannot be resolved by a block of lateral
size $\Delta x\!\sim\!100$~m. Peaceman's resolution~\citep{peaceman1978} is to
relate the \emph{numerical} block pressure $P_i$ to the \emph{physical}
flowing bottom-hole pressure $P_{\mathrm{BHP}}$ through an equivalent drainage
radius $r_e$ at which the steady radial solution equals the block-average
pressure. For an isotropic Cartesian block,
\begin{equation}
r_e = 0.28\,\frac{\big[(\,k_y/k_x)^{1/2}\Delta x^2+(k_x/k_y)^{1/2}\Delta y^2\big]^{1/2}}
{(k_y/k_x)^{1/4}+(k_x/k_y)^{1/4}}\;\xrightarrow{\,k_x=k_y\,}\;0.2\sqrt{\Delta x\,\Delta y},
\label{eq:re}
\end{equation}
the isotropic limit used here ($r_e=0.2\sqrt{\Delta x\Delta y}$,
$r_w=0.1$~m).

\subsubsection{Layer well index and phase rates}
For completion layer $\ell$ of thickness $h_\ell$ and permeability $K_\ell$,
the well index and phase rate are
\begin{equation}
WI_\ell=\frac{2\pi K_\ell h_\ell}{\ln(r_e/r_w)+S_\ell},\qquad
Q_{\alpha,\ell}=WI_\ell\,\frac{K_{r\alpha}(S_{\alpha,\ell})}{\mu_\alpha B_\alpha}
\big(P_\ell-P_{\mathrm{wf}}\big),
\label{eq:peaceman}
\end{equation}
with $S_\ell$ the mechanical skin and $P_{\mathrm{wf}}$ the flowing bottom-hole
pressure. We use Peaceman's steady equivalent radius $r_e$ directly, so no
pseudo-steady $-\tfrac34$ (Dietz) term appears in $WI_\ell$. The total phase rate sums over the $N_c$
completion layers,
\begin{equation}
Q_\alpha^{\mathrm{well}}=\sum_{\ell=1}^{N_c}Q_{\alpha,\ell}
=\Big(\sum_{\ell}WI_\ell\,\tfrac{K_{r\alpha,\ell}}{\mu_\alpha B_\alpha}\Big)
\big(\bar P_{\mathrm{d}}-P_{\mathrm{wf}}\big),
\label{eq:qwell}
\end{equation}
where the transmissibility-weighted average reservoir (drainage) pressure is
\begin{equation}
\bar P_{\mathrm{d}}=\frac{\sum_{\ell\in\cC}W_\ell\,p_\ell}{\sum_{\ell\in\cC}W_\ell},
\qquad W_\ell=\frac{2\pi K_\ell h_\ell}{\ln(r_e/r_w)} .
\label{eq:bhp}
\end{equation}
The flowing bottom-hole pressure $P_{\mathrm{wf}}$ is the well-constraint
unknown ($\bar P_{\mathrm{d}}\!\neq\!P_{\mathrm{wf}}$, so
$Q_\alpha^{\mathrm{well}}\!\neq\!0$); the producer response we assimilate as
``WBHP\_PROD'' is the transmissibility-weighted average reservoir pressure
$\bar P_{\mathrm{d}}$ of \eqr{eq:bhp}, i.e.\ a near-well drainage-pressure
proxy rather than the flowing BHP itself.
The production wells are completed with $N_c=20$ intervals (shale-baffle gaps
at layers $6,12$, consistent with $f_{ij}\!\approx\!0$), and \texttt{Inj1}
with $15$. The well responses assimilated downstream are WOPR ($Q_o$), WWPR
($Q_w$), WBHP (injector BHP), WBHP\_PROD (producer BHP), and the cumulative
water cut WWCT\_CUM.

\subsubsection{Why a learned Peaceman head}
The map $(K,p,S_w,S_o)\!\mapsto\!Q_\alpha$ in~\eqr{eq:qwell} is strongly
nonlinear through $K_{r\alpha}(S)$: at near-connate water
($S_w\!\approx\!S_{wc}$), $K_{rw}\!\approx\!0$ on a plateau and rises sharply
at breakthrough, so a scalar near-well feature cannot generalise across the
bifurcation. The surrogate therefore carries five dedicated 1-D FNO
\emph{Peaceman heads}, each consuming a per-well feature tensor assembled from
the rolled-out fields,
\begin{equation}
\mathbf X^{(p)}\in\R^{B\times(4N_p+2)\times T}
=\big[\,k_1,k_2\mid p_1,p_2\mid S_{w,1},S_{w,2}\mid S_{o,1},S_{o,2}\mid t,\Delta t\,\big],
\label{eq:peacefeat}
\end{equation}
(producers, $N_p{=}2$; the injector uses a $3N_{\mathrm{inj}}{+}3$ layout).
This couples the global FNO pressure/saturation prediction to the local
multi-layer completion physics of~\eqr{eq:bhp}, which the comparison study
(\S\ref{sec:results}) shows is decisive for WOPR/WWPR fidelity.

\subsection{The autoregressive PINO surrogate}
\label{sec:forward}

\subsubsection{The Fourier neural operator}
A neural operator learns a map between function spaces rather than between
finite vectors~\citep{kovachki2021neural}. Let
$\cA=\cA(\Omega;\R^{d_a})$, $\cU=\cU(\Omega;\R^{d_u})$ be separable Banach
spaces; the target is the (nonlinear) solution operator
$\cG^\dagger:\cA\to\cU$, $a=(K,\phi,f)\mapsto u=(p,S_w,S_o)$. The Fourier
Neural Operator~\citep{li2020fourier} parametrises $\cG_\theta=Q\circ
v^{(L)}\circ\cdots\circ v^{(1)}\circ P$ with lifting $P$, projection $Q$, and
spectral layers
\begin{equation}
v^{(\ell+1)}(x)=\sigma\!\Big(W^{(\ell)}v^{(\ell)}(x)
+\big(\cK^{(\ell)}v^{(\ell)}\big)(x)\Big),\quad
\big(\cK^{(\ell)}v\big)(x)=\cF^{-1}\!\big(R^{(\ell)}\!\cdot\cF v\big)(x),
\label{eq:fno}
\end{equation}
where $\cF$ is the discrete Fourier transform on the $(22\!\times\!46\!\times\!
99)$ grid and $R^{(\ell)}\in\mathbb C^{d_v\times d_v\times k_{\max}\times
k_{\max}\times k_{\max}}$ a learnable complex weight tensor truncated to
$k_{\max}=16$ modes per spatial dimension, with latent width $d_v=32$ and $L=4$
spectral layers. A key departure from the vanilla FNO is the \emph{decoder}: the
default pointwise projection $Q$ is replaced by a higher-capacity per-cell
Conv1d ``fully-connected'' decoder (six $\mathrm{kernel}{=}1$ layers of width
$128$, SiLU activations, Xavier-uniform initialisation, no weight
normalisation), which sharpens the near-well pressure and displacement-front
detail; the activation $\sigma$ throughout is SiLU. The full head configuration
is given in Table~\ref{tab:arch}. The truncation still imposes a
\emph{spectral bias}: the layer represents only correlations at scales
$\lambda\gtrsim\Omega/k_{\max}$, ideal for the globally-correlated elliptic
pressure, but a systematic under-resolution of $\mathcal{O}(\Delta x)$-wide
saturation fronts. Each layer costs $\mathcal{O}(N\log N\,d_v^2)$ for the FFT
plus $\mathcal{O}(k_{\max}^3 d_v^2)$ for the spectral multiply.

\begin{table}[h]\centering\small
\caption{Network architecture and per-head hyperparameters (from the trained
model). All heads use SiLU activations and the Sym-style Conv1d decoder
(kernel $1$, Xavier-uniform init, no weight norm). ``modes'' is the retained
Fourier-mode count per spatial dimension; ``latent'' is the FNO latent width
$d_v$. The composite operator has $11$ heads: three $3$-D dynamic, three $3$-D
initial-state, and five $1$-D Peaceman.}
\label{tab:arch}
\setlength{\tabcolsep}{4pt}\small
\begin{tabular}{lccccccc}
\toprule
Head group & dim & spectral layers & modes & latent & padding & decoder & in/out ch.\\
\midrule
Dynamic $p,S_w,S_o$        & 3-D & 4 & 16 & 32 & 8 & $6\times128$ & $d_{\mathrm{in}}$/1\\
Initial-state $p_0,S_{w,0},S_{o,0}$ & 3-D & 4 & 16 & 32 & 8 & $6\times128$ & $5$/1\\
Peaceman (WOPR, WWPR, WBHP\_PROD, WWCT) & 1-D & 3 & 4 & 16 & 20 & $2\times64$ & $10$/2\\
Peaceman (WBHP, injector)  & 1-D & 3 & 4 & 16 & 20 & $2\times64$ & $6$/1\\
\bottomrule
\end{tabular}
\end{table}

The same framework \citep{physicsnemo} also supports attention-based
(Transolver) and
physics-informed (PINO, PI-Transolver) backbones behind an identical head
interface; the results below use the FNO/PINO configuration of
Table~\ref{tab:arch}.

\paragraph{Spectral resolution, vertically and laterally.}
What the truncation does and does not discard is best counted against the \emph{available} half-spectrum after padding rather than against grid
cells. With FFT padding of $8$, the padded axes are
$(n_z,n_x,n_y)+8=(30,54,107)$, whose real-FFT half-spectra hold
$\lfloor n/2\rfloor+1 = (16,28,54)$ modes. Retaining $k_{\max}=16$ per dimension
therefore keeps the \emph{entire} vertical spectrum: the $22$-layer column is
represented exactly, so the sealing shale baffles at layers $6$ and $12$ are
\emph{not} spectrally truncated, and the vertical stratigraphy the sweep depends
on is fully resolved. The truncation bites only laterally, where $16$ of $28$
($x$) and $16$ of $54$ ($y$) modes are retained. There the caveat stands: a
sealing fault ($f=0$) is a step discontinuity with energy at all wavenumbers, so
the retained lateral band cannot render it perfectly sharply. The compensating
local path is why we replace the FNO's pointwise projection with the
higher-capacity per-cell Conv1d decoder above; a truly local receptive field
(a U-shaped skip or an explicit convolution branch) remains the standard further
remedy, and our own U-DeepONet experience on comparable synthetic reservoir
models suggests skip paths
help. We flag an ablation on this as a priority (\S\ref{sec:limitations}).

\begin{theorem}[Operator approximation, after
\citealp{kovachki2021neural,lanthaler2022error}]
\label{thm:approx}
Let $\cA\subset L^\infty(\Omega)$ be compact, let $s>0$, and let
$\cG^\dagger:\cA\to H^s(\Omega)$ be a continuous solution operator of the
elliptic pressure problem implied by \eqr{eq:water}--\eqr{eq:oil}. Then
\begin{enumerate}
\item[\textup{(i)}] \textup{(universality)} for every $\varepsilon>0$ there
exist a depth $L$, a width $d_v$, a truncation $k_{\max}$ and parameters
$\theta$ such that the corresponding FNO $\cG_\theta$ of \eqr{eq:fno} satisfies
$\sup_{a\in\cA}\norm{\cG_\theta(a)-\cG^\dagger(a)}_{L^2(\Omega)}\le\varepsilon$;
\item[\textup{(ii)}] \textup{(truncation rate)} if $\cG^\dagger(\cA)$ is bounded
in $H^s(\Omega)$ and $\Pi_{k}$ denotes the $L^2$-orthogonal projection onto the
retained Fourier modes $|\kappa|\le k$, then
\begin{equation}
\sup_{a\in\cA}\norm{\Pi_{k_{\max}}\cG^\dagger(a)-\cG^\dagger(a)}_{L^2(\Omega)}
\le C\,k_{\max}^{-s},\qquad
C=C(s,\Omega)\,\sup_{a\in\cA}\norm{\cG^\dagger(a)}_{H^s(\Omega)} .
\label{eq:fnoerr}
\end{equation}
\end{enumerate}
\end{theorem}
Part~(i) carries no rate; part~(ii) is the quantitative ingredient, and it ties
the admissible $k_{\max}$ to the smoothness of the pressure solution. Two
caveats apply. The error attained by a trained operator also contains an
optimisation error and a generalisation gap, and uniform bounds for the latter,
of order $\sqrt{N_\theta\log N_\theta/M}$ in the parameter count $N_\theta$ and
the number $M$ of training trajectories, are vacuous at
$N_\theta\approx2\times10^{8}$, $M=190$; the held-out test error of
\S\ref{sec:results} is therefore the operative accuracy statement. And because
saturation fields lie in $BV(\Omega)$ rather than $H^s(\Omega)$, $s>1/2$,
hypothesis~(ii) fails across a front, where the rate degrades to the
$\mathcal{O}(k_{\max}^{-1/2})$ of a projected jump: the mechanism behind the
mild long-horizon oil-$R^2$ decline of \S\ref{sec:results}.

\subsubsection{The eleven-head composite operator}
The composite model exposes three 3-D dynamic heads (pressure, water
saturation, oil saturation), three 3-D initial-state heads learning the
static$\to$initial map
$\cG_\theta^{\mathrm{ini}}:(K,\phi,f)\mapsto(p_0,S_{w,0},S_{o,0})$, and five
1-D Peaceman heads of \S\ref{sec:peaceman}. The
ini-state gradient path is computed and back-propagated first, then
\emph{detached}, so the static$\to$initial graph never
entangles with the dynamic rollout graph, a deliberate decoupling that lets
each head converge on its own loss landscape.

\subsubsection{Exposure bias: a total-variation analysis}
Let $\mu_n$ be the law of the true state $x_n$ across the training ensemble,
and $\hat\mu_n^\theta$ the law of the surrogate's closed-loop prediction
$\hat x_n=\cG_\theta^{(n)}(x_0;\cdot)$. Teacher-forced (one-step) training
minimises
\begin{equation}
\cL_{1\text{-}1}(\theta)=\E_{x\sim\mu}\Big[\textstyle\sum_{n=0}^{T-1}
\ell\big(\cG_\theta(u,x_n,c_n,\tau_n),x_{n+1}\big)\Big],
\label{eq:tf}
\end{equation}
optimising over the \emph{true} marginals $\mu_n$. At deployment the operator
is evaluated at inputs drawn from $\hat\mu_n^\theta\ne\mu_n$. The mismatch is
quantified by a total-variation bound.

\begin{proposition}[Closed-loop drift]\label{prop:tv}
Fix the parameter $u$, the controls $(c_n,\tau_n)$ and the initial condition,
and let $\cT$ be a tube around the reference trajectory containing both
$\{x_n\}$ and $\{\hat x_n\}$. Assume
\begin{equation}
\norm{\cG_\theta(u,x,c_n,\tau_n)-\cG_\theta(u,x',c_n,\tau_n)}\le L\norm{x-x'}
\qquad\text{for all }x,x'\in\cT,\ 0\le n<T,
\label{eq:tube}
\end{equation}
and write $e_n=\norm{\cG_\theta(u,x_n,c_n,\tau_n)-x_{n+1}}$ for the one-step
errors and $\delta_0=\hat x_0-x_0$. Then the laws $\hat\mu_n^\theta$ and $\mu_n$
satisfy, in the Kantorovich (Wasserstein-1) distance,
\begin{equation}
W_1\big(\hat\mu_n^\theta,\mu_n\big)\;\le\;\E\norm{\hat x_n-x_n}
\;\le\;\sum_{j=0}^{n-1}L^{\,n-1-j}\,\E e_j\;+\;L^{\,n}\,\E\norm{\delta_0}.
\label{eq:tvbound}
\end{equation}
If, in addition, the fed-back state is perturbed by the Gaussian noise
\eqr{eq:noise} of scale $\sigma$, then the smoothed laws obey
\begin{equation}
d_{\mathrm{TV}}\big(\hat\mu_n^\theta,\mu_n\big)\;\le\;
\frac{1}{2\sigma}\,\E\norm{\hat x_n-x_n}.
\label{eq:tvsmooth}
\end{equation}
\end{proposition}
\begin{proof}
Couple the two processes through the shared $(u,c_n,\tau_n)$ and the shared
randomness of the training ensemble; this is an admissible coupling, so
$W_1\le\E\norm{\hat x_n-x_n}$, and the second inequality in \eqr{eq:tvbound} is
Proposition~\thmr{prop:rollout}. For \eqr{eq:tvsmooth}, conditionally on that
coupling the two smoothed laws are $\cN(\hat x_n,\sigma^2I)$ and
$\cN(x_n,\sigma^2I)$, whose Kullback--Leibler divergence is
$\norm{\hat x_n-x_n}^2/(2\sigma^2)$; Pinsker's inequality bounds their total
variation by $\norm{\hat x_n-x_n}/(2\sigma)$, and the claim follows from
convexity of $d_{\mathrm{TV}}$ in each argument and Jensen's inequality.
\end{proof}
The distinction matters: two deterministic flows differing by an arbitrarily
small amount can have mutually singular laws, so $W_1$, not total variation, is
the metric for the unsmoothed rollout. When $L>1$ the bound \eqr{eq:tvbound}
grows geometrically and the empirical-risk objective~\eqr{eq:tf} ceases to be a
faithful proxy for the deployment loss
$\E_{\hat\mu}[\sum_n\ell(\hat x_{n+1},x_{n+1})]$.

\subsubsection{Closed-loop (autoregressive) training}
Autoregressive training removes the proxy gap by optimising over the
\emph{deployment} distribution directly, following the sequential-PINO
construction of~\citep{etienam2025seqpino,etienam2026seqecmor}:
\begin{equation}
\cL_{\mathrm{AR}}(\theta)=\E\Big[\textstyle\sum_{n=0}^{T-1}
\ell(\hat x_{n+1},x_{n+1})\Big],\qquad
\hat x_{n+1}=\cG_\theta(u,\hat x_n,c_n,\tau_n),\;\hat x_0=x_0.
\label{eq:ar}
\end{equation}
Concretely, the rollout loops over the $T$ report steps: at step $0$ the
detached ini-state predictions seed the initial pressure and saturations; at
each subsequent step the previous \emph{predicted} state is fed back as input.
The loss is accumulated per field with the weights
$w_P{=}2,\,w_W{=}2,\,w_O{=}1$ (the elliptic and water-front terms are
up-weighted relative to the smoother oil field).

\begin{proposition}[Autoregressive stability advantage]\label{prop:arstab}
Let $\cT_r$ be the tube of radius $r$ about the reference trajectory and suppose
that closed-loop training attains a Lipschitz constant $\tilde L<L$
in the sense of \eqr{eq:tube} on $\cT_r$. If the accumulated error remains
within the tube, that is
$\sum_{j=0}^{n-1}\tilde L^{\,n-1-j}\E e_j+\tilde L^{\,n}\E\norm{\delta_0}\le r$
for all $n\le T$, then
\begin{equation}
\E\norm{\hat x_n-x_n}\;\le\;\tilde L^{\,n}\,\E\norm{\delta_0}
+\sum_{j=0}^{n-1}\tilde L^{\,n-1-j}\,\E e_j ,
\label{eq:arstab}
\end{equation}
so the geometric amplification rate in Propositions~\thmr{prop:tv}
and~\thmr{prop:rollout} is reduced from $L$ to $\tilde L$.
\end{proposition}
The hypothesis $\tilde L<L$ is empirical: closed-loop training supplies the
mechanism, not a guarantee. Training on its own perturbed states rewards the
network for pulling them back toward the data manifold, which is a reduction of
the local Jacobian norm along the visited directions; the tube condition is
verified a posteriori by the errors reported in \S\ref{sec:results}.

\subsubsection{The PINO physics penalty in detail}
Pure data fitting does not enforce conservation; over long rollouts a
data-only surrogate may produce $S_w>1$ or violate mass balance. PINO adds the
\emph{discrete} residual~\eqr{eq:residual} as a soft constraint. The
training objective is
\begin{equation}
\cL_{\mathrm{PINO}}(\theta)=\cL_{\mathrm{AR}}(\theta)
+\lambda_R\,\E\Big[\textstyle\sum_{n=0}^{T-1}
\norm{R\big(\hat x_{n+1},\hat x_n;u,c_n,\tau_n\big)}^2\Big],\qquad
\lambda_R=10^{-6},
\label{eq:pino}
\end{equation}
evaluated once every $\lfloor0.01\,T_{\max}\rfloor=7$ epochs,
which amortises the residual cost while keeping the physics gradient present
throughout training. Concretely, $R$ decomposes into a Darcy-flux
\emph{pressure} residual and a transport \emph{saturation} residual,
\begin{align}
R_P &= \grad\!\cdot\!\big[\mathrm{FTM}(x)\,T(x,S)\,\grad\hat P\big]+Q_t,
\label{eq:Rp}\\
R_{S} &= \frac{V_i}{\Delta t_n}\Delta\!\Big(\tfrac{\varphi \hat S_w}{B_w}\Big)
-\!\!\sum_{j\in\cN(i)}\!\!T_{w,ij}^{\mathrm{up}}\big(\hat P_j-\hat P_i\big)-Q_{w,i},
\label{eq:Rs}
\end{align}
assembled by the discrete black-oil residual operator from the surrogate's own
$(\hat P,\hat S_w,\hat S_o)$, the static $K,\phi,f$, the controls, and the
tabulated PVT/relperm. Note that the gradients of~\eqr{eq:Rp}--\eqr{eq:Rs}
\emph{do} flow back into $\theta$ through $\hat P,\hat S$ (autograd through the
residual assembly), so PINO genuinely shapes the operator, but no Newton
iteration and no PDE-symbolic differentiation are ever required.

\begin{theorem}[Contractivity of the exactly-constrained operator]
\label{thm:contract}
Consider the finite-volume advance map $\cF$ of \eqr{eq:fv} and assume
\begin{enumerate}
\item[\textup{(H1)}] the scheme is a two-point-flux upwind discretisation with
non-negative transmissibilities $T_{ij}\ge0$ and monotone (non-decreasing)
upwinded mobilities, so that the implicit pressure matrix $A_n$ is an
irreducibly diagonally dominant M-matrix;
\item[\textup{(H2)}] the report step satisfies the transport stability
condition $\Delta t_n\,\max_i V_i^{-1}\sum_{j\in\cN(i)}T^{\mathrm{up}}_{w,ij}
\,\partial_{S}f_w\le1$, with at least one cell connected to a well or boundary
with strictly positive outflow;
\item[\textup{(H3)}] the physics penalty is exactly satisfied,
$R(\cG_{\theta^\star}(x),x)=0$ for all $x$ in a neighbourhood $\cU$ of the
discrete solution manifold $\cM=\{x:R(x,\cdot)=0\}$, and $\cG_{\theta^\star}$
maps $\cU$ into $\cU$.
\end{enumerate}
Then $\cG_{\theta^\star}=\cF$ on $\cU$ and there is $\alpha<1$, depending only on
the transmissibilities, the total mobility and $\Delta t_n$, with
$\norm{D_x\cG_{\theta^\star}}_\infty\le\alpha$ on $\cU$. The bound decomposes as
$\alpha=\max(\alpha_P,\alpha_O)$, with $\alpha_P\ll1$ for the elliptic pressure
block and $\alpha_O$ closest to unity for the advective block.
\end{theorem}
\begin{proof}
By (H3) the two maps solve the same system on $\cU$ and $R$ has an invertible
derivative in its first argument by (H1), so $\cG_{\theta^\star}=\cF$ there and
it suffices to bound $D_x\cF$. Linearising \eqr{eq:fv} about a state in $\cU$
factors the update into the implicit pressure solve and the upwind transport
step. For the first, (H1) makes $A_n$ an M-matrix, so $A_n^{-1}\ge0$
entrywise and $\norm{A_n^{-1}B_n}_\infty\le1$ by the discrete maximum
principle, with strict inequality on any row connected to a well; hence
$\alpha_P<1$, and $\alpha_P\ll1$ because the pressure diffusivity $\kappa_P$
relaxes a report step globally (\S\ref{sec:pvt}). For the second, (H2) makes the
transport Jacobian a non-negative matrix with row sums at most one, so
$\norm{\cdot}_\infty\le1$, again with strict inequality where there is outflow.
The composition of the two blocks therefore satisfies
$\norm{D_x\cF}_\infty\le\max(\alpha_P,\alpha_O)<1$, and the block ordering
follows from the diffusivity hierarchy $\kappa_P\gg\kappa_S$.
\end{proof}

Hypothesis (H3) is an idealisation: training drives the penalty in \eqr{eq:pino}
to a small but non-zero value. The perturbation statement below is what applies,
and it makes explicit that a small residual does not suffice on its own; its
\emph{derivative} must be small as well.

\begin{proposition}[Contractivity under a residual defect]\label{prop:defect}
Let $r(x)=R(\cG_\theta(x),x)$ be the defect of the trained operator, let $R$ be
$C^1$ with $\partial_1R$ invertible and $\norm{(\partial_1R)^{-1}}\le C_R$ on a
neighbourhood of $\cM$, and let $\partial_1R,\partial_2R$ be Lipschitz there
with constant $C_R'$. If $\sup_{x\in\cU}\norm{r(x)}\le\delta$ and
$\sup_{x\in\cU}\norm{D_xr(x)}\le\delta'$, then
\begin{equation}
\norm{\cG_\theta-\cF}_{\cU}\le C_R\,\delta,\qquad
\norm{D_x\cG_\theta}\le\alpha+C_R\,\delta'+C_RC_R'\,\delta ,
\label{eq:defect}
\end{equation}
with $\alpha$ as in Theorem~\thmr{thm:contract}. In particular the rollout
remains contractive provided $C_R\delta'+C_RC_R'\delta<1-\alpha$.
\end{proposition}
\begin{proof}
Since $R(\cF(x),x)=0$ and $R(\cG_\theta(x),x)=r(x)$, the mean-value form of the
first-argument derivative gives
$\cG_\theta(x)-\cF(x)=(\partial_1R)^{-1}r(x)$ evaluated at an intermediate
point, whence the first bound. Differentiating both identities in $x$ yields
$\partial_1R\,D_x\cG_\theta+\partial_2R=D_xr$ and
$\partial_1R\,D_x\cF+\partial_2R=0$, with the partials evaluated at
$\cG_\theta(x)$ and $\cF(x)$ respectively. Subtracting, using the Lipschitz
bound on the partials together with the first estimate to control the difference
of the evaluation points, and applying $\norm{(\partial_1R)^{-1}}\le C_R$ gives
$\norm{D_x\cG_\theta-D_x\cF}\le C_R\delta'+C_RC_R'\delta$; the triangle
inequality with Theorem~\thmr{thm:contract} completes the proof.
\end{proof}
The constants $C_R,C_R',\delta,\delta'$ are not estimated here, so
\eqr{eq:defect} is structural rather than certified: it identifies what the
penalty must control for the contraction to survive, and the evidence that it
does is the long-horizon rollout accuracy of \S\ref{sec:results}.

\subsubsection{Covariate-shift (noise) injection}
With probability $0.5$ per batch the fed-back state is perturbed,
$\tilde x=x+\varepsilon$, $\varepsilon\sim\cN(0,\sigma_{\mathrm{eff}}^2)$,
\begin{equation}
\sigma_{\mathrm{eff}}^2=\max\!\big(\sigma_{\mathrm{abs}}^2,\,
(\sigma_{\mathrm{rel}}\,|x|)^2,\,\sigma_{\mathrm{std}}^2\big),\qquad
(\sigma_{\mathrm{abs}},\sigma_{\mathrm{rel}},\sigma_{\mathrm{std}})
=(0.001,0.05,0.02).
\label{eq:noise}
\end{equation}
This is empirical-risk minimisation over a tube around the closed-loop
trajectory: it explicitly populates the off-manifold inputs that
Proposition~\thmr{prop:tv} warns about, and trains $\cG_\theta$ to contract
them, reducing $\tilde L$ in Proposition~\thmr{prop:arstab}. It is the
operator-learning analogue of scheduled sampling~\citep{bengio2015scheduled}
and DAgger-style distribution matching~\citep{venkatraman2015improving}.

\subsubsection{$K$-step truncated BPTT}
The exact AR gradient contains Jacobian products
$\prod_{j=m}^{n}D_x\cG_\theta(\hat x_j)$ across all $T=20$ steps, which
explode or vanish~\citep{pascanu2013difficulty}. TBPTT partitions
$\{0,\dots,T-1\}$ into windows $\cW_r$ of size $K=10$ and detaches the graph at
window boundaries, giving the estimator
\begin{equation}
g_{\mathrm{TBPTT}}(\theta)=\sum_{r}\grad_\theta\!\Big(\textstyle\sum_{n\in\cW_r}
\ell(\hat x_{n+1},x_{n+1})\Big).
\label{eq:tbptt}
\end{equation}
This corresponds to a window-end condition, $(n{+}1)\bmod K=0$ or $n=T{-}1$,
at which the accumulated window loss is back-propagated and the fed-back state
is detached. Gradients are clipped to unit norm and a NaN/Inf-safe update skips
the optimiser step on any non-finite gradient, essential for a 20-step
rollout where one bad batch could corrupt the weights.

\begin{proposition}[Geometric TBPTT bias, after \citealp{tallec2017unbiased}]
\label{prop:tbptt}
Assume that along the rollout (i) $\norm{\partial_x\ell}\le G_\ell$,
(ii) $\norm{\partial_\theta\cG_\theta}\le B_\theta$, and (iii)
$\norm{D_x\cG_\theta}\le\rho<1$, the last as furnished by
Theorem~\thmr{thm:contract} or Proposition~\thmr{prop:defect}. Then the
truncated estimator \eqr{eq:tbptt} has bias
\begin{equation}
\norm{\E\big[g_{\mathrm{TBPTT}}\big]-\grad_\theta\cL_{\mathrm{AR}}}
\;\le\;T\,G_\ell B_\theta\,\frac{\rho^{K}}{1-\rho}.
\label{eq:tbpttbias}
\end{equation}
\end{proposition}
\begin{proof}
The exact gradient of \eqr{eq:ar} is a sum over pairs $(n,m)$ with $m\le n$ of
terms $\partial_x\ell(\hat x_{n+1})\big(\prod_{j=m+1}^{n}D_x\cG_\theta(\hat
x_j)\big)\partial_\theta\cG_\theta(\hat x_m)$. Truncation discards exactly those
pairs whose lag $n-m$ exceeds the window length $K$. By (i)--(iii) each
discarded term is bounded by $G_\ell B_\theta\rho^{\,n-m}$, and summing the
geometric tail $\sum_{\ell\ge K}\rho^{\ell}=\rho^K/(1-\rho)$ over the at most
$T$ values of $n$ gives \eqr{eq:tbpttbias}.
\end{proof}
For $\rho\lesssim0.5$, $K=10$ and $T=20$ the factor $\rho^K/(1-\rho)$ is of order
$2\times10^{-3}$, so the truncated gradient is effectively unbiased for the
investigated problem, and the memory of the backward graph is bounded to $K$
steps rather than $T$.

\begin{proposition}[Rollout error bound]\label{prop:rollout}
Assume the tube Lipschitz condition \eqr{eq:tube} with constant $L$, let
$e_n=\norm{\cG_\theta(u,x_n,c_n,\tau_n)-x_{n+1}}$ and $\delta_n=\hat x_n-x_n$.
Then $\norm{\delta_n}\le L^{\,n}\norm{\delta_0}
+\sum_{j=0}^{n-1}L^{\,n-1-j}e_j$ for all $n\le T$; in particular
$\norm{\hat x_n-x_n}\le\sum_{j=0}^{n-1}L^{\,n-1-j}e_j$ when $\hat x_0=x_0$.
\end{proposition}
\begin{proof}
Adding and subtracting $\cG_\theta(u,x_n,c_n,\tau_n)$ and using \eqr{eq:tube},
$\norm{\delta_{n+1}}\le\norm{\cG_\theta(u,\hat x_n,\cdot)-\cG_\theta(u,x_n,\cdot)}
+e_n\le L\norm{\delta_n}+e_n$. The claim follows by induction on $n$, the
hypothesis of \eqr{eq:tube} being available at each step because both
trajectories remain in $\cT$.
\end{proof}

\subsubsection{Training algorithm}
\begin{algorithm}[H]
\caption{\textsc{ARTrainingStep}, autoregressive PINO with $K$-TBPTT}
\label{alg:train}
\small
\begin{algorithmic}[1]
\Require composite $\cG_\theta$; $T{=}20$, $K{=}10$; weights $w_P,w_W,w_O,w_Y$; $\lambda_R$; noise $(\sigma,p)$
\State forward ini-state heads $\to(\hat p_0,\hat S_{w,0},\hat S_{o,0})$; \textbf{backward}$(\ell_{\mathrm{ini}})$; \textbf{detach}
\State $\ell_{\mathrm{win}}\gets0$
\For{$n=0,\dots,T-1$}
  \State seed: $n{=}0$ use detached ini-states; $n{>}0$ use $(\hat p_{n-1},\hat S_{w,n-1},\hat S_{o,n-1})$, w.p.\ $p$ add $\varepsilon$~\eqr{eq:noise}
  \State $\mathbf U^{(n)}\gets\mathrm{cat}(\log K,\phi,f,\hat p_{n-1},\hat S_{w,n-1},\hat S_{o,n-1},Q_n,Q_{w,n},\Delta t_n,t_n)$
  \State $\hat p_n,\hat S_{w,n},\hat S_{o,n}\gets\cG_\theta(\mathbf U^{(n)})$ \Comment{pressure/saturation/oil heads}
  \State $\ell_n\gets w_P\ell_P+w_W\ell_W+w_O\ell_O$
  \If{PINO and $\mathrm{epoch}\bmod7=0$}
     \State assemble $R$ by the discrete black-oil residual operator; $\ell_n\mathrel{+}=\lambda_R(\norm{R_P}^2+\norm{R_S}^2)$
  \EndIf
  \State $\ell_{\mathrm{win}}\mathrel{+}=\ell_n$
  \If{$(n{+}1)\bmod K=0$ \textbf{or} $n=T{-}1$}
     \State \textbf{backward}$(\ell_{\mathrm{win}}/T)$; $\ell_{\mathrm{win}}\gets0$; \textbf{detach}$\hat x_n$ \Comment{TBPTT cut}
  \EndIf
\EndFor
\State assemble Peaceman features~\eqr{eq:peacefeat} at $(i_w,j_w)$; forward 5 heads; \textbf{backward}(well losses)
\State clip grads to unit norm; \textsc{SafeStep} all optimisers (skip on NaN/Inf); step LR schedulers
\State \Return total loss
\end{algorithmic}
\end{algorithm}

\subsubsection{The truncated-BPTT computational graph}
Figure~\ref{fig:tbptt} renders the closed-loop graph and its TBPTT cuts. The
state is fed back at every step (solid arrows); the gradient flows backward
only within each window of $K=10$ steps, the graph being detached at the
window boundary (dashed cut), bounding the retained activation memory to $K$
steps and the Jacobian-product length in~\eqr{eq:tbptt} to $K$ factors.

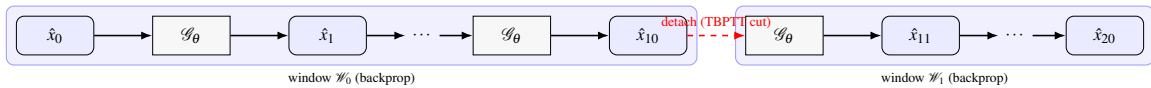
\begin{figure}[h]\centering
\resizebox{0.92\textwidth}{!}{%
\begin{tikzpicture}[node distance=1.05cm,every node/.style={font=\footnotesize}]
\node (x0)[startstop]{$\hat x_0$};
\node (g1)[proc,right=of x0]{$\cG_\theta$};
\node (x1)[startstop,right=of g1]{$\hat x_1$};
\node (gd)[right=0.7cm of x1]{$\cdots$};
\node (g10)[proc,right=0.6cm of gd]{$\cG_\theta$};
\node (x10)[startstop,right=of g10]{$\hat x_{10}$};
\node (g11)[proc,right=of x10]{$\cG_\theta$};
\node (x11)[startstop,right=of g11]{$\hat x_{11}$};
\node (ge)[right=0.7cm of x11]{$\cdots$};
\node (x20)[startstop,right=0.6cm of ge]{$\hat x_{20}$};
\draw[arrow](x0)--(g1);\draw[arrow](g1)--(x1);\draw[arrow](x1)--(gd);
\draw[arrow](gd)--(g10);\draw[arrow](g10)--(x10);
\draw[arrow,dashed,red](x10)--node[above]{\scriptsize detach (TBPTT cut)}(g11);
\draw[arrow](g11)--(x11);\draw[arrow](x11)--(ge);\draw[arrow](ge)--(x20);
\begin{scope}[on background layer]
\node[fit=(x0)(g1)(x1)(gd)(g10)(x10),draw=blue!40,fill=blue!6,rounded corners,inner sep=5pt,label=below:{\scriptsize window $\cW_0$ (backprop)}]{};
\node[fit=(g11)(x11)(ge)(x20),draw=blue!40,fill=blue!6,rounded corners,inner sep=5pt,label=below:{\scriptsize window $\cW_1$ (backprop)}]{};
\end{scope}
\end{tikzpicture}}
\caption{Autoregressive rollout with $K{=}10$ truncated BPTT. State is fed back
every step; the backward graph is detached at the window boundary (red), so
gradients propagate at most $K$ steps and memory is bounded.}
\label{fig:tbptt}
\end{figure}

\subsubsection{Three-phase convergence analysis}
Table~\ref{tab:loss} extracts per-component training losses at key epochs from
the forward log, revealing three regimes. \emph{Phase I (1--50):} the total
loss falls $\sim\!200\times$, dominated by the pressure and initial-state heads
($\ell_P:2.86\!\times\!10^4\!\to\!50$), as the FNO's leading Fourier modes
capture the smooth elliptic pressure structure, the spectral-error term of
Theorem~\thmr{thm:approx} vanishing first. \emph{Phase II (50--300):} the
saturation heads dominate the reduction ($\ell_{S_w}:235\!\to\!27$) as the
injector front is resolved across the $K{=}10$ window. \emph{Phase III
(300--700):} Peaceman refinement under TBPTT gradient variance, with the
characteristic non-monotone excursion at epoch $500$
($\mathrm{Var}[g_{\mathrm{TBPTT}}]\propto\sigma_e^2/K$) damped by the noise
injection.

\begin{table}[h]\centering\small
\caption{Per-component training loss at key epochs (from the forward log).
$\ell_\Pi$ is the summed Peaceman loss; $\eta$ the learning rate. The columns
shown are a representative subset of the $11$ weighted heads, so they need not
sum to the total $\cL^{(e)}$, which also includes the water/oil initial-state
heads and the seven PINO residual terms.}
\label{tab:loss}
\begin{tabular}{rccccccc}
\toprule
Epoch & $\cL^{(e)}$ & $\ell_P$ & $\ell_{S_w}$ & $\ell_{S_o}$ & $\ell_{P,\mathrm{ini}}$ & $\ell_\Pi$ & $\eta$\\
\midrule
1   & $1.69\!\times\!10^5$ & $2.86\!\times\!10^4$ & $1.37\!\times\!10^4$ & $2.11\!\times\!10^4$ & $2.85\!\times\!10^4$ & $77.5$ & $2.00\!\times\!10^{-3}$\\
50  & $8.38\!\times\!10^2$ & $50.3$ & $235.4$ & $109.9$ & $9.93$ & $5.27$ & $2.00\!\times\!10^{-3}$\\
200 & $4.53\!\times\!10^2$ & $25.3$ & $34.8$  & $14.8$  & $2.14$ & $2.76$ & $2.00\!\times\!10^{-3}$\\
300 & $1.24\!\times\!10^2$ & $14.4$ & $27.0$  & $14.3$  & $5.46$ & $2.74$ & $1.99\!\times\!10^{-3}$\\
500 & $7.57\!\times\!10^1$ & $21.4$ & $19.1$  & $9.42$  & $1.33$ & $1.68$ & $1.99\!\times\!10^{-3}$\\
700 & $2.49\!\times\!10^1$ & $6.75$ & $3.85$  & $3.36$  & $0.98$ & $1.39$ & $1.98\!\times\!10^{-3}$\\
\bottomrule
\end{tabular}
\end{table}

\subsubsection{Forward training results}
Training ran $700$ epochs on $4\times$GB200 with DDP/NCCL,
$\approx\!19.9$~s/epoch, $3$h$58$m wall-clock, final loss $24.90$, best
validation $23.57$. Figure~\ref{fig:metrics} shows all 13 MLflow-tracked loss
components. Convergence is three-phase: a rapid $\sim\!200\times$ spectral
collapse of the elliptic and initial-state heads within $\sim\!50$ epochs (the
low-frequency pressure structure is captured first, consistent with
Theorem~\thmr{thm:approx}); saturation-front learning over epochs $50$--$300$;
and Peaceman refinement to epoch $700$ under the TBPTT gradient variance
$\mathrm{Var}[g_{\mathrm{TBPTT}}]\propto\sigma_e^2/K$.

\begin{figure}[h]\centering
\includegraphics[width=0.74\linewidth]{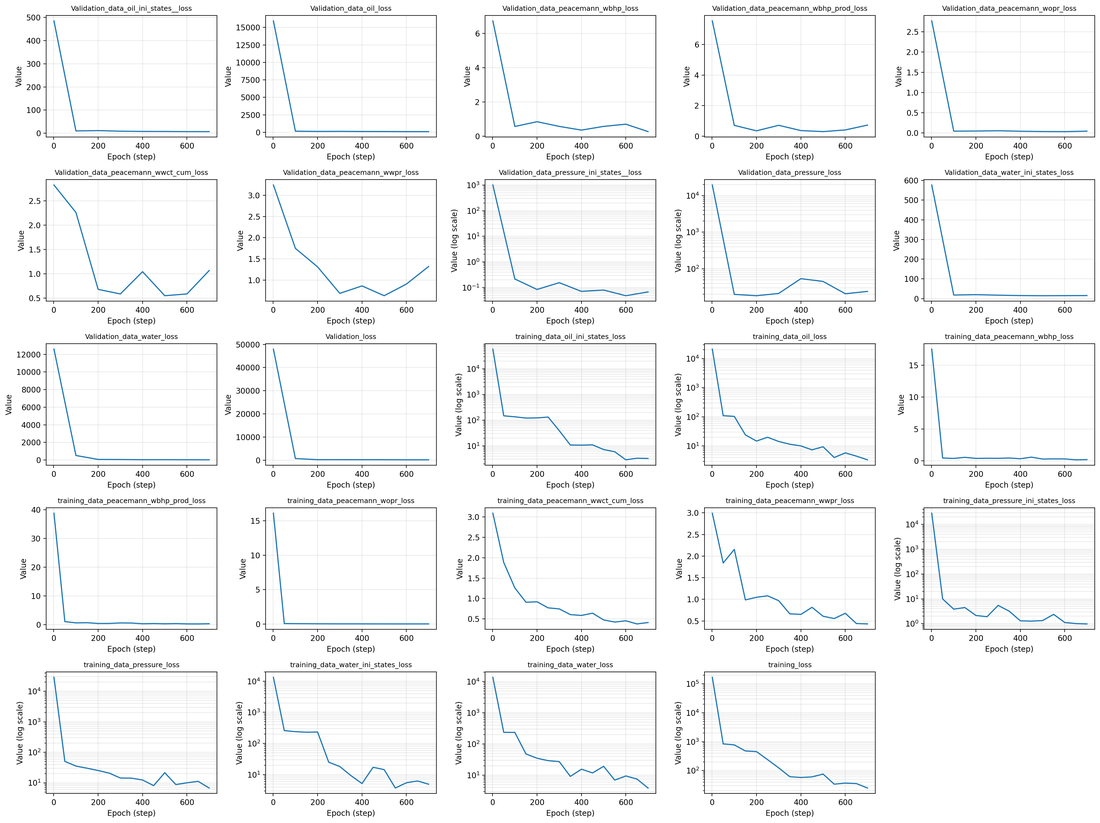}
\caption{MLflow training/validation curves for all 13 PhyNeMo loss components
over 700 epochs. All converge within $\approx200$ epochs;
oscillations reflect the $K{=}10$ TBPTT window schedule and the AdamW
learning-rate decay $\gamma=0.99999$.}
\label{fig:metrics}
\end{figure}

\subsection{Closed-loop inference}
\label{sec:inference}

At inference the surrogate runs in gradient-free mode with \emph{zero} teacher
forcing: the state at step $n$ is purely the model output of step $n{-}1$
(Algorithm~\ref{alg:inference}). GPU memory is recycled when usage exceeds
$90\%$, enabling $N_e=190$ members on a single 192~GB HBM3e GPU. The FNO
internal layout $(B,T,n_z,n_x,n_y)$ is permuted to the analysis layout
$(B,T,n_x,n_y,n_z)$. A stochastic variant injects Gaussian state
perturbations for UQ rollouts, realising an
implicit ensemble around the deterministic trajectory.

\begin{algorithm}[H]
\caption{\textsc{ForwardModelEnsemble}, closed-loop rollout}
\label{alg:inference}
\small
\begin{algorithmic}[1]
\Require $N_e$ realisations $(K_i,\phi_i,f_i)$; trained heads; $\{t_n,\Delta t_n\}$
\For{$i=1,\dots,N_e$}
  \State $(\hat p^{(-1)},\hat S_w^{(-1)},\hat S_o^{(-1)})\gets$ ini-state heads$(K_i,\phi_i,f_i)$
  \For{$n=0,\dots,T-1$}
     \State $\mathbf U^{(n)}\gets\mathrm{cat}(K_i,\phi_i,\hat p^{(n-1)},\hat S_w^{(n-1)},\hat S_o^{(n-1)},f_i,Q_n,Q_{w,n},\Delta t_n,t_n)$
     \State $\hat p^{(n)},\hat S_w^{(n)},\hat S_o^{(n)}\gets\mathrm{clamp}(\cG_\theta(\mathbf U^{(n)}),0,1)$; store; update state
     \If{GPU mem $>90\%$} \State empty cache \EndIf
  \EndFor
\EndFor
\State permute to analysis layout; denormalise; extract well features; forward 5 Peaceman heads
\State \Return $P,S_w,S_o$, well arrays $\{$WOPR,WWPR,WBHP,WBHP\_PROD,WWCT\_CUM$\}$
\end{algorithmic}
\end{algorithm}

Figure~\ref{fig:cmp} compares surrogate and FV well responses on a
member of the $5$-realisation held-out test set (\S\ref{sec:results});
Figure~\ref{fig:r2l2} reports time-resolved field and seismic accuracy; Figure~\ref{fig:dyn} shows 3-D field snapshots spanning the full
production history. Across the 20-step rollout the field errors stay small
(pressure differences $<\!2500$~kPa, $<\!4\%$ of range; saturation errors
localised at the injector front), consistent with
Proposition~\thmr{prop:rollout}.

\begin{figure}[h]\centering
\begin{subfigure}{0.33\linewidth}\includegraphics[width=\linewidth]{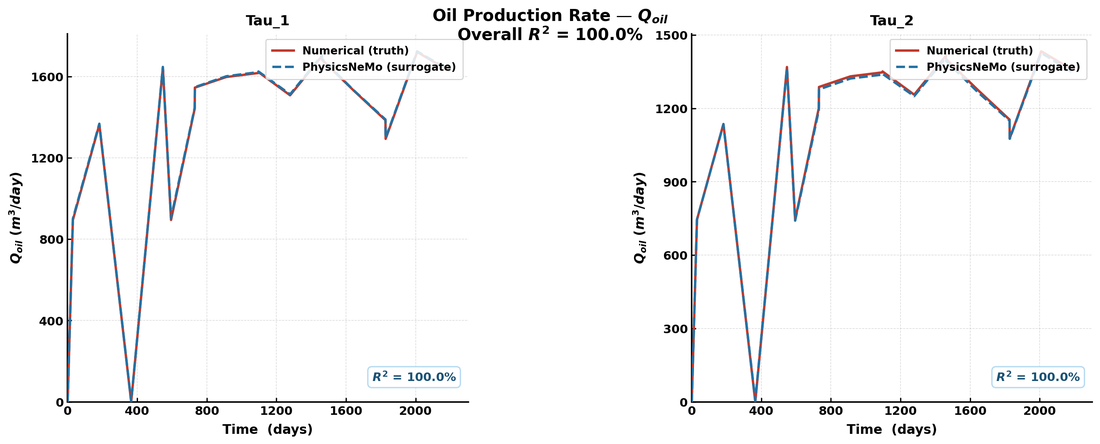}\caption{WOPR}\end{subfigure}
\begin{subfigure}{0.33\linewidth}\includegraphics[width=\linewidth]{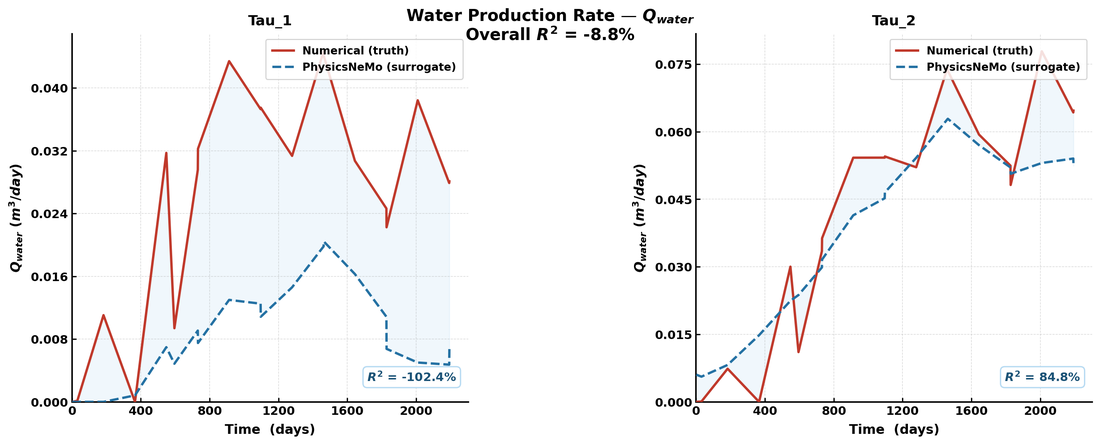}\caption{WWPR}\end{subfigure}
\begin{subfigure}{0.33\linewidth}\includegraphics[width=\linewidth]{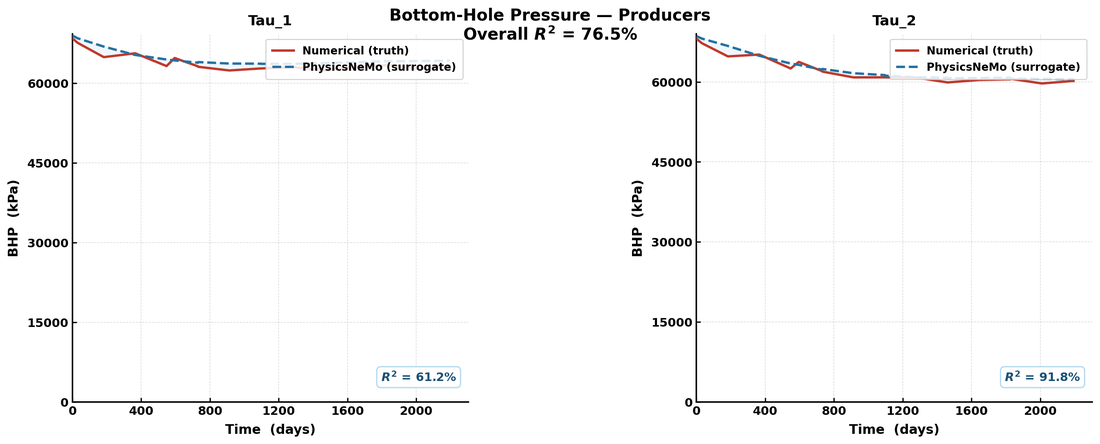}\caption{WBHP (prod.)}\end{subfigure}
\caption{Surrogate (FNO Peaceman head) vs.\ FV reference well responses for
\texttt{Tau\_1}/\texttt{Tau\_2}: WOPR $R^2\!>\!0.9999$, WBHP\_PROD $R^2\!=\!0.98$
(reported to finite precision; we avoid the ``$100\%$'' shorthand).}
\label{fig:cmp}
\end{figure}

\begin{figure}[h]\centering
\begin{subfigure}{0.49\linewidth}\includegraphics[width=\linewidth]{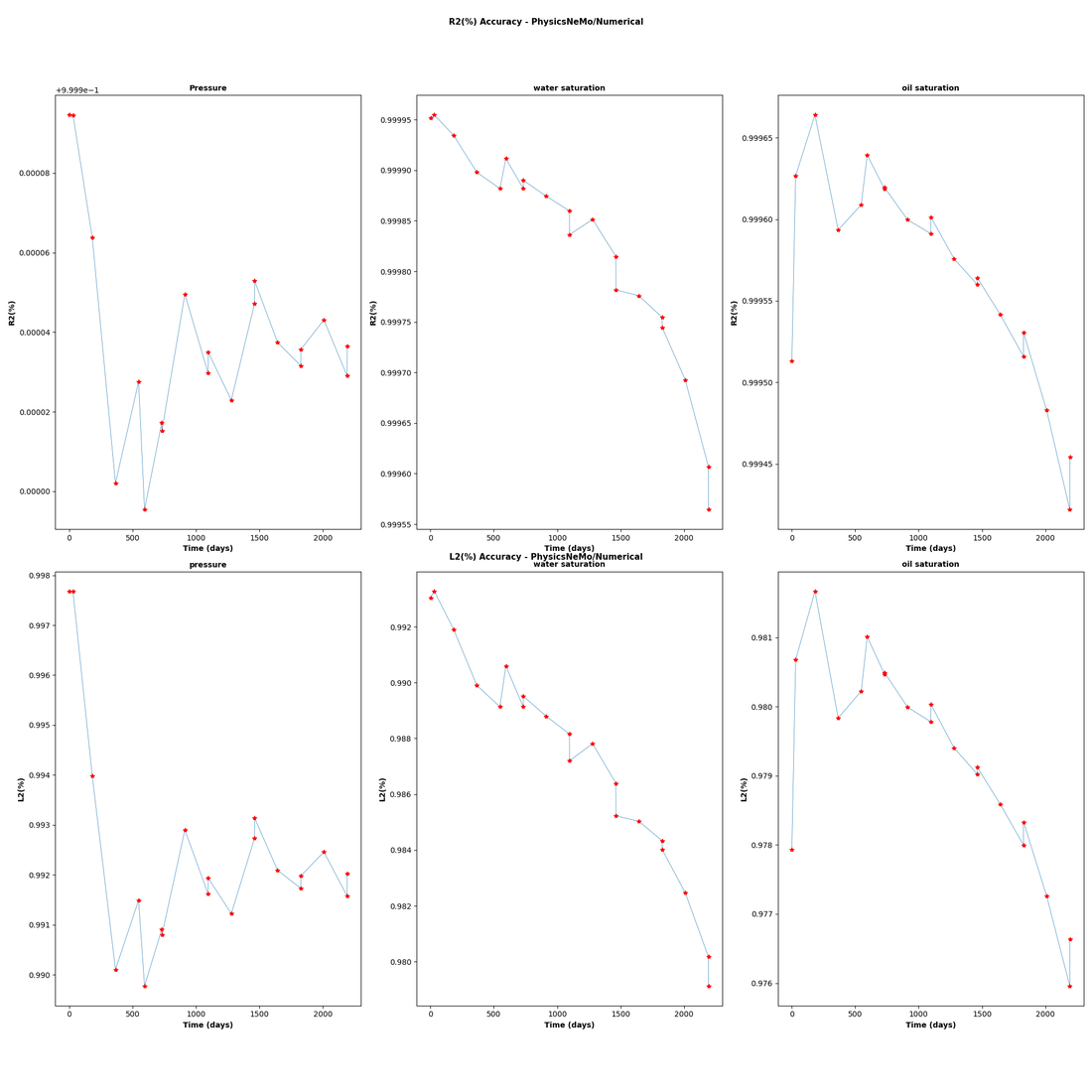}\caption{Production fields}\end{subfigure}
\begin{subfigure}{0.49\linewidth}\includegraphics[width=\linewidth]{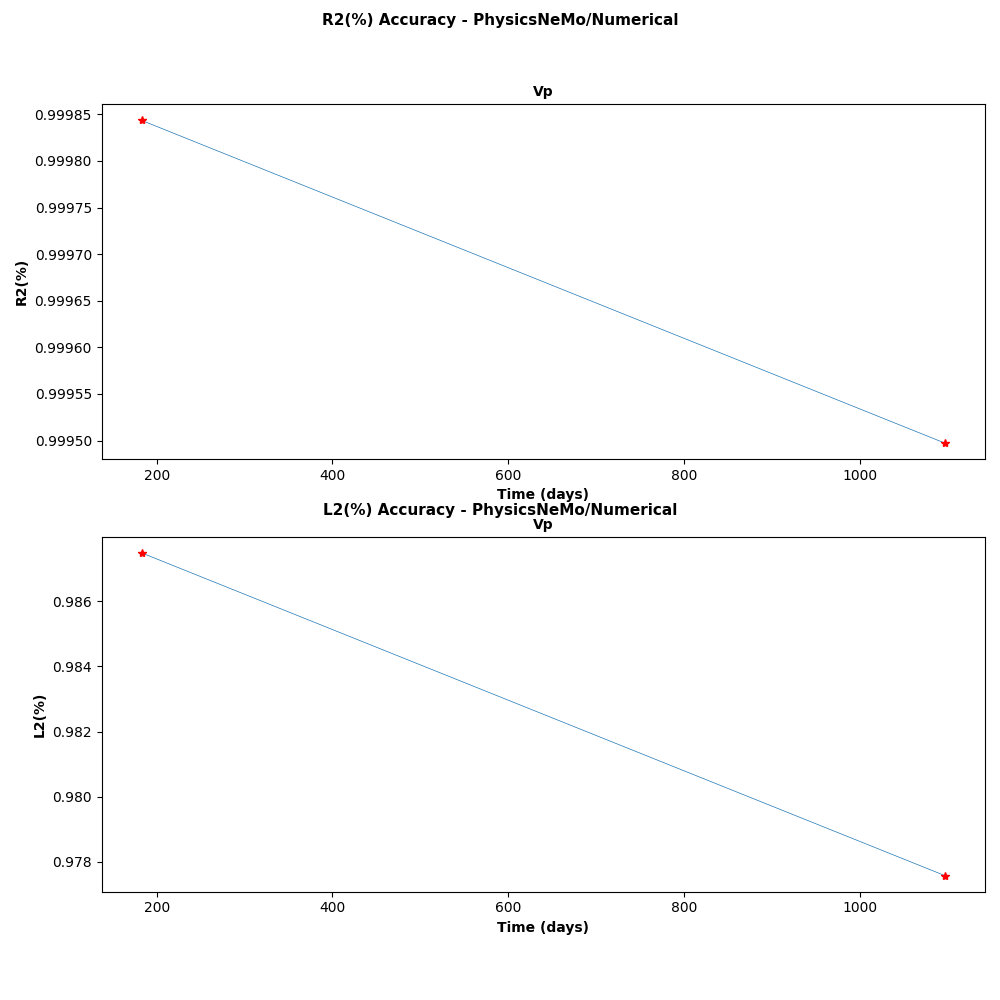}\caption{Seismic $\Delta V_p$}\end{subfigure}
\caption{Time-resolved $R^2$/$L^2$ for production fields (left) and the
Gassmann-derived 4-D seismic observable (right). Pressure/saturation
$R^2>0.999$; the seismic delta is reproduced with high fidelity, validating
the petro-elastic forward channel of \S\ref{sec:seismic}.}
\label{fig:r2l2}
\end{figure}

\begin{figure}[h]\centering
\begin{subfigure}{0.49\linewidth}\fbox{\includegraphics[width=\linewidth]{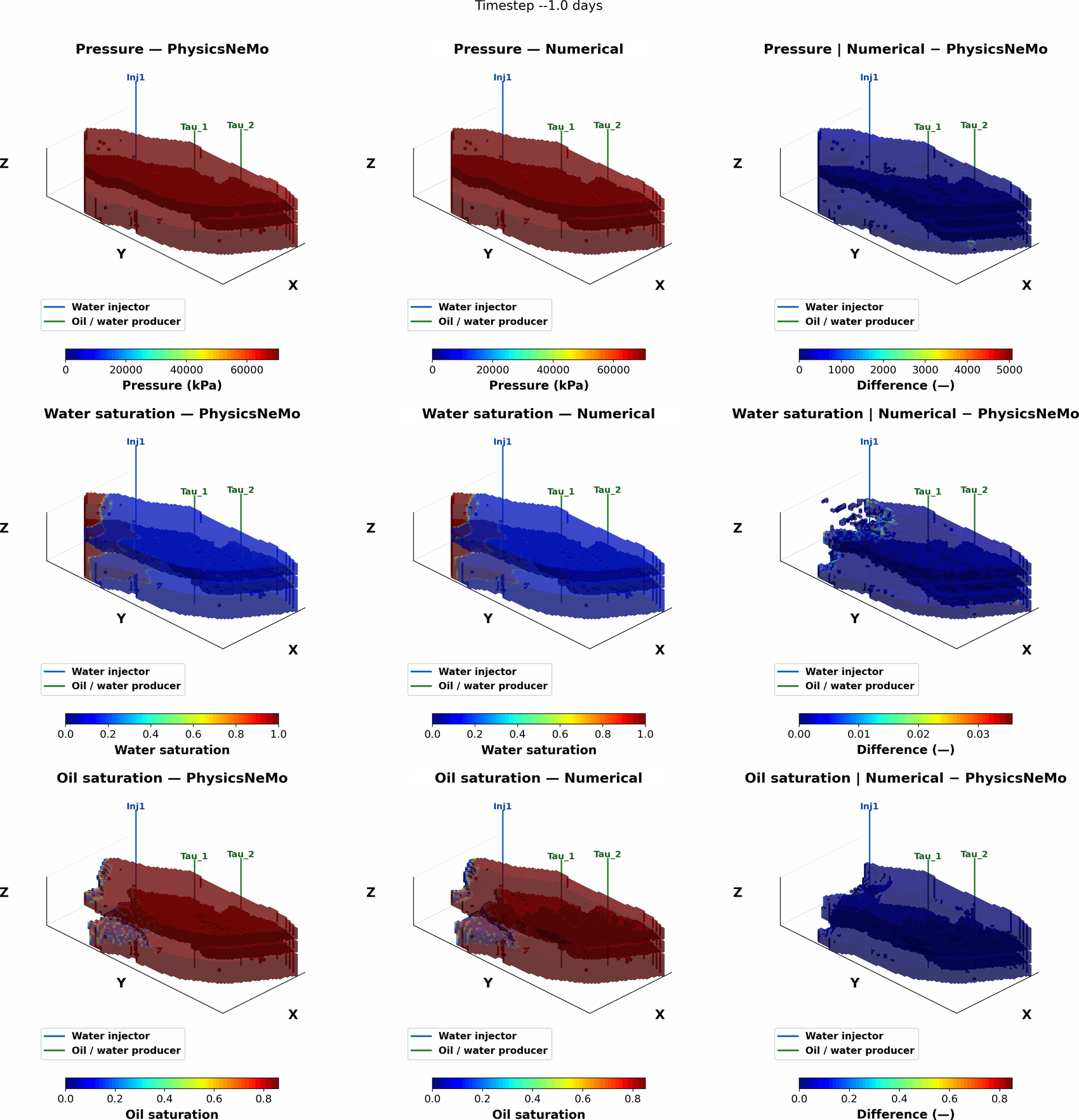}}\caption{$t\!\approx\!1$~d (initial)}\end{subfigure}\hfill
\begin{subfigure}{0.49\linewidth}\fbox{\includegraphics[width=\linewidth]{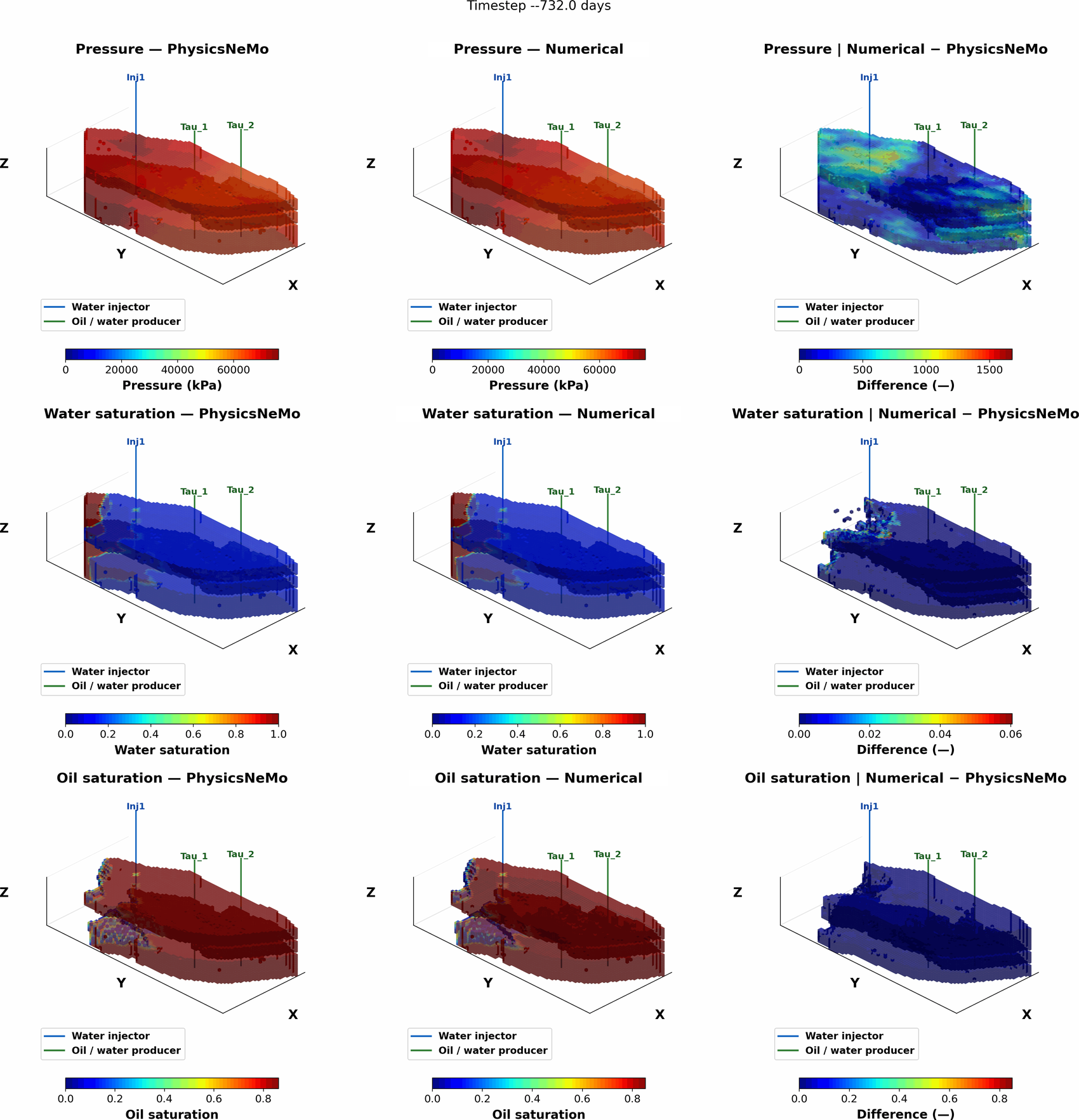}}\caption{$732$~d}\end{subfigure}\\[3pt]
\begin{subfigure}{0.49\linewidth}\fbox{\includegraphics[width=\linewidth]{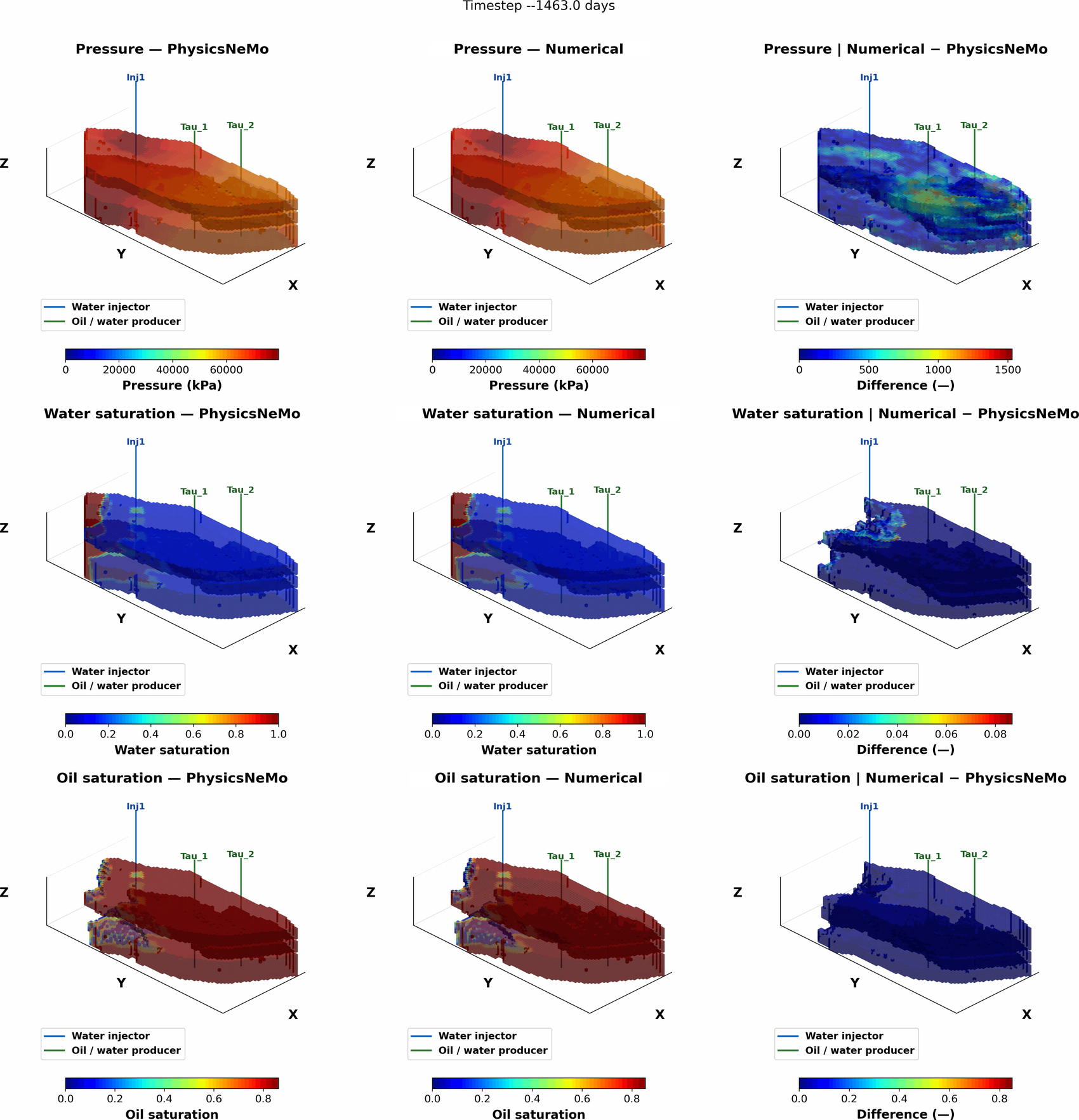}}\caption{$1463$~d}\end{subfigure}\hfill
\begin{subfigure}{0.49\linewidth}\fbox{\includegraphics[width=\linewidth]{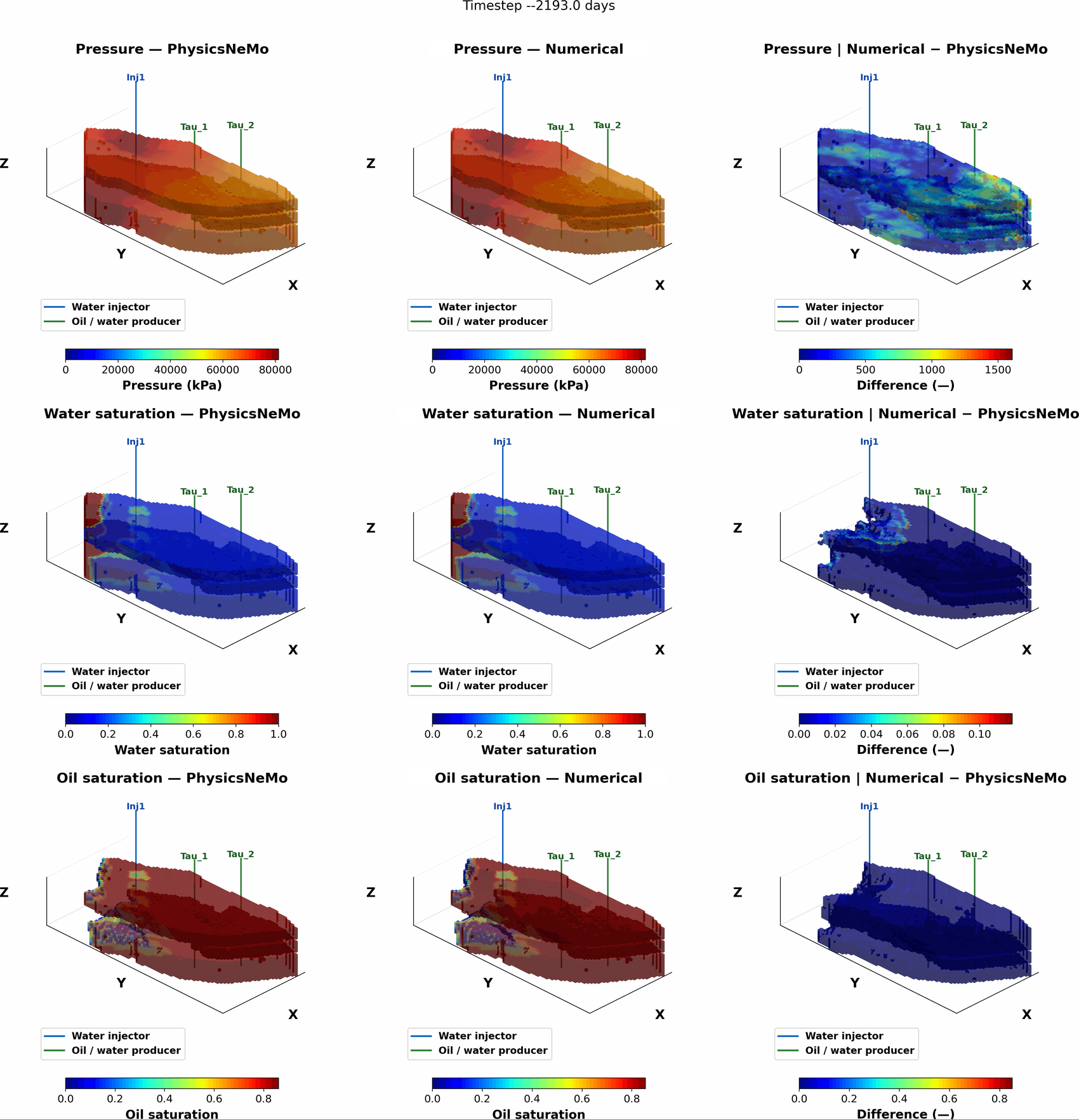}}\caption{$2193$~d (final)}\end{subfigure}
\caption{Three-dimensional field comparison at four report steps spanning the
$\approx\!2200$-day reservoir production history (extracted from the
rollout animation). Within each panel the rows are pressure, water saturation
and oil saturation, and the columns are the autoregressive PINO surrogate
(left), the finite-volume reference (centre) and their absolute difference
(right); the injector \texttt{Inj1} and producers \texttt{Tau\_1},
\texttt{Tau\_2} are marked. Pressure differences stay $<\!2500$~kPa and the
saturation error remains localised at the advancing injector front, growing only
mildly over the rollout: the signature of the bounded Lipschitz amplification
of Proposition~\thmr{prop:rollout} under the $K{=}10$ TBPTT/PINO contraction
(Theorem~\thmr{thm:contract}).}
\label{fig:dyn}
\end{figure}

\begin{figure}[h]\centering
\begin{subfigure}{0.49\linewidth}\fbox{\includegraphics[width=\linewidth]{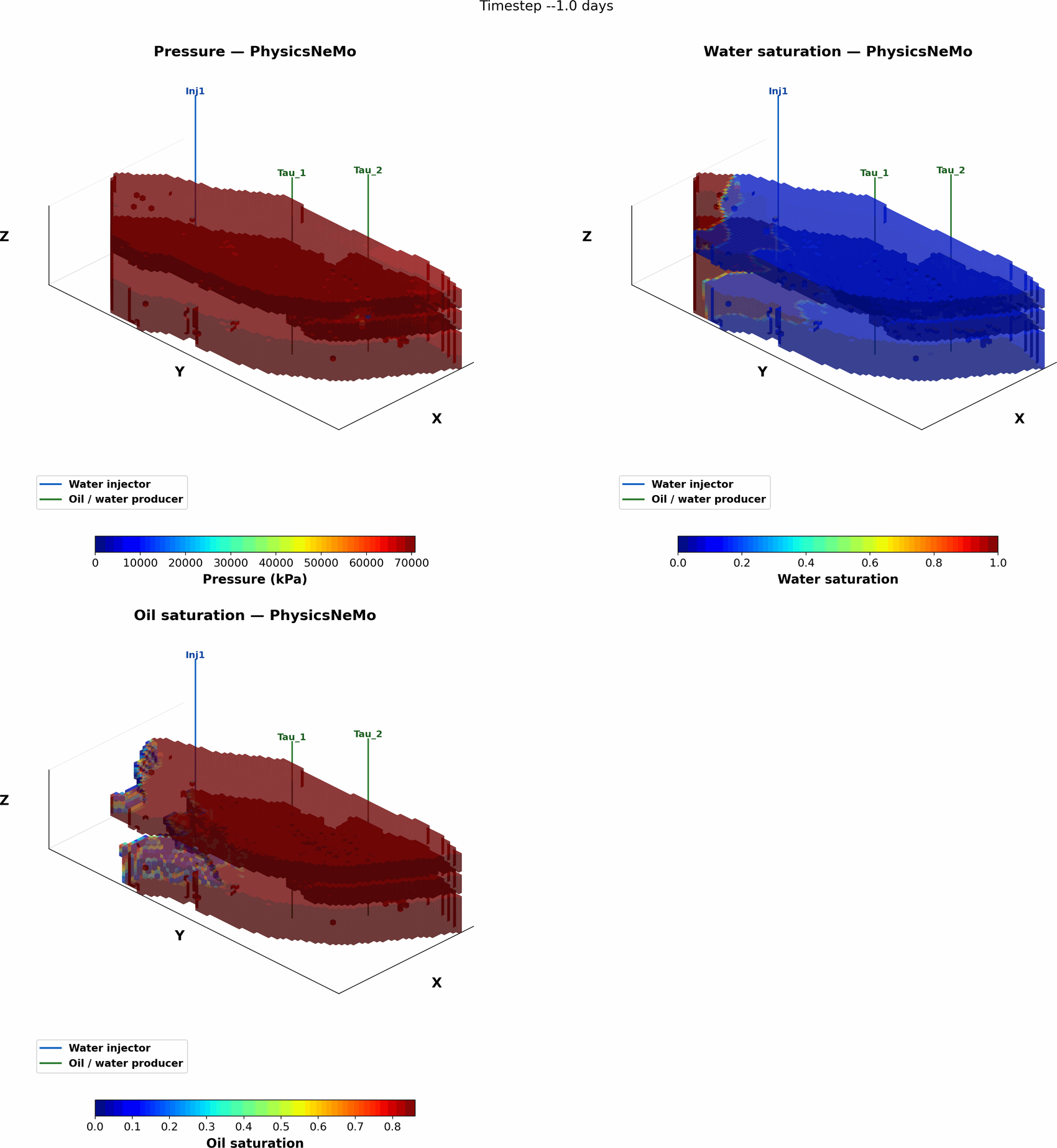}}\caption{$t\!\approx\!1$~d}\end{subfigure}\hfill
\begin{subfigure}{0.49\linewidth}\fbox{\includegraphics[width=\linewidth]{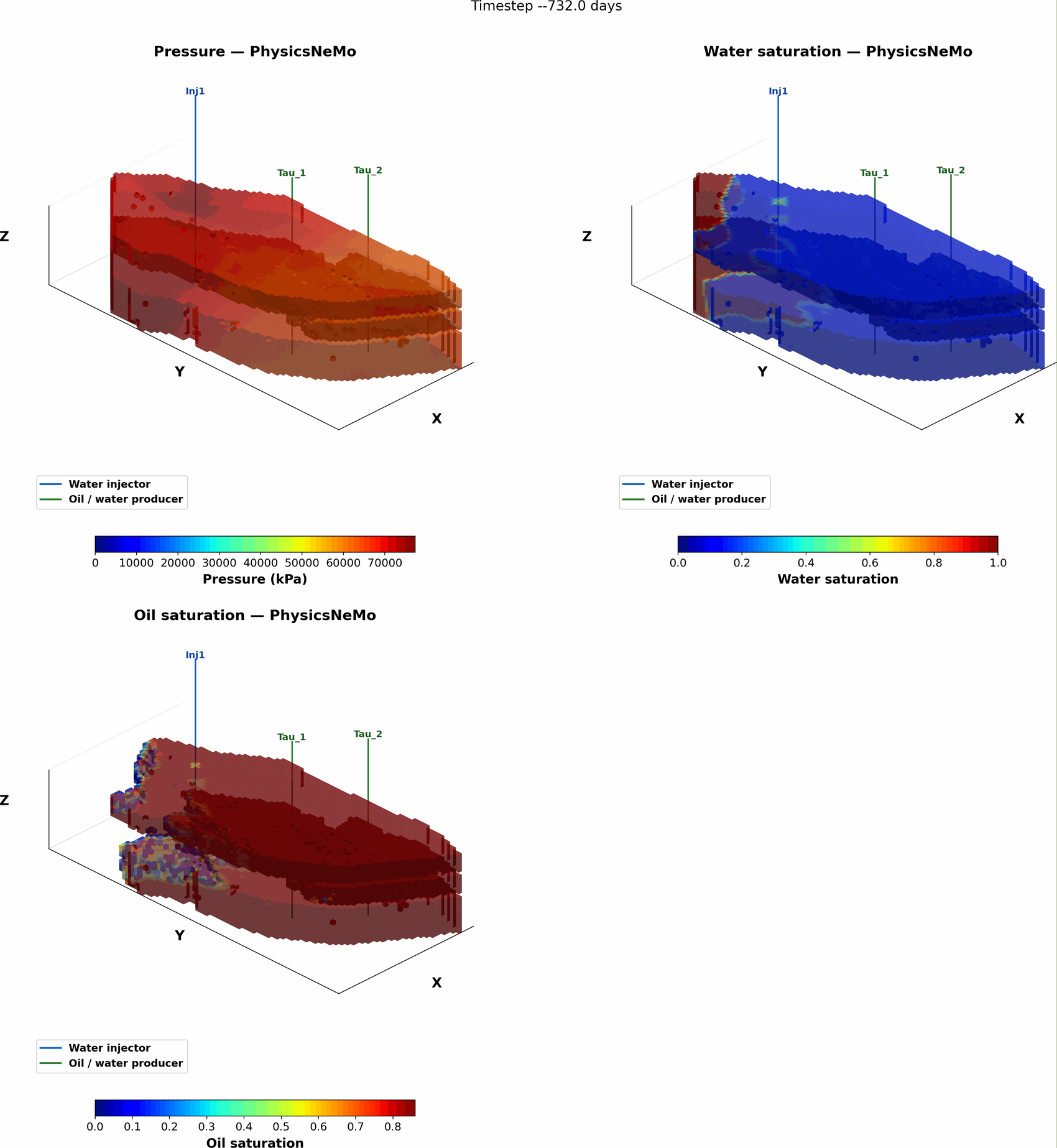}}\caption{$732$~d}\end{subfigure}\\[3pt]
\begin{subfigure}{0.49\linewidth}\fbox{\includegraphics[width=\linewidth]{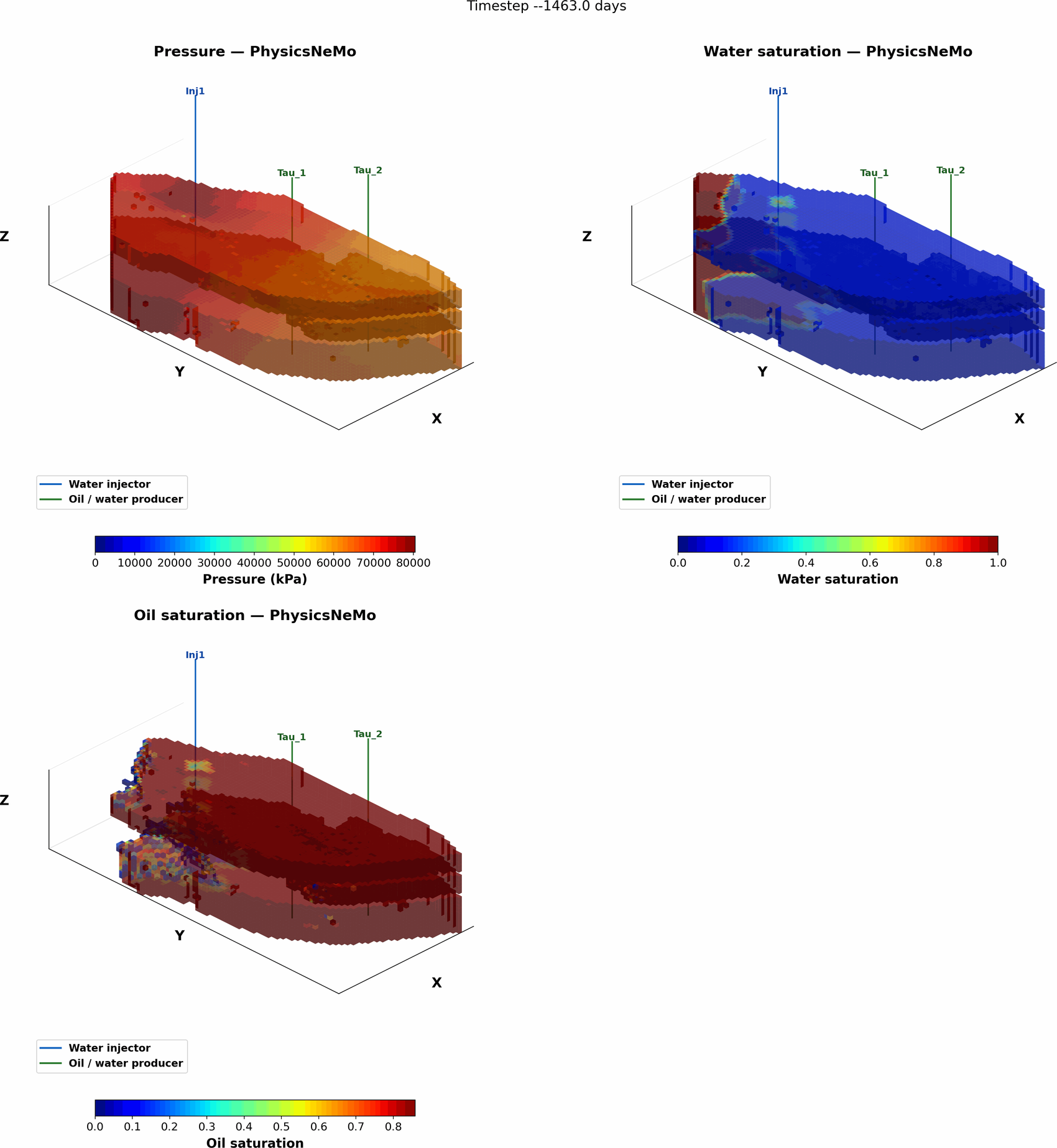}}\caption{$1463$~d}\end{subfigure}\hfill
\begin{subfigure}{0.49\linewidth}\fbox{\includegraphics[width=\linewidth]{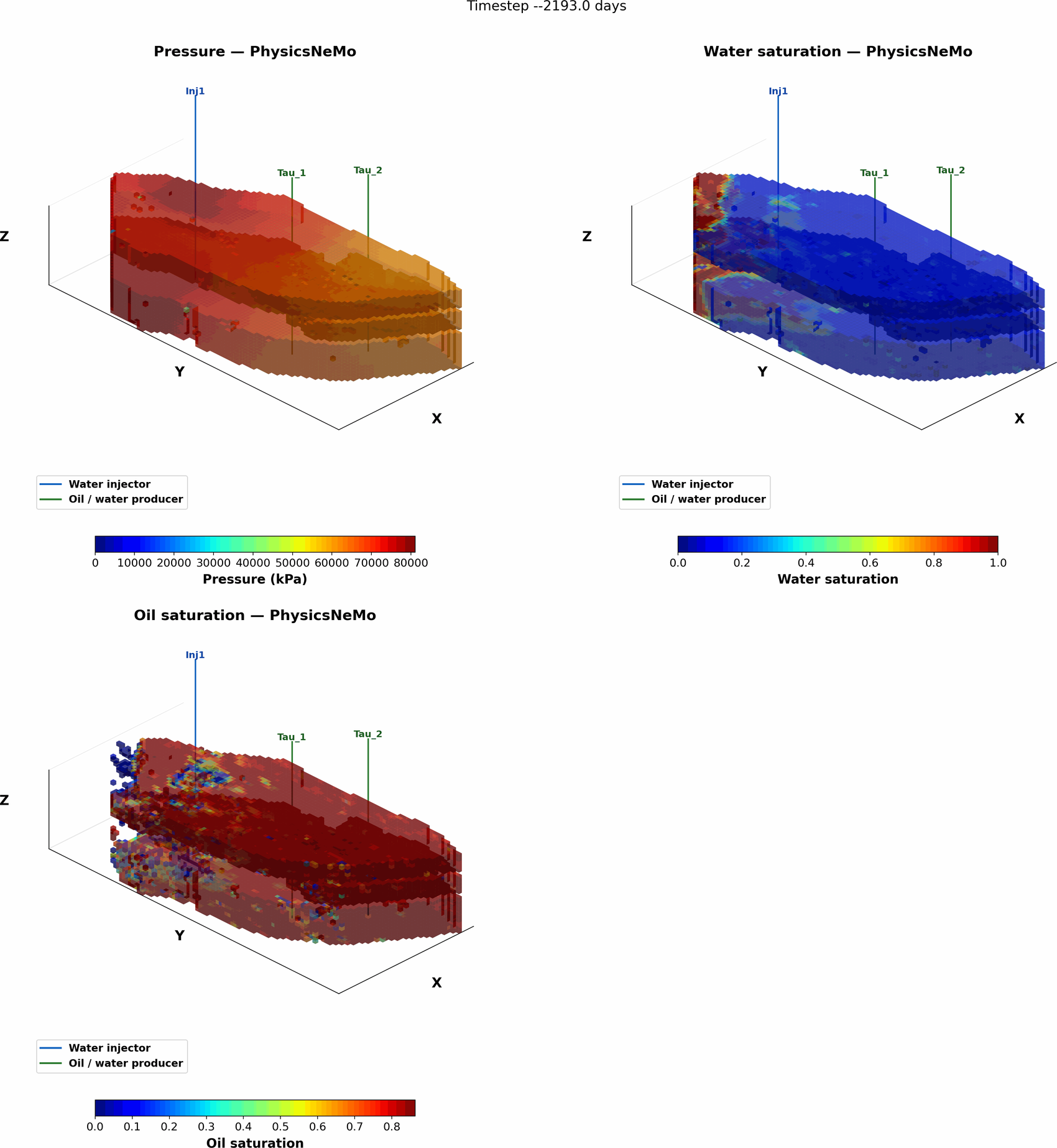}}\caption{$2193$~d}\end{subfigure}
\caption{The autoregressive surrogate's \emph{stand-alone} forward rollout of the
full three-dimensional reservoir-model state, pressure, water saturation and oil
saturation (top to bottom in each panel), predicted in closed loop in
gradient-free mode with \emph{zero} teacher forcing, at four report
steps. Injected water advances from \texttt{Inj1} and sweeps oil toward the
producers \texttt{Tau\_1}, \texttt{Tau\_2}; the pressure field relaxes globally
(elliptic) while the saturation front advects locally (hyperbolic), exactly the
diffusivity hierarchy $\kappa_P\!\gg\!\kappa_S$ of Theorem~\thmr{thm:contract}.
This is the state cube that the Peaceman heads and the Gassmann
petro-elastic map read to form the observables assimilated in
\S\ref{sec:inverse}.}
\label{fig:evo1}
\end{figure}

\subsection{A classical baseline: mixture-of-experts CCR}
\label{sec:ccr}

The fourth stage of the pipeline trains a classical
cluster-classify-regress (CCR) well-model surrogate
\citep{bernholdt2019ccr} as a baseline against the FNO Peaceman heads, and as an
optional in-loop well model. CCR was introduced as a general method for learning
discontinuous functions, and its deep mixture-of-experts form, in which the gate
and the experts are trained jointly, is developed in
\citet{etienam2024gpexperts}. We summarise its mathematics because the
comparison sharpens \emph{why} the operator approach wins.

\subsubsection{The Cluster-Classify-Regress decomposition}
CCR approximates a discontinuous map $x\mapsto y$ (here the near-well
feature$\to$well-response map) by a mixture of local experts gated by a learned
classifier. Given training pairs $\{(x_i,y_i)\}$:
(i) \emph{cluster} the joint space by $K$-means into $\{1,\dots,N_K\}$, with
$N_K$ chosen by the elbow (knee) heuristic on the within-cluster distortion
$D(k)=\frac1n\sum_i\min_c\norm{x_i-\mu_c}$;
(ii) \emph{classify} inputs to clusters by an XGBoost (or random-forest) gate
$\pi:x\mapsto\{1,\dots,N_K\}$;
(iii) \emph{regress} a local expert $r_c$ (polynomial/sparse-GP/XGBoost) per
cluster \citep{bernholdt2019ccr}. The prediction is the hard-gated composition
\begin{equation}
\hat y(x)=\sum_{c=1}^{N_K}\mathbb{1}\!\big[\pi(x)=c\big]\,r_c(x),
\label{eq:ccr}
\end{equation}
which can represent discontinuities the global smooth experts cannot, by
partitioning the domain at the gate boundaries. For the Tau model, the CCR
baseline uses XGBoost
experts with regularised, early-stopped trees.

\subsubsection{Why the operator wins}
The CCR gate consumes only the low-dimensional near-well feature vector (the
$4N_p+2=10$ producer features), so it cannot see the global pressure field that
drives drawdown, nor track the saturation front across the $20$-step rollout.
For WOPR, dominated by pressure drawdown at near-connate $S_w$, and for
WWPR, governed by the breakthrough bifurcation in $K_{rw}(S_w)$, this is
fatal: the comparison records FNO WOPR $R^2>99.99\%$ vs.\ CCR $18.7\%$, and
FNO WWPR $R^2=82.8\%$ vs.\ CCR $-72.4\%$ (worse than the mean predictor),
despite CCR's in-sample $R^2\approx0.999$, a textbook
train/test generalisation gap from fitting $\sim\!4000$ points in $10$
dimensions. The FNO, by contrast, extracts the relevant near-well behaviour
\emph{implicitly} from the full $100{,}188$-cell field it has rolled out, and
evaluates in $0.579$~s vs.\ CCR's $42$~s on $144$ CPU cores. Table~\ref{tab:ccr}
gives the head-to-head numbers from the comparison log.

\begin{table}[h]\centering\small
\caption{FNO Peaceman head vs.\ CCR well model on a held-out test realisation.
FNO inference $0.579$~s (1 GB200); CCR $42$~s (144 CPU cores).}
\label{tab:ccr}
\begin{tabular}{lcccc}
\toprule
\textbf{Quantity} & \textbf{FNO $R^2$} & \textbf{CCR $R^2$} & \textbf{FNO RMSE} & \textbf{CCR RMSE}\\
\midrule
WOPR (m$^3$/day)  & \textbf{$>$99.99\%} & $18.7\%$   & $0.470$ & $63.49$\\
WWPR (m$^3$/day)  & \textbf{82.8\%}  & $-72.4\%$  & $0.0019$& $0.0137$\\
WBHP (kPa, inj.)  & $89.8\%$         & \textbf{90.2\%} & $276.6$ & $270.9$\\
WBHP\_PROD (kPa)  & \textbf{98.0\%}  & $89.5\%$   & $63.4$  & $129.5$\\
\bottomrule
\end{tabular}
\end{table}

Scalable variational and deep-gated treatments of such expert mixtures, which
relax the hard gate of \eqr{eq:ccr}, are developed in
\citet{etienam2024gpexperts}.

The single quantity CCR wins, injector WBHP, is the one dominated by the
local injectivity index $II=2\pi Kh/(\mu_w B_w\ln(r_e/r_w))$, i.e.\ the near-well
permeability the CCR feature captures directly; everywhere the global field
matters, the operator wins decisively. This is the empirical case for replacing
the classical surrogate with the autoregressive operator inside the inversion
loop.

\section{The Bayesian Inverse Problem}
\label{sec:bayes}

We now construct the inverse problem rigorously, because the ensemble methods
of \S\ref{sec:inverse} are consistent approximations of the object built here.

\subsection{Measure-theoretic foundations}
\label{sec:mtf}

Throughout this section we follow the function-space Bayesian framework of
\citet{stuart2010inverse} and \citet{dashti2017bayesian}: the unknown is a
\emph{function}, the posterior is a \emph{measure} on a function space defined
through its Radon--Nikodym derivative with respect to the prior, and every
statement is made so that it survives the mesh-refinement limit $d\to\infty$.
The finite-dimensional (grid) problem actually solved by the ensemble is then a
consistent Galerkin-type approximation of this limit object, and its Lebesgue
density is recovered as a special case.

\subsubsection{Parameter space and the Borel structure}
\label{sec:paramspace}
The unknown is the pair of fields $u=(\log K,\phi):\Omega\to\R^2$. Idealised on
the continuum, $u$ lives in a separable Banach space $\big(\cX,\norm{\cdot}_\cX
\big)$, concretely a Sobolev space $H^s(\Omega;\R^2)$ or a H\"older space
$C^{0,\beta}(\Omega;\R^2)$ chosen so that draws from the prior have the desired
regularity; discretised on the model grid, $u\in\R^{d}$ with
$d=2N=200{,}376$. Separability is essential: it guarantees that $\cX$ is a
Radon space, so every Borel probability measure on it is inner-regular and tight
\citep{bogachev1998gaussian}. We equip $\cX$ with its \emph{Borel
$\sigma$-algebra}
\begin{equation}
\cB(\cX)=\sigma\big(\{\,U\subset\cX : U\ \text{open}\,\}\big),
\label{eq:borel}
\end{equation}
the smallest $\sigma$-algebra containing the open sets. On a separable Banach
space $\cB(\cX)$ coincides with the cylindrical $\sigma$-algebra generated by
the continuous linear functionals $\cX^\ast$, so a measure is determined by its
finite-dimensional projections, the property that makes the ensemble
(particle) representation of \S\ref{sec:inverse} meaningful. The pair
$(\cX,\cB(\cX))$ is the measurable space on which all subsequent probability
measures, prior $\mu_0$, posterior $\mu^y$, and their ensemble
approximations, are defined.

\subsubsection{The Gaussian prior and its Cameron--Martin space}
\label{sec:gaussprior}
The prior over the petrophysical properties is a Borel probability measure
$\mu_0$ on
$(\cX,\cB(\cX))$. We take it Gaussian, $\mu_0=\cN(m_0,C_0)$, with mean
$m_0\in\cX$ and covariance operator $C_0:\cX^\ast\to\cX$ that is self-adjoint,
positive, and \emph{trace-class}; trace-class is precisely the condition under
which a Gaussian measure exists as a countably additive measure on the
infinite-dimensional space $\cX$ (a ``white-noise'' prior with $C_0=I$ does
\emph{not} live on $L^2$ and must be handled through its Cameron--Martin space)
\citep{bogachev1998gaussian,daprato2014stochastic}. Writing $C_0$ through its
spectral decomposition $C_0\varphi_k=\lambda_k^2\varphi_k$ in an orthonormal
basis $\{\varphi_k\}$ of $L^2(\Omega)$, the Karhunen--Lo\`eve expansion
\begin{equation}
u=m_0+\sum_{k\ge1}\lambda_k\,\xi_k\,\varphi_k,\qquad
\xi_k\stackrel{\text{iid}}{\sim}\cN(0,1),\qquad
\sum_{k\ge1}\lambda_k^2=\mathrm{tr}(C_0)<\infty,
\label{eq:kl}
\end{equation}
represents a prior draw and makes the summability requirement explicit. The
\emph{Cameron--Martin space}
\begin{equation}
\cH:=\mathrm{Im}(C_0^{1/2}),\qquad
\langle h_1,h_2\rangle_\cH:=\langle C_0^{-1/2}h_1,C_0^{-1/2}h_2\rangle,
\label{eq:cm}
\end{equation}
is the Hilbert space of ``directions of finite prior precision''; its norm is
the Tikhonov regulariser $\norm{u-m_0}_{C_0^{-1}}^2$ appearing in the MAP
functional \eqr{eq:map}. A defining subtlety of the infinite-dimensional
theory is that $\mu_0(\cH)=0$: a prior draw is almost surely \emph{rougher}
than any element of its own Cameron--Martin space, which is exactly what
endows the samples with geological texture rather than artificial smoothness.
The Cameron--Martin theorem characterises the admissible shifts under which
$\mu_0$ remains equivalent (mutually absolutely continuous): $\mu_0(\cdot-h)
\sim\mu_0$ iff $h\in\cH$, and the Feldman--H\'ajek dichotomy states that two
Gaussian measures are either equivalent or mutually singular
\citep{bogachev1998gaussian}. These facts govern which prior perturbations the
data can inform and underlie the well-posedness theory below. In practice the
prior may instead be a non-Gaussian geostatistical measure, known only through
its samples (\S\ref{sec:prior}); the
construction \eqr{eq:borel}--\eqr{eq:cm} then applies to the Gaussian
reference measure in the latent space of the VCAE chart (\S\ref{sec:param}).

\subsubsection{Forward operator and the data model}
Let $\cG:\cX\to\R^{k}$ be the (surrogate) parameter-to-observable map,
the composition of the closed-loop rollout (Algorithm~\ref{alg:inference})
with the Peaceman heads and, for the joint problem, the Gassmann seismic map
(\S\ref{sec:seismic}). The data are generated by
\begin{equation}
y=\cG(u)+\eta,\qquad \eta\sim\cN(0,\Gamma),\quad \eta\perp u,
\label{eq:datamodel}
\end{equation}
with observation-noise covariance $\Gamma$. The negative log-likelihood is the
\emph{data-misfit potential}
\begin{equation}
\Phi(u;y)=\tfrac12\norm{y-\cG(u)}_{\Gamma^{-1}}^2
=\tfrac12\big(y-\cG(u)\big)^\top\Gamma^{-1}\big(y-\cG(u)\big).
\label{eq:phi}
\end{equation}

\subsubsection{Assumptions on the potential}
\label{sec:assumptions}
The existence and stability of the posterior rest not on the fine structure of
$\cG$ but on four qualitative properties of the potential $\Phi$, stated below
in the form used by \citet{stuart2010inverse} and \citet{dashti2017bayesian}.
Write $\Phi:\cX\times\R^k\to\R$ for the map \eqr{eq:phi}.

\begin{assumption}[Properties of $\Phi$]
\label{ass:phi}
There is a measurable set $\cX'\subseteq\cX$ with $\mu_0(\cX')=1$ such that,
for every $r>0$:
\begin{enumerate}
\item[\textup{(A1)}] \textbf{(Lower bound / positivity)} there is
$M(r)\in\R$ with $\Phi(u;y)\ge M(r)$ for all $u\in\cX'$ and all
$\norm{y}\le r$; in the least-squares case \eqr{eq:phi}, $\Phi\ge0$, so
$M=0$;
\item[\textup{(A2)}] \textbf{(Local boundedness above)} there is $K(r)>0$ with
$\Phi(u;y)\le K(r)$ for all $\norm{u}_\cX\le r$ and $\norm{y}\le r$;
\item[\textup{(A3)}] \textbf{(Local Lipschitz in $u$)} there is $L(r)>0$ with
$|\Phi(u_1;y)-\Phi(u_2;y)|\le L(r)\,\norm{u_1-u_2}_\cX$ for
$\norm{u_1}_\cX,\norm{u_2}_\cX\le r$ and $\norm{y}\le r$;
\item[\textup{(A4)}] \textbf{(Local Lipschitz in $y$)} there is $C(r)>0$ with
$|\Phi(u;y_1)-\Phi(u;y_2)|\le C(r)\big(1+g(u)\big)\norm{y_1-y_2}$ for
$\norm{y_1},\norm{y_2}\le r$, where $g\in L^2_{\mu_0}(\cX)$.
\end{enumerate}
\end{assumption}

\begin{proposition}[The surrogate potential satisfies Assumption~\thmr{ass:phi}]
\label{prop:phiok}
Let $\cG=\cG_\theta$ be the composition of the closed-loop rollout
(Alg.~\ref{alg:inference}) with the Peaceman heads and the Gassmann seismic map
\eqr{eq:gassmann}--\eqr{eq:vp}, and let $\Gamma\succ0$. Then $\Phi$ in
\eqr{eq:phi} satisfies \textup{(A1)--(A4)}.
\end{proposition}
\begin{proof}[Sketch]
$\Phi=\tfrac12\norm{\Gamma^{-1/2}(y-\cG(u))}^2\ge0$ gives (A1) with $M=0$. The
FNO layers \eqr{eq:fno} are compositions of bounded linear maps, truncated
Fourier multipliers, and globally Lipschitz activations, hence $\cG_\theta$ is
Lipschitz on bounded sets of $\cX$; the Peaceman map \eqr{eq:qwell} and the
input-clipped Gassmann map are likewise locally Lipschitz. Boundedness of
$\cG_\theta$ on $\norm{u}_\cX\le r$ then yields (A2), and the identity
$\Phi(u_1)-\Phi(u_2)=\tfrac12\langle \cG(u_2)-\cG(u_1),\,\Gamma^{-1}(2y-\cG(u_1)
-\cG(u_2))\rangle$ combined with local Lipschitzness of $\cG_\theta$ yields (A3).
For (A4), $\partial_y\Phi=\Gamma^{-1}(y-\cG(u))$ is affine in $y$ and its norm
is controlled by $1+\norm{\cG(u)}$, which is $\mu_0$-square-integrable because a
bounded-range surrogate output has all moments; take $g(u)=\norm{\cG(u)}$.
\end{proof}

\subsubsection{Posterior as a Radon--Nikodym density}
\label{sec:rnpost}
With the assumptions in place, the posterior $\mu^y$ is \emph{defined} by its
density with respect to the prior. This is the only definition that makes sense
in infinite dimensions, where there is no Lebesgue measure and hence no ``$\pi(u
\mid y)\propto\pi(y\mid u)\pi(u)$'' to write down; Bayes' theorem is instead the
statement that $\mu^y\ll\mu_0$ with an explicit Radon--Nikodym derivative.

\begin{theorem}[Existence of the posterior, after
\cite{stuart2010inverse,dashti2017bayesian}]
\label{thm:existence}
Let Assumption~\thmr{ass:phi} hold and fix $y$ with $Z(y):=\int_\cX
e^{-\Phi(u;y)}\dd\mu_0(u)$. Then $0<Z(y)<\infty$, the probability measure
$\mu^y$ given by
\begin{equation}
\frac{\dd\mu^y}{\dd\mu_0}(u)=\frac{1}{Z(y)}\exp\!\big(-\Phi(u;y)\big),\qquad
Z(y)=\int_\cX\exp\!\big(-\Phi(u;y)\big)\dd\mu_0(u),
\label{eq:rn}
\end{equation}
is well defined, absolutely continuous with respect to $\mu_0$, and is the
posterior corresponding to the data model \eqr{eq:datamodel}.
\end{theorem}
\begin{proof}[Sketch]
By (A2) there is a ball $B=\{\norm{u}_\cX\le \rho\}$ with $\mu_0(B)>0$ (every
ball has positive Gaussian measure) on which $\Phi\le K(\rho)$; hence
$Z(y)\ge\int_B e^{-K(\rho)}\dd\mu_0=e^{-K(\rho)}\mu_0(B)>0$. By (A1),
$e^{-\Phi}\le e^{-M}<\infty$ so $Z(y)\le e^{-M}<\infty$. Positivity and
finiteness of $Z(y)$ make the right-hand side of \eqr{eq:rn} a bona-fide
$\mu_0$-density, and $\int\dd\mu^y=1$ by construction.
\end{proof}
\subsubsection{Specialisation to a Lebesgue density}
\label{sec:lebesgue}
The abstract density \eqr{eq:rn} reduces to the familiar finite-dimensional
Bayes formula exactly when a Lebesgue reference measure exists, i.e.\ after
discretisation. On the model grid $\cX=\R^d$, the Gaussian prior
$\mu_0=\cN(m_0,C_0)$ has a Lebesgue density
$\pi_0(u)=\big((2\pi)^d\det C_0\big)^{-1/2}\exp\!\big(-\tfrac12\norm{u-m_0}_{
C_0^{-1}}^2\big)$ with respect to $\dd u$, and \eqr{eq:rn} multiplies it by
$e^{-\Phi}$: the posterior has Lebesgue density
\begin{equation}
\pi^y(u)=\frac{1}{Z(y)}\,\exp\!\Big(-\underbrace{\tfrac12\norm{y-\cG(u)}_{
\Gamma^{-1}}^2}_{\text{likelihood }=\Phi}
-\underbrace{\tfrac12\norm{u-m_0}_{C_0^{-1}}^2}_{\text{prior}}\Big),
\label{eq:lebesgue}
\end{equation}
the product of the Gaussian likelihood and Gaussian prior densities. The reason
\eqr{eq:lebesgue} cannot be written in infinite dimensions is that there is no translation-invariant (Lebesgue) measure on an
infinite-dimensional separable Banach space, so ``$\pi_0(u)\,\dd u$'' is
meaningless there; only the \emph{ratio} $\dd\mu^y/\dd\mu_0$ in \eqr{eq:rn}
survives the limit $d\to\infty$. The finite-$d$ density \eqr{eq:lebesgue} is
the projection of the well-defined measure $\mu^y$ onto the grid, and the
mesh-independence of the ensemble solver (\S\ref{sec:inverse}) is a direct
consequence of working with the dimension-robust object \eqr{eq:rn} rather
than \eqr{eq:lebesgue}. The maximum-a-posteriori (MAP) estimator minimises the
negative log-density,
\begin{equation}
u_{\mathrm{MAP}}=\arg\min_{u}\Big\{\tfrac12\norm{y-\cG(u)}_{\Gamma^{-1}}^2
+\tfrac12\norm{u-m_0}_{C_0^{-1}}^2\Big\},
\label{eq:map}
\end{equation}
a Tikhonov-regularised nonlinear least-squares problem whose regulariser is
exactly the prior precision $\norm{\cdot}_\cH^2$ of \eqr{eq:cm}. In infinite
dimensions the same functional is recovered rigorously as the Onsager--Machlup
functional of $\mu^y$, the object whose minimisers are the centres of the
small-ball modes of the posterior \citep{dashti2013map}, with the
Cameron--Martin norm in place of $\norm{u-m_0}_{C_0^{-1}}^2$, which is why the
grid MAP \eqr{eq:map} is a consistent estimator as the mesh is refined.

\subsubsection{Well-posedness in the Hellinger metric}
\label{sec:wellposed}
Existence (Theorem~\thmr{thm:existence}) is not enough: for the inverse problem
to be worth solving, the posterior must depend \emph{continuously} on the data,
so that finite measurement and surrogate errors perturb the answer boundedly.
The natural metric on measures for this purpose is the Hellinger distance,
\begin{equation}
d_{\mathrm{Hell}}(\mu,\mu')=\Big(\tfrac12\int_\cX\Big(\sqrt{\tfrac{\dd\mu}{\dd\nu}}
-\sqrt{\tfrac{\dd\mu'}{\dd\nu}}\Big)^2\dd\nu\Big)^{1/2},
\label{eq:hell}
\end{equation}
for any common dominating measure $\nu$ (e.g.\ $\nu=\mu_0$); it controls the
error in every square-integrable posterior expectation, since
$|\E_{\mu}f-\E_{\mu'}f|\le 2\big(\E_\mu f^2+\E_{\mu'}f^2\big)^{1/2}
d_{\mathrm{Hell}}(\mu,\mu')$ \citep{stuart2010inverse}.

\begin{theorem}[Well-posedness, after \cite{stuart2010inverse,dashti2017bayesian}]
\label{thm:wellposed}
Let Assumption~\thmr{ass:phi} hold. Then for every $r>0$ there is $C=C(r)>0$ such
that, for all $y,y'$ with $\norm{y},\norm{y'}\le r$, the posteriors
$\mu^y,\mu^{y'}$ given by \eqr{eq:rn} satisfy
\begin{equation}
d_{\mathrm{Hell}}\big(\mu^y,\mu^{y'}\big)\le C\,\norm{y-y'}.
\label{eq:hellbound}
\end{equation}
Consequently $y\mapsto\mu^y$ is locally Lipschitz, and every posterior moment
$\E_{\mu^y}f$ with $f\in L^2_{\mu_0}$ is locally Lipschitz in $y$.
\end{theorem}
\begin{proof}[Sketch]
Both $Z(y),Z(y')$ are bounded below by a positive constant depending only on
$r$ (Theorem~\thmr{thm:existence}). Expanding \eqr{eq:hell} with $\nu=\mu_0$,
$2d_{\mathrm{Hell}}^2 = \int\big(Z(y)^{-1/2}e^{-\Phi(u;y)/2}
-Z(y')^{-1/2}e^{-\Phi(u;y')/2}\big)^2\dd\mu_0$. Add and subtract
$Z(y)^{-1/2}e^{-\Phi(u;y')/2}$, and bound the two resulting terms using
(i) the local Lipschitz property (A4), $|e^{-\Phi(u;y)/2}-e^{-\Phi(u;y')/2}|
\le\tfrac12 C(r)(1+g(u))\norm{y-y'}e^{-M/2}$ with $g\in L^2_{\mu_0}$, and
(ii) the elementary estimate $|Z(y)^{-1/2}-Z(y')^{-1/2}|\le\tfrac12
\min(Z(y),Z(y'))^{-3/2}|Z(y)-Z(y')|$ together with
$|Z(y)-Z(y')|\le\int|e^{-\Phi(u;y)}-e^{-\Phi(u;y')}|\dd\mu_0\le
C'(r)\norm{y-y'}$. Squaring and integrating gives \eqr{eq:hellbound}; the
moment statement follows from the display preceding the theorem.
\end{proof}

By Proposition~\thmr{prop:phiok} the surrogate potential satisfies
Assumption~\thmr{ass:phi}, so Theorems~\thmr{thm:existence}
and~\thmr{thm:wellposed} apply verbatim to the \emph{surrogate} posterior:
replacing the full-physics $\cG$ by the composite operator $\cG_\theta$
preserves existence and well-posedness. What the pipeline additionally requires
is that the posterior obtained with the surrogate be close to the one obtained
with the simulator, and that is a separate statement.

\begin{proposition}[Surrogate perturbation of the posterior]\label{prop:surr}
Let $\mu^y$ and $\mu^y_\theta$ denote the posteriors \eqr{eq:rn} built from
$\cG$ and from $\cG_\theta$ respectively, both satisfying
Assumption~\thmr{ass:phi}, and suppose the surrogate is uniformly accurate on a
set of full prior mass,
\begin{equation}
\varepsilon_s:=\big\|\,\norm{\cG(u)-\cG_\theta(u)}_{\Gamma^{-1}}\,\big\|
_{L^2_{\mu_0}(\cX)}<\infty .
\label{eq:surrerr}
\end{equation}
Then there is $C=C(r)$, depending only on the bounds in
Assumption~\thmr{ass:phi} for $\norm{y}\le r$, such that
\begin{equation}
d_{\mathrm{Hell}}\big(\mu^y,\mu^y_\theta\big)\le C\,\varepsilon_s .
\label{eq:surrbound}
\end{equation}
\end{proposition}
\begin{proof}
Write $\Phi$ and $\Phi_\theta$ for the two potentials. By the polarisation
identity for the $\Gamma^{-1}$ inner product,
$|\Phi(u;y)-\Phi_\theta(u;y)|\le\tfrac12\norm{\cG(u)-\cG_\theta(u)}_{\Gamma^{-1}}
\big(2\norm{y}_{\Gamma^{-1}}+\norm{\cG(u)}_{\Gamma^{-1}}
+\norm{\cG_\theta(u)}_{\Gamma^{-1}}\big)$, and the bracket is bounded in
$L^2_{\mu_0}$ by (A2) and (A4). Since $|e^{-a}-e^{-b}|\le|a-b|$ for
$a,b\ge0$, the two unnormalised densities differ in $L^2_{\mu_0}$ by at most a
constant multiple of $\varepsilon_s$, and both normalising constants are bounded
below as in the proof of Theorem~\thmr{thm:existence}. Repeating the two-term
splitting used for Theorem~\thmr{thm:wellposed}, with $\varepsilon_s$ in place of
$\norm{y-y'}$, gives \eqr{eq:surrbound}.
\end{proof}

Inequality \eqr{eq:surrbound} is the licence for the pipeline: it bounds the
error in the \emph{answer} by the error in the forward map, in a metric that
controls every square-integrable posterior expectation. It is conditional, since
$\varepsilon_s$ is not known a priori; \S\ref{sec:results} supplies an estimate
of $\varepsilon_s$ on a held-out test set, and the surrogate posterior inherits
whatever accuracy that estimate certifies.

\subsubsection{Ensemble (particle) approximation}
Sampling $\mu^y$ exactly in $d\!\sim\!2\times10^5$ dimensions is infeasible.
Ensemble Kalman methods (\S\ref{sec:inverse}) approximate $\mu^y$ by a
weighted/unweighted set of particles $\{u^{(j)}\}_{j=1}^{N_e}$ propagated by
the surrogate, replacing the intractable density~\eqr{eq:lebesgue} with its
Gaussian projection characterised by the empirical mean and covariance. The
quality of this Gaussian projection, and the rank deficiency that follows from
$N_e\ll d$, motivate both the reparametrisation of \S\ref{sec:param} and the
localisation of \S\ref{sec:inverse}.

\subsubsection{The prior ensemble}
\label{sec:prior}
The prior $\mu_0$ is represented by an ensemble of $N_e$ realisations of the
petrophysical properties, drawn from a pool of $1400$ geostatistical
realisations supplied with the Tau model (\S\ref{sec:results}). Each
realisation $u^{(j)}$ is flattened in Fortran order to a column of the
$N\times N_e$ ensemble matrix. We treat the pool as a set of samples from an
unknown and generally non-Gaussian prior measure: the ensemble methods of
\S\ref{sec:inverse} require only samples of $\mu_0$, not the algorithm that
generated them. Geologically constructed realisations of this kind can carry
higher-order spatial statistics, such as connected high-permeability geobodies
and sharp facies contacts, that a two-point (Gaussian) covariance does not
represent, and this is the structure the seismic data later sharpen.
It is also the geological content that the VCAE chart (\S\ref{sec:param})
learns to encode, so that the EKI update stays on the manifold of geologically
plausible realisations.

\subsubsection{The observation covariance}
The observation covariance is diagonal and percentage-of-data, with a variance
floor to prevent unbounded weight on near-zero data,
\begin{equation}
\Gamma=\mathrm{diag}\big(\sigma_i^2\big),\qquad
\sigma_i=\max\!\big(r_i\,|d_i|,\ \sigma_i^{\min}\big),\qquad
r_p=0.25,\ r_s=0.30,
\label{eq:gamma}
\end{equation}
which scales noise to each datum's magnitude, equalising signal-to-noise across
multi-decade-range observations (WOPR in m$^3$/day vs.\ WBHP in kPa vs.\
$\Delta V_p$ in m/s). The floor $\sigma_i^{\min}$ is essential for early-time
WWPR, where $|d_i|\!\approx\!0$ before breakthrough would otherwise assign
unbounded weight to precisely the datum the surrogate reproduces least
accurately; we set $\sigma_i^{\min}$ to a small fraction of each response's
peak magnitude. For the joint problem
$\Gamma=\mathrm{blkdiag}(\Gamma_p,\Gamma_s)$; the optional rescaling
$\langle\Gamma_s\rangle\!\approx\!\langle\Gamma_p\rangle$ (joint-noise balancing)
is discussed, and its cost stated, in \S\ref{sec:seismic}.

\subsection{The 4-D seismic forward channel}
\label{sec:seismic}

The enrichment that resolves the inter-well null space is a closed-form
petro-elastic map from the surrogate state to time-lapse seismic velocity.

\subsubsection{What 4-D seismic measures}
A 4-D (time-lapse) seismic survey acquires a baseline and one or more monitor
surveys; their difference images the change in elastic properties caused by
production-induced changes in pressure and fluid saturation. Where water
replaces oil, or pressure drops, the rock's saturated bulk modulus and density
change, shifting the compressional velocity $V_p$. The time-lapse $\Delta V_p$
is therefore a \emph{map of the swept and depleted volume}, precisely the
quantity that fixes inter-well connectivity and is invisible to point-well
production data.

\subsubsection{Gassmann fluid substitution}
For pore pressure $P$, oil saturation $S_o$ and porosity $\phi$, Gassmann's
equation~\citep{gassmann1951} gives the saturated bulk modulus from the dry
frame and fluid:
\begin{equation}
K_{\mathrm{sat}}=K_{\mathrm{dry}}+
\frac{\big(1-K_{\mathrm{dry}}/K_{\mathrm{min}}\big)^2}
{\phi/K_{\mathrm{fl}}+(1-\phi)/K_{\mathrm{min}}-K_{\mathrm{dry}}/K_{\mathrm{min}}^2},
\qquad G_{\mathrm{sat}}=G_{\mathrm{dry}},
\label{eq:gassmann}
\end{equation}
where the mineral modulus $K_{\mathrm{min}}$ is a Voigt--Reuss--Hill
quartz/feldspar average. The shear modulus is fluid-independent (Gassmann's
key assumption). The dry frame is pressure-sensitive (MacBeth
2004,~\cite{macbeth2004}),
\begin{equation}
K_{\mathrm{dry}}(P)=\frac{K_\infty}{1+S_K\,e^{-P_{\mathrm{eff}}/P_c}},\quad
G_{\mathrm{dry}}(P)=\frac{G_\infty}{1+S_G\,e^{-P_{\mathrm{eff}}/P_c}},\quad
P_{\mathrm{eff}}=P_{\mathrm{ob}}-P,
\label{eq:macbeth}
\end{equation}
with $P_{\mathrm{eff}}$ clamped at $0$ for numerical safety. The pore fluid is
a Batzle--Wang~\citep{batzle1992} oil/water mixture under the Reuss
(isostress) average,
\begin{equation}
\frac{1}{K_{\mathrm{fl}}}=\frac{S_w}{K_w}+\frac{S_o}{K_o},\qquad
\rho_{\mathrm{fl}}=S_w\rho_w+S_o\rho_o,\qquad S_w=1-S_o,
\label{eq:fluid}
\end{equation}
with $K_o,\rho_o$ from the pressure-dependent Batzle--Wang oil correlations.
Explicitly, with $P$ in MPa, temperature $T_C$ in $^\circ$C, gas gravity
$G$ and reference oil density $\rho_0=141.5/(\mathrm{API}+131.5)$, we use
the Batzle--Wang dead-oil density and acoustic velocity
\begin{align}
\rho_{\mathrm{oil}}&=\frac{\rho_0+(0.00277P-1.71\!\times\!10^{-7}P^3)(\rho_G-1.15)^2+3.49\!\times\!10^{-4}P}
{0.972+3.81\!\times\!10^{-4}(T_C+17.78)^{1.175}},
\label{eq:bwrho}\\
V_{\mathrm{oil}}&=2096\sqrt{\tfrac{\rho_{\mathrm{ps}}}{2.6-\rho_{\mathrm{ps}}}}
-3.7\,T_C+4.64\,P+0.0115\,T_C P\big(4.12\sqrt{1.08/\rho_{\mathrm{ps}}-1}-1\big),
\label{eq:bwv}
\end{align}
and $K_o=\rho_{\mathrm{oil}}V_{\mathrm{oil}}^2$, while the brine modulus
$K_w=1/C_w$ (in consistent units) uses the configured water compressibility.
The compressional velocity and impedance are
\begin{equation}
V_p=\sqrt{\frac{K_{\mathrm{sat}}+\tfrac43 G_{\mathrm{sat}}}{\rho_b}},\qquad
\rho_b=(1-\phi)\rho_{\mathrm{min}}+\phi\rho_{\mathrm{fl}},\qquad I_p=\rho_b V_p.
\label{eq:vp}
\end{equation}
The map~\eqr{eq:gassmann}--\eqr{eq:vp} is backend-agnostic, input-clipped
against degenerate cells, and \emph{fully differentiable}, so it can sit inside
an automatic-differentiation loop; for ensemble assimilation it is dispatched
across GPUs, evaluated at the two acquisition steps $[2,10]$ (baseline,
monitor), and replaces inactive cells with the active-cell mean to keep the
downstream transform well-behaved.

\subsubsection{The 4-D observable and its compression}
The seismic datum is the depth-averaged absolute time-lapse change
\begin{equation}
\Delta(x,y)=\Big|\,\langle V_p(x,y,\cdot,t_1)\rangle_z
-\langle V_p(x,y,\cdot,t_0)\rangle_z\,\Big|,
\label{eq:delta}
\end{equation}
matching the monitor-minus-baseline convention of field practice. Two
encodings are supported.

\paragraph{Raw ($k_s=n_xn_y=4554$).}
The full areal map $\Delta$ is used directly, then row-scaled to a one-digit
dynamic range so each datum contributes comparably to the misfit. This is the
mode used in all reported inversion and prediction runs ($k_s=4554$, raw mode).

\paragraph{DCT ($k_s=K_1K_2$).}
The 2-D discrete cosine transform $\cD$ of $\Delta$ is truncated to its
leading $K_1\times K_2$ low-frequency block, $\tilde d_s=\Pi_K\cD\,\Delta$,
keeping a fraction $0.05$ of modes. Because $\cD$
is orthonormal (a unitary change of basis), \emph{Parseval's identity} holds:
\begin{equation}
\norm{\cD\Delta_{\mathrm{obs}}-\cD\Delta_{\mathrm{sim}}}_2
=\norm{\Delta_{\mathrm{obs}}-\Delta_{\mathrm{sim}}}_2,
\label{eq:parseval}
\end{equation}
so the innovation norm, and hence the Mahalanobis misfit under a diagonal
$\Gamma_s$, is preserved by the transform. White physical-space noise maps
to white coefficient-space noise, making a diagonal $\Gamma_s$ in DCT space a
defensible default. Truncation discards high-wavenumber modes that carry little
assimilable signal (and that the depth-averaging has already largely removed),
cutting $k_s$ by $\sim\!20\times$ while preserving the swept-volume signature.
The implementation uses the Makhoul FFT-based DCT for even-length axes with a
guardrail enforcing the parity requirement.

\subsubsection{Noise balancing in the joint problem}
Seismic ($V_p\!\sim\!10^3$ m/s) and production (WWPR$\sim\!10^{-2}$ m$^3$/day)
differ by orders of magnitude. Optionally, the seismic block $\Gamma_s$ is
rescaled so its mean diagonal matches that of $\Gamma_p$, equalising the two
modalities' leverage on the Kalman gain. This is a
\emph{regularisation choice, not a correction}: if $\Gamma_p$ and $\Gamma_s$ are
both percentage-of-data covariances, their relative leverage is already
correct, and rescaling deliberately departs from the likelihood of the stated
data model~\eqr{eq:datamodel}, trading statistical fidelity for numerical
balance between an $O(10^2)$-dimensional production channel and an
$O(10^3)$-dimensional seismic channel. Its cost is a posterior that is no longer
the exact Bayesian posterior for $\Gamma$; we report the reported run with this
option and note in \S\ref{sec:limitations} that the seismic information gain
should be shown to survive without it.

Figure~\ref{fig:evoseis} renders the petro-elastic channel in three dimensions:
the compressional velocity $V_p$ synthesised by \eqr{eq:gassmann}--\eqr{eq:vp}
from the surrogate state, at the two seismic acquisition steps (baseline and
monitor) that define the time-lapse difference \eqr{eq:delta}.

\begin{figure}[h]\centering
\begin{subfigure}{0.49\linewidth}\fbox{\includegraphics[width=\linewidth]{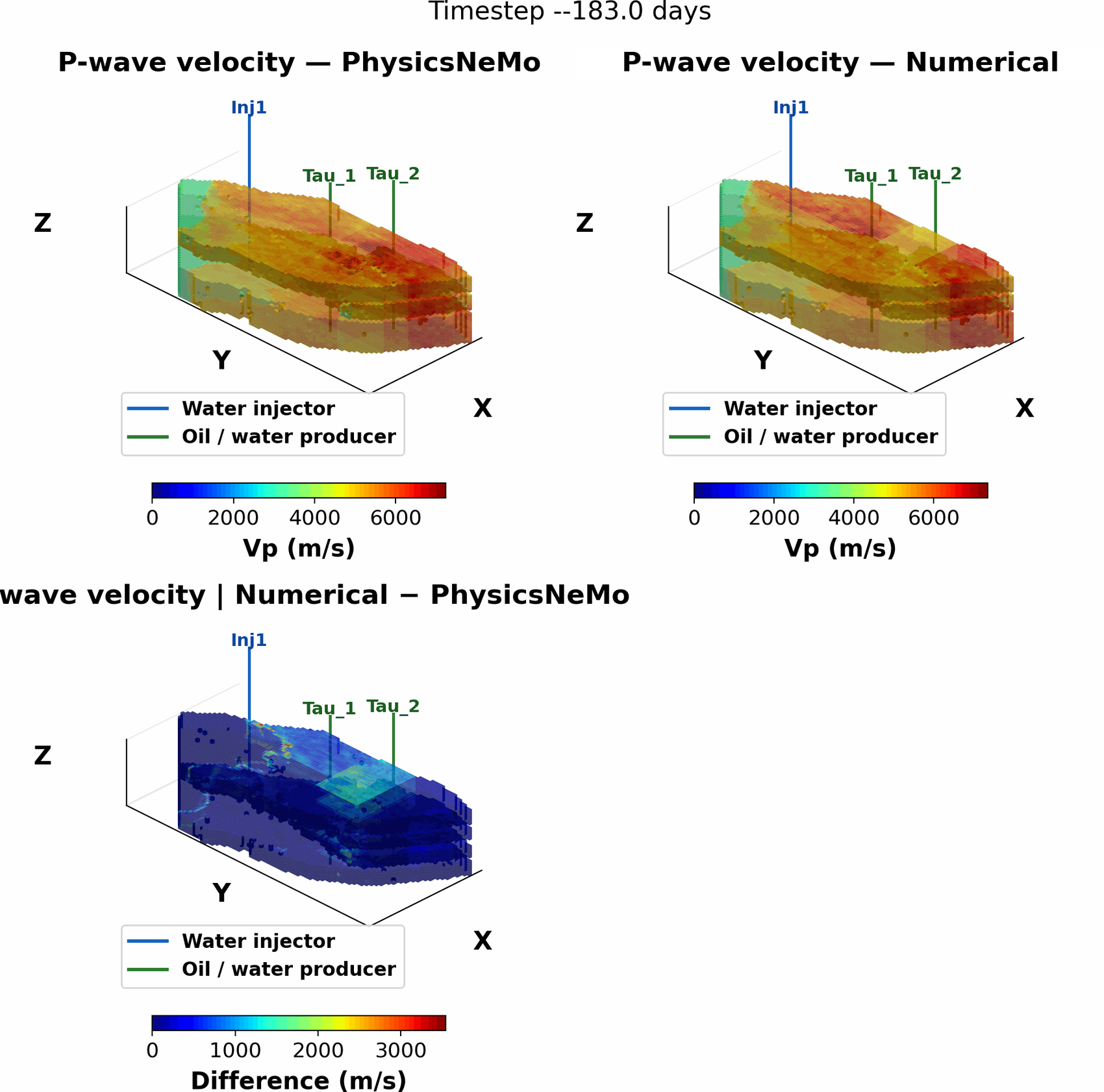}}\caption{Baseline survey, $183$~d}\end{subfigure}\hfill
\hfill
\begin{subfigure}{0.49\linewidth}\fbox{\includegraphics[width=\linewidth]{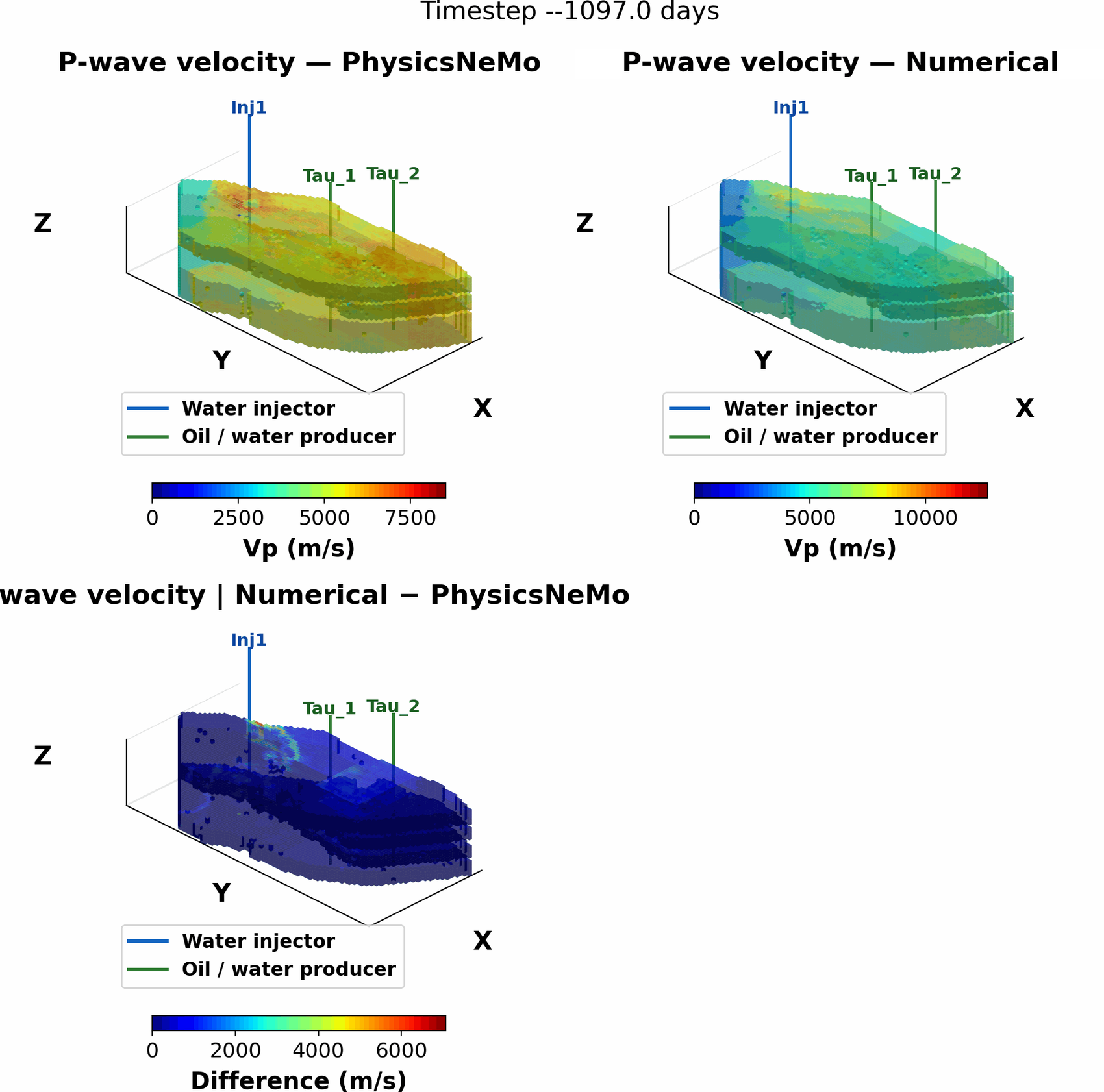}}\caption{Monitor survey, $1097$~d}\end{subfigure}
\caption{Three-dimensional compressional velocity $V_p$ (m/s) of the synthetic
Tau reservoir model at the two 4-D seismic acquisition steps, synthesised by the
closed-form Gassmann petro-elastic transform \eqr{eq:gassmann}--\eqr{eq:vp} from the
autoregressive surrogate's own pressure and saturation cubes. Each panel shows
the surrogate-derived $V_p$ (left), the finite-volume reference (centre) and
their absolute difference (right), with \texttt{Inj1}, \texttt{Tau\_1},
\texttt{Tau\_2} marked. The velocity slow-down in the water-swept region between
baseline and monitor is precisely the time-lapse signature $\Delta V_p$
\eqr{eq:delta} that the joint inversion assimilates to resolve the inter-well
connectivity invisible to production data.}
\label{fig:evoseis}
\end{figure}

\subsection{Parametrisation of permeability and porosity}
\label{sec:param}

\subsubsection{Why reduce dimension at all}
The grid parameter has dimension $d=2N\approx2\times10^5$, but the ensemble
has only $N_e=190$ members. The sample covariance
\begin{equation}
\hat C_{uu}=\frac{1}{N_e-1}\sum_{j=1}^{N_e}(u^{(j)}-\bar u)(u^{(j)}-\bar u)^\top
\in\R^{d\times d}
\label{eq:cuu}
\end{equation}
has rank at most $N_e-1=189\ll d$: it is profoundly rank-deficient. The Kalman
update can only move the ensemble within the $189$-dimensional span of the
prior anomalies, and any spurious long-range correlations in $\hat C_{uu}$
inject artefacts into the update. Two remedies are available: \emph{reduce} the
parameter dimension so the ensemble adequately samples it (this section), and
\emph{taper} the covariance to kill spurious correlations
(localisation, \S\ref{sec:inverse}). The framework supports both the
reparametrisation charts and covariance localisation; for the
reported run, grid-space inversion with $N_e=190$ furnished sufficient rank,
but we document the machinery because it is essential at larger scales.

\subsubsection{DCT parametrisation (linear chart)}
The DCT chart writes the field in a truncated cosine basis,
\begin{equation}
\tilde m=\Pi_K\cD\,m,\qquad m=\cD^{-1}\Pi_K^\top\tilde m,\qquad
K=\big\lceil r_{\mathrm{DCT}}\,n_xn_y\big\rceil,
\label{eq:dctparam}
\end{equation}
with a retention ratio $r_{\mathrm{DCT}}=15\%$ keeping $\sim\!225$ of the $4554$
areal modes. Because
permeability fields are spatially smooth (correlated over
many cells), their energy concentrates in low-wavenumber DCT coefficients, so a
$15\%$ truncation captures most of the variance while collapsing the effective
dimension from $\sim\!10^5$ to $\sim\!10^2$, now \emph{below} $N_e$, so the
reduced sample covariance is full-rank. The chart is linear and orthonormal, so
the prior pushes forward to a Gaussian in $\tilde m$ and the EKI update is
exact in the reduced coordinates. Its limitation: a linear basis cannot
represent sharp, multi-modal facies boundaries without Gibbs ringing.

\subsubsection{VCAE parametrisation (nonlinear chart)}
A variational convolutional auto-encoder learns a \emph{nonlinear} chart from
the prior ensemble,
\begin{equation}
\tilde m=\mathrm{enc}_\theta(m)\in\R^{\mathrm{latent}},\qquad
m=\mathrm{dec}_\theta(\tilde m),
\label{eq:vcae}
\end{equation}
trained by maximising the evidence lower bound
\begin{equation}
\mathcal{L}_{\mathrm{ELBO}}(\theta)=\E_{q_\theta(\tilde m|m)}
\big[\log p_\theta(m|\tilde m)\big]
-\mathrm{KL}\!\big(q_\theta(\tilde m|m)\,\|\,\cN(0,I)\big),
\label{eq:elbo}
\end{equation}
so that the latent prior is standard Gaussian by construction. Inverting in
$\tilde m$ and decoding back guarantees that \emph{every} ensemble member,
prior, intermediate, and posterior, is a geologically plausible draw from
the learned manifold: geobodies stay connected, facies contacts stay sharp, and the
posterior can be genuinely multi-modal (a property a single Gaussian or a
linear DCT basis cannot represent). This is why VCAE is the preferred chart
when the prior is a complex, non-Gaussian geostatistical model: it bakes the
geological prior into the search coordinates, so the Gaussian EKI update in
latent space corresponds to a strongly non-Gaussian, prior-consistent move in
grid space. It requires a pre-trained auto-encoder checkpoint, and is the
default reparametrisation chart.

\begin{remark}[Geometry]
DCT is a fixed orthonormal frame: cheap, exact, but blind to geology beyond
smoothness. VCAE is a learned curved chart: it adapts the coordinate system to
the prior manifold, trading exactness for geological fidelity. Both reduce the
EKI search to a space the ensemble can sample, restoring full rank to the
reduced cross-covariance $C_{up}$.
\end{remark}

\subsection{The inverse engine: aREKI and the two-stage schedule}
\label{sec:inverse}

\subsubsection{Ensemble Kalman Inversion}
EKI~\citep{iglesias2013ensemble} approximates the posterior by an interacting
particle system. With $G^{(j)}=\cG_\theta(u^{(j)})$ and ensemble means
$\bar u,\bar G$, define the empirical cross- and data-covariances
\begin{equation}
C_{up}=\tfrac{1}{N_e-1}\!\sum_j(u^{(j)}-\bar u)(G^{(j)}-\bar G)^\top,\qquad
C_{pp}=\tfrac{1}{N_e-1}\!\sum_j(G^{(j)}-\bar G)(G^{(j)}-\bar G)^\top.
\label{eq:cov}
\end{equation}
The (perturbed-observation) analysis update at iteration $k$ is
\begin{equation}
u^{(j)}\leftarrow u^{(j)}+\underbrace{C_{up}\big(C_{pp}+\alpha_k\Gamma\big)^{-1}}_{=:K_k\ (\text{Kalman gain})}
\big(y+\sqrt{\alpha_k}\,\eta^{(j)}-G^{(j)}\big),\qquad\eta^{(j)}\sim\cN(0,\Gamma).
\label{eq:eki}
\end{equation}
The $k\times k$ solve ($k=k_p$ for production-only, $k_p+k_s$ jointly) is
trivial; cost is dominated by the $N_e$ surrogate rollouts per iteration. It is
important to distinguish the \emph{assimilated} data dimension from the
\emph{predicted} one: the surrogate always rolls out the full horizon $T=20$,
producing $9\!\times\!20=180$ well-response values per member, but only the
first $T_a=15$ report steps enter the misfit, so
$k_p=9\!\times\!15=135$ and the final $5$ steps are never seen by the update
(\S\ref{sec:results}). In the mean-field/continuum limit
the ensemble mean follows the preconditioned gradient flow
$\dot{\bar u}=-C_{uu}D_u\Phi$, so EKI is preconditioned gradient descent on the
misfit~\eqr{eq:phi} in the $C_{uu}$-metric, converging to the regularised
minimiser~\eqr{eq:map} in the linear-Gaussian
limit~\citep{schillings2017convergence}.

\subsubsection{Continuous-time limit and the gradient-flow structure}
\label{sec:gradflow}
To see why~\eqr{eq:eki} is a sound posterior sampler, take the
artificial-time limit $\alpha_k\to\infty$ with $\sum_k\alpha_k^{-1}=1$ and let
$\Delta s=\alpha_k^{-1}$ be the pseudo-time increment. Subtracting $u^{(j)}$
from both sides of~\eqr{eq:eki} and dividing by $\Delta s$, the deterministic
part of the update becomes, to leading order,
\begin{equation}
\frac{u^{(j)}_{k+1}-u^{(j)}_k}{\Delta s}
= C_{up}\,\Gamma^{-1}\big(y-G^{(j)}\big)+\mathcal{O}(\Delta s),
\label{eq:disc}
\end{equation}
since $(C_{pp}+\alpha_k\Gamma)^{-1}\to\alpha_k^{-1}\Gamma^{-1}$. Linearising
the forward map $G^{(j)}=\cG_\theta(u^{(j)})\approx\cG_\theta(\bar u)
+J(u^{(j)}-\bar u)$ gives $C_{up}=\hat C_{uu}J^\top$ and, for the ensemble mean,
the continuous-time gradient flow
\begin{equation}
\frac{\dd\bar u}{\dd s}=-\hat C_{uu}\,J^\top\Gamma^{-1}\big(\cG_\theta(\bar u)-y\big)
=-\hat C_{uu}\,D_u\Phi(\bar u;y),
\label{eq:flow}
\end{equation}
with $\Phi$ the misfit potential~\eqr{eq:phi}. Equation~\eqr{eq:flow} is
\emph{preconditioned gradient descent} on $\Phi$ in the metric induced by the
ensemble covariance $\hat C_{uu}$: the ensemble cannot move outside
$\mathrm{span}\{u^{(j)}-\bar u\}$, but within that subspace it descends the
misfit along the steepest direction in the $\hat C_{uu}$-geometry. In the
linear-Gaussian case the flow~\eqr{eq:flow} converges to the
MAP/Tikhonov minimiser~\eqr{eq:map}; the spread of the particles around the
mean reproduces the posterior covariance~\citep{schillings2017convergence}.
This is the rigorous sense in which aREKI \emph{samples}~\eqr{eq:lebesgue}
rather than merely optimising it.

\subsubsection{Adaptive regularisation: aREKI}
The inflation $\alpha_k$ tempers each update; its adaptive selection is the
``a'' in aREKI and follows \citet{iglesias2021adaptive}, building on the
regularising EKI of \citet{iglesias2016regularising}, as in our earlier
surrogate-based history-matching studies \citep{etienam2024modulus,etienam2024norne}. Writing
$J(m)=\tfrac12\norm{y-\cG_\theta(m)}_{\Gamma^{-1}}^2$ for the per-member misfit
and $N_{\mathrm{op}}=k$ for the observation dimension, we choose
\begin{equation}
\alpha_k=\min\!\Big(\max\big(\alpha_{\mathrm{left}},\alpha_{\mathrm{right}}\big),\;
\big(1-\textstyle\sum_{j<k}\alpha_j^{-1}\big)^{-1}\Big),\qquad
\alpha_{\mathrm{left}}=\frac{N_{\mathrm{op}}}{2\,\langle J\rangle},\;\;
\alpha_{\mathrm{right}}=\sqrt{\frac{N_{\mathrm{op}}}{2\,\mathrm{Var}(J)}}.
\label{eq:alpha}
\end{equation}
The interpretation is the \emph{Morozov discrepancy principle}: at convergence
a well-calibrated misfit satisfies $\E[2J]\approx N_{\mathrm{op}}$ (one unit of
$\chi^2$ per observation). $\alpha_{\mathrm{left}}$ sizes the step so the
\emph{mean} misfit relaxes toward $N_{\mathrm{op}}$; $\alpha_{\mathrm{right}}$
guards against ensemble collapse when the misfit \emph{variance} is large; the
final cap enforces the ES-MDA budget $\sum_k\alpha_k^{-1}=1$ so that the
product of tempered likelihoods reconstructs the full
likelihood,
\begin{equation}
\prod_{k}\exp\!\Big(-\tfrac{\Phi}{\alpha_k}\Big)
=\exp\!\Big(-\Phi\textstyle\sum_k\alpha_k^{-1}\Big)=e^{-\Phi}.
\label{esmdaid}
\end{equation}
The iteration terminates early when the budget is exhausted
($\sum_{l\le k}\alpha_l^{-1}\ge1$). After each update a physics projection
clips $K\in[50,2\times10^4]$ mD and $\phi\in[10^{-3},0.25]$ and resets
sealed-fault zeros, returning the ensemble to the physically admissible set.

\subsubsection{ES-MDA as exact tempered assimilation}
\label{sec:esmda}
The joint phase uses ES-MDA rather than aREKI. ES-MDA assimilates the data
$N_a$ times, each with an \emph{inflated} covariance $\alpha_k\Gamma$ and a
correspondingly inflated observation perturbation
$\varepsilon_k\sim\cN(0,\alpha_k\Gamma)$, under the constraint
$\sum_{k=1}^{N_a}\alpha_k^{-1}=1$. The constraint is exactly what makes the
multiple assimilation equivalent to a single Bayesian update for a linear
Gaussian model: the product of the $N_a$ tempered Gaussian likelihoods is
\begin{equation}
\prod_{k=1}^{N_a}\exp\!\Big(-\tfrac{1}{2\alpha_k}\norm{y-\cG(u)}_{\Gamma^{-1}}^2\Big)
=\exp\!\Big(-\tfrac12\norm{y-\cG(u)}_{\Gamma^{-1}}^2\!\sum_k\alpha_k^{-1}\Big)
=\exp\!\Big(-\tfrac12\norm{y-\cG(u)}_{\Gamma^{-1}}^2\Big),
\label{eq:esmdaexact}
\end{equation}
i.e.\ the full likelihood. Distributing the update over $N_a$ gentler steps
linearises the forward map locally at each step and so handles the
nonlinearity of $\cG_\theta$ far better than a single Gauss--Newton step, while
the \emph{deterministic} schedule $\alpha_k=N_a$ avoids the variance-criterion
degeneracy~\eqr{eq:alpha} that afflicts aREKI on a pre-conditioned ensemble.
This is the precise reason the schedule switches algorithms between phases.

\subsubsection{Covariance localisation}
For $N_e<d$, the rank-deficient $\hat C_{uu}$ carries spurious long-range
sample correlations. Furrer/Gaspari--Cohn localisation~\citep{furrer2007,
gaspari1999} tapers them by a Schur (element-wise) product with a compactly
supported correlation function,
\begin{equation}
K_k^{\mathrm{loc}}=L\odot K_k,\qquad
\rho(r)=\begin{cases}
1-\tfrac53 r^2+\tfrac58 r^3+\tfrac12 r^4-\tfrac14 r^5,&0\le r\le1,\\[2pt]
4-5r+\tfrac53 r^2+\tfrac58 r^3-\tfrac12 r^4+\tfrac1{12}r^5-\tfrac{2}{3r},&1<r\le2,\\
0,&r>2,
\end{cases}\quad r=\frac{\norm{x_i-x_j}}{r_0},
\label{eq:gc}
\end{equation}
which zeroes correlations beyond $2r_0$, raising the effective rank of the
update and suppressing noise-driven updates at distant cells. For the reported
run $N_e=190$ gave sufficient rank and localisation was disabled; it is
recommended for $N_e<100$ or finer grids.

\subsubsection{The two-stage production$\to$joint schedule}
The inversion proceeds in three steps.
\textbf{Phase 1 (production aREKI):} condition the prior on $d_p$ only, using
the adaptive schedule~\eqr{eq:alpha}, which is well-behaved while the misfit
variance is informative.
\textbf{Inter-phase inflation:}
$m_j^{\mathrm{inf}}=\bar m+\gamma(m_j-\bar m)$ with $\gamma=1.01$, giving
$\mathrm{Cov}(m^{\mathrm{inf}})=\gamma^2\mathrm{Cov}(m)$, a $\sim\!2\%$
variance restoration that re-opens the collapsed ensemble spread just enough to
keep the Phase-2 cross-covariance $C_{up}$ full-rank, without discarding
Phase-1 information.
\textbf{Phase 2 (joint ES-MDA):} assimilate the stacked
$d=(d_p,d_s)$, $\Gamma=\mathrm{blkdiag}(\Gamma_p,\Gamma_s)$ (with seismic
balancing), minimising the joint Mahalanobis cost
\begin{equation}
J(m)=\big(g_p(m)-d_p\big)^\top\Gamma_p^{-1}\big(g_p(m)-d_p\big)
+\big(g_s(m)-d_s\big)^\top\Gamma_s^{-1}\big(g_s(m)-d_s\big),
\label{eq:jointcost}
\end{equation}
where $g_s=$ Gassmann$\circ$ rollout (\S\ref{sec:seismic}). ES-MDA's
\emph{deterministic} schedule $\alpha_k=N_a$ is used here because the
Phase-1-conditioned ensemble has collapsed misfit variance, in which regime the
aREKI variance criterion $\alpha_{\mathrm{right}}$ in~\eqr{eq:alpha}
degenerates, a subtle but important robustness choice.

\subsubsection{Posterior selection}
Each phase's posterior is scored by a dimensionless misfit normalised against
the prior RMSE in each modality, tie-break order P2$>$P1. On the Tau model,
$s_{\mathrm{P1}}=2.556$ (production) and $s_{\mathrm{P2}}=0.876$ (joint), so
P2 is selected with gap $1.68$.

\begin{equation}
K_k=C_{up}\big(C_{pp}+\alpha_k\Gamma\big)^{-1}.
\label{eq:gain}
\end{equation}

\begin{algorithm}[H]
\caption{\textsc{TwoStage-aREKI}, joint seismic--production inversion}
\label{alg:areki}
\small
\begin{algorithmic}[1]
\Require prior $\{u^{(j)}\}_{j=1}^{190}$ (optionally in DCT/VCAE coords); $d_p,d_s$; $\Gamma_p,\Gamma_s$; $N_{\mathrm{iter}}{=}10$; $\gamma{=}1.01$
\State \textbf{Phase 1 (aREKI, production):}
\For{$k=1,\dots,N_{\mathrm{iter}}$}
  \State $G^{(j)}\gets g_p(u^{(j)})$ via Alg.~\ref{alg:inference}; $C_{up},C_{pp}\gets$~\eqr{eq:cov}
  \State $\alpha_k\gets$~\eqr{eq:alpha}; $K_k\gets$~\eqr{eq:gain}; update $u^{(j)}$ by~\eqr{eq:eki}; (optional $K_k\!\to\!L\odot K_k$)
  \State \textsc{ClipToBounds}: clip $K,\phi$; reset fault zeros
  \State \textbf{if} $\sum_{l\le k}\alpha_l^{-1}\ge1$ \textbf{break}
\EndFor
\State \textbf{Inflation:} $u^{(j)}\gets\bar u+\gamma(u^{(j)}-\bar u)$
\State \textbf{Phase 2 (ES-MDA, joint):} stack $d{=}(d_p,d_s)$, $\Gamma{=}\mathrm{blkdiag}(\Gamma_p,\Gamma_s)$ (balanced)
\For{$k=1,\dots,N_a$}
  \State $G^{(j)}\gets(g_p(u^{(j)}),\,g_s(u^{(j)}))$, $g_s=$ Gassmann$\circ$rollout; $\alpha_k\gets N_a$; update by~\eqr{eq:eki}; \textsc{ClipToBounds}
\EndFor
\State score P1, P2 (tie-break P2$>$P1); extract P10/P50/P90/MAP/best
\State \Return selected posterior ensemble \& statistics
\end{algorithmic}
\end{algorithm}

\subsubsection{Non-monotone convergence and Jensen bias}
\label{sec:nonmono}
A subtlety of nonlinear EKI worth stating precisely: the ensemble-mean misfit
need not decrease monotonically, even as the best member improves steadily.
Expanding the mean update~\eqr{eq:eki},
\begin{equation}
\bar u_{k+1}=\bar u_k+K_k\big(y-\cG_\theta(\bar u_k)\big)
-K_k\underbrace{\big(\overline{\cG_\theta(u_k)}-\cG_\theta(\bar u_k)\big)}_{\delta_{\mathrm{NL}}\ (\text{Jensen gap})},
\label{eq:jensen}
\end{equation}
the term $\delta_{\mathrm{NL}}$ is the ensemble nonlinearity bias: by Jensen's
inequality $\overline{\cG_\theta(u)}\ne\cG_\theta(\bar u)$ whenever
$\cG_\theta$ is nonlinear, and its sign is indefinite. Hence the
\emph{mean}-RMSE trajectory can transiently rise while individual members
descend the regularised misfit~\eqr{eq:map}; the \emph{best}-member RMSE,
which tracks the unregularised data term, decreases near-monotonically. This is
the standard signature of nonlinear ensemble Kalman dynamics and is benign: the
posterior statistics (P10/P50/P90/MAP) are extracted from the final ensemble,
not from the mean trajectory. The fixed ES-MDA budget
$\sum_k\alpha_k^{-1}=1$ guarantees the tempered updates compose to the full
Bayesian update~\eqr{esmdaid} regardless of the transient behaviour.

\subsubsection{Why joint helps: a linearised information argument}
\label{sec:whyjoint}
The dominance $s_{\mathrm{P2}}\ll s_{\mathrm{P1}}$ is structural, not a tuning
artefact. Linearise the forward map about the current mean, $\cG_\theta(u)
\approx \cG_\theta(\bar u)+J_\bullet(u-\bar u)$, with Jacobian
$J_p\in\R^{k_p\times d}$ (production) and $J_s\in\R^{k_s\times d}$ (seismic).
The Gaussian posterior precision is the prior precision plus the
Fisher information of each modality,
\begin{equation}
C_{\mathrm{post}}^{-1}=C_0^{-1}+\underbrace{J_p^\top\Gamma_p^{-1}J_p}_{\mathcal I_p}
+\underbrace{J_s^\top\Gamma_s^{-1}J_s}_{\mathcal I_s}.
\label{eq:fisher}
\end{equation}
Production data, sampled at three wells over $20$ steps, give
$\mathrm{rank}(\mathcal I_p)\le k_p=O(10^2)$: the production Fisher information
is supported on a low-dimensional subspace, the directions of $u$ to which
the well responses are sensitive, leaving a vast \emph{null space}
$\ker(\mathcal I_p)$ along which the posterior reverts to the prior. The
inter-well permeability connectivity lies largely in this null space, because
two permeability fields with the same near-well values but different inter-well
connectivity produce nearly identical well rates. The seismic map, by contrast, has
$\mathrm{rank}(\mathcal I_s)\le k_s=4554$ and its rows are areal sensitivities
$\partial\Delta V_p(x,y)/\partial u$ that are non-zero precisely in the swept
inter-well region. Hence $\mathcal I_s$ supplies independent constraints across
$\ker(\mathcal I_p)$, and
\begin{equation}
\mathcal I_p\preceq\mathcal I_p+\mathcal I_s
\;\Longrightarrow\; C_{\mathrm{post}}^{\mathrm{joint}}\preceq C_{\mathrm{post}}^{\mathrm{prod}},
\label{eq:loewner}
\end{equation}
in the L\"owner order. This is \emph{elementary} monotonicity of
the Gaussian posterior precision under added data, linearised about the ensemble
mean of a strongly nonlinear map; it is a motivating inequality, not a proof of
dominance for the true nonlinear problem or for real seismic with model error.
It nonetheless captures the mechanism behind the recovered
\texttt{Inj1}$\to$\texttt{Tau\_2} high-permeability corridor (Fig.~\ref{fig:joint_perm}) and the
collapsed WWPR histogram (Fig.~\ref{fig:hist}): water rate is governed by the
swept volume, which seismic sees directly and production sees only after
breakthrough. The empirical demonstration, not the inequality, is the evidence.

\section{Implementation and Results}
\label{sec:impl}

\subsection{Workflow orchestration and distributed computing}
\label{sec:hpc}

\subsubsection{Workflow stages}
The workflow executes as a sequence of six stages (Figure~\ref{fig:workflow}):
(1) optional extraction of the simulator output into tensors; (2) autoregressive
forward training; (3) evaluation of the training diagnostics; (4) the MoE/CCR
comparison surrogate; (5) the comparison of the surrogate against the
finite-volume reference; and (6) the two-stage inverse problem. A run is
specified by three choices, namely the run mode, whether training is
distributed, and the operator type, together with the number of GPU ranks
$1\le R\le$ \#GPUs. Only the forward stage and, optionally, the comparison and
inverse stages use $R>1$ ranks; the remaining stages run on a single rank. Each
stage is timed and logged, and a failure is reported against the stage in
which it occurred.

\begin{figure}[h]\centering
\resizebox{0.93\textwidth}{!}{%
\begin{tikzpicture}[node distance=0.5cm and 1.0cm]
\node (start)[startstop]{Start};
\node (yaml)[proc,below=of start]{Run specification\\(mode, distributed training, operator type)};
\node (gpu)[io,below=of yaml]{Detect GPUs; set $R$; initialise NCCL};
\node (d1)[dec,below=of gpu]{Extract\\data?};
\node (extract)[proc,left=1.4cm of d1]{Data extraction\\simulator output$\to$tensors};
\node (fwd)[proc,below=0.7cm of d1]{Forward training, AR-PINO, $K$-TBPTT\\$R\times$GB200 (DDP/NCCL)};
\node (plot)[proc,below=of fwd]{Training diagnostics\\13 loss components};
\node (moe)[proc,below=of plot]{MoE / CCR baseline\\Cluster-Classify-Regress};
\node (cmp)[proc,below=of moe]{FV comparison\\$R^2/L^2$ vs FV; Gassmann seismic};
\node (inv)[proc,right=1.6cm of cmp]{Inverse problem\\Two-stage aREKI$\to$ES-MDA};
\node (post)[io,above=of inv]{Posterior P10/P50/P90/MAP/best;\\score P1 vs P2 (select P2)};
\node (stop)[startstop,above=of post]{Timestamped logs; End};
\draw[arrow](start)--(yaml);
\draw[arrow](yaml)--(gpu);
\draw[arrow](gpu)--(d1);
\draw[arrow](d1)--node[above]{\scriptsize Yes}(extract);
\draw[arrow](extract)|-(fwd);
\draw[arrow](d1)--node[right]{\scriptsize No}(fwd);
\draw[arrow](fwd)--(plot);
\draw[arrow](plot)--(moe);
\draw[arrow](moe)--(cmp);
\draw[arrow](cmp)--(inv);
\draw[arrow](inv)--(post);
\draw[arrow](post)--(stop);
\end{tikzpicture}}
\caption{End-to-end orchestration of the pipeline. Each box is a distributed
stage logged to a timestamped directory; the forward stage runs data-parallel
over $R$ GB200 GPUs, the inverse stage drives the two-stage aREKI$\to$ES-MDA
joint inversion.}
\label{fig:workflow}
\end{figure}
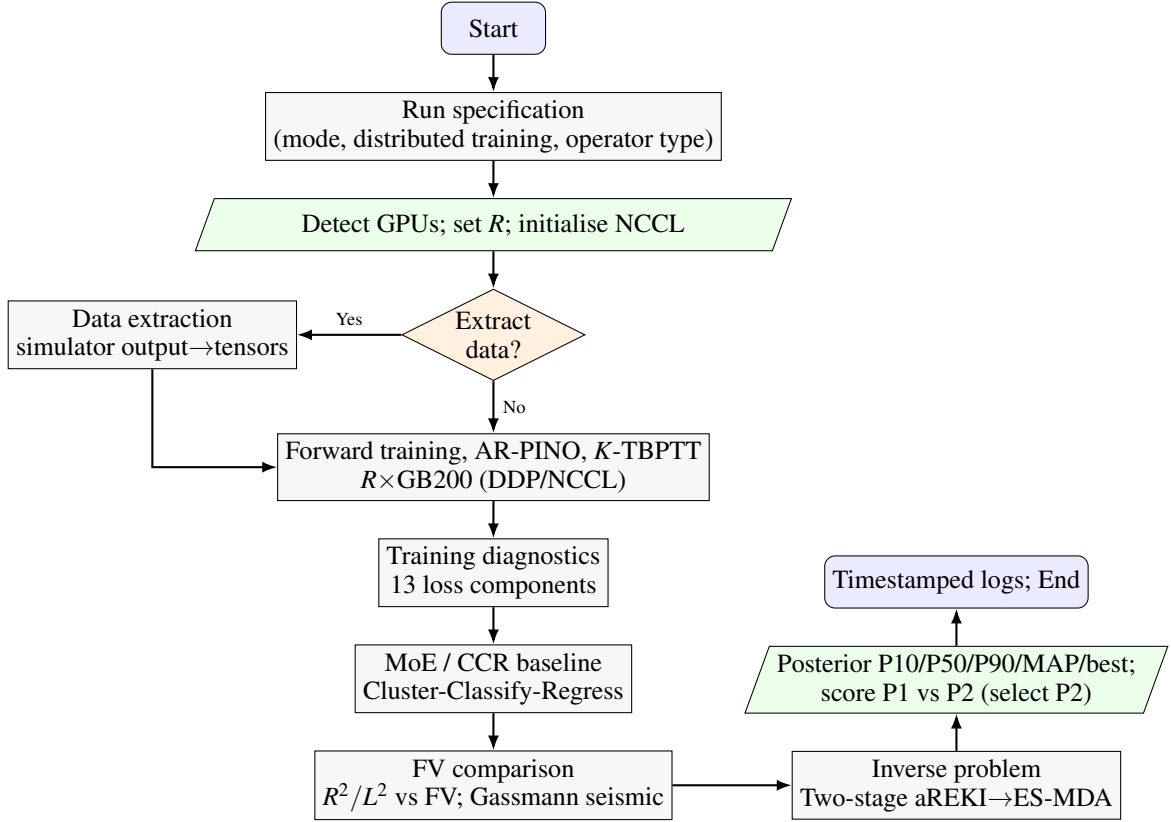

Algorithm~\ref{alg:master} collects the entire mathematical pipeline, from
the raw simulation ensemble to the selected posterior, in a single
end-to-end statement, cross-referencing the equations and per-stage algorithms
that define each step. It is the formal counterpart of the cartoon in
Figure~\ref{fig:cartoon}.

\begin{algorithm}[H]
\caption{\textsc{iSIM} (integrated Surrogate Inversion and Modelling), the
end-to-end surrogate-accelerated Bayesian history-match pipeline of this paper}
\label{alg:master}
\small
\begin{algorithmic}[1]
\Require reservoir-simulation ensemble $\{(u^{(j)},x_{0:T}^{(j)})\}_{j=1}^{200}$; observed data
$d=(d_p,d_s)$; noise $\Gamma=\mathrm{blkdiag}(\Gamma_p,\Gamma_s)$; $N_e{=}190$, $N_a{=}10$
\Statex \textbf{\emph{Stage A, data \& normalisation}} \hfill (\S\ref{sec:blackoil})
\State parse simulator outputs $\to$ tensors $\{P,S_w,S_o\}\in\R^{200\times T\times n_xn_yn_z}$; cache maxima, static $(\log K,\phi,f)$
\Statex \textbf{\emph{Stage B, train the autoregressive PINO surrogate}} \hfill (\S\ref{sec:forward}, Alg.~\ref{alg:train})
\State $\theta^\star\!\gets\!\arg\min_\theta\ \cL_{\mathrm{AR}}(\theta)+\lambda_R\,\E\!\sum_n\norm{R(\hat x_{n+1},\hat x_n)}^2$ \Comment{closed loop, $K$-TBPTT, noise~\eqr{eq:noise}}
\State \quad distributed over $R$ GB200 GPUs by DDP/NCCL all-reduce~\eqr{eq:ddp}
\Statex \textbf{\emph{Stage C, forward observables}} \hfill (\S\ref{sec:forward},\,\ref{sec:seismic})
\State define $g_p(u)=\Pi_{\mathrm{well}}\!\circ\cG_{\theta^\star}^{(T)}(u)$ (Peaceman) and $g_s(u)=\mathrm{Gassmann}\!\circ\cG_{\theta^\star}^{(T)}(u)$~\eqr{eq:delta}
\State $\cG(u)\gets(g_p(u),g_s(u))$; \quad potential $\Phi(u;d)=\tfrac12\norm{d-\cG(u)}_{\Gamma^{-1}}^2$~\eqr{eq:phi}
\Statex \textbf{\emph{Stage D, prior ensemble}} \hfill (\S\ref{sec:mtf})
\State draw $\{u^{(j)}\}_{j=1}^{N_e}\sim\mu_0$ from the prior realisation pool; (optional DCT/VCAE chart~\eqr{eq:dctparam}--\eqr{eq:vcae})
\Statex \textbf{\emph{Stage E, two-stage Bayesian inversion}} \hfill (\S\ref{sec:inverse}, Alg.~\ref{alg:areki})
\State \textbf{Phase 1:} aREKI on $d_p$ with adaptive $\alpha_k$~\eqr{eq:alpha}; \textsc{ClipToBounds}
\State inflate $u^{(j)}\!\gets\!\bar u+\gamma(u^{(j)}-\bar u)$, $\gamma{=}1.01$
\State \textbf{Phase 2:} ES-MDA on $d=(d_p,d_s)$, $N_a$ steps $\alpha_k{=}N_a$, gain~\eqr{eq:gain}
\Statex \textbf{\emph{Stage F, posterior}} \hfill (\S\ref{sec:results})
\State score P1 vs P2 (select lower); extract P10/P50/P90/MAP/best from $\mu^y$
\State \Return posterior ensemble approximating $\dd\mu^y/\dd\mu_0\propto e^{-\Phi}$~\eqr{eq:rn}
\end{algorithmic}
\end{algorithm}

\subsubsection{PyTorch DistributedDataParallel}
The forward stage trains under PyTorch \emph{DistributedDataParallel} (DDP),
the standard single-program multiple-data paradigm. Each of the $R$ ranks owns
one GPU and a full replica of the composite model; the launcher establishes the
process group via a master address/port rendezvous and assigns ranks. The
dataset is sharded by a distributed sampler so each rank sees a disjoint
mini-batch slice. In
the backward pass, DDP \emph{all-reduces} the per-rank gradients so every
replica applies the same averaged update, keeping the replicas bit-identical:
\begin{equation}
g\leftarrow\frac1R\sum_{r=1}^{R}g_r,\qquad
\theta\leftarrow\theta-\eta\,g,
\label{eq:ddp}
\end{equation}
which is mathematically an $R$-fold larger effective batch with linear weak
scaling. The reduction is performed by NVIDIA's NCCL collective library, which
overlaps gradient communication with backward computation by bucketing
gradients and launching the all-reduce as soon as each bucket fills, so the
gradient payload is reduced in the shadow of the $19.9$~s/epoch compute, making
communication essentially free. Concretely, the six $3$-D heads carry
$\approx\!2\times10^{8}$ parameters in total (\S\ref{sec:complexity}), i.e.\ an
FP32 gradient payload of $\approx\!800$~MB per all-reduce. On the GB200
node the all-reduce rides NVLink-5 (1.8~TB/s/GPU) through the NVSwitch
crossbar, with NCCL configured over four InfiniBand HCAs.

A ring all-reduce of an $M$-byte payload over $R$ ranks transfers $2(R-1)M/R$
bytes per rank, asymptotically independent of $R$, so the gradient communication
is $\mathcal{O}(M)$ rather than $\mathcal{O}(RM)$ and, at $M\!\approx\!800$~MB
over NVLink-5, costs $\mathcal{O}(1)$~ms per reduction, small relative to the
per-step compute and overlapped with the backward pass. We therefore expect
near-ideal weak scaling but do not claim a measured efficiency here: a proper $1/2/4$-GPU timing study is left to future
work, and the $\approx\!19.9$~s/epoch figure is reported at the four-GPU
configuration only.

\subsubsection{CUDA, mixed precision and graphs}
Each replica's forward/backward executes as CUDA kernels on the Blackwell GPU.
Three features matter for the AR-TBPTT loop. (i) \emph{Mixed precision}: the
Transformer-Engine FP8/BF16 path accelerates the dense FNO channel mixing while
keeping the spectral multiply and loss accumulation in higher precision for
stability. (ii) \emph{Memory}: weights, gradients, optimiser state and the
retained rollout activations together remain a small fraction of the $192$~GB
HBM3e, so no activation checkpointing or domain decomposition is needed; the
TBPTT detach at every $K=10$ steps caps the retained backward graph to one
window rather than the full $T=20$ horizon, which is what keeps the activation
footprint bounded. (iii) \emph{CUDA
graphs}: when enabled, the per-epoch kernel sequence is captured once and
replayed, removing CPU launch latency from the critical path. With $T=20$ and
$K=10$ the rollout contains two TBPTT windows and hence two backward passes per
step-through, so the launch overhead the capture eliminates is
$\mathcal{O}(\text{kernels per rollout})$ repeated over $700$ epochs. The
NaN/Inf-guarded optimiser step and
unit-norm gradient clipping protect the long run from a single corrupt batch.
The inverse stage runs single-GPU but batches all $N_e=190$ surrogate rollouts,
and dispatches the Gassmann transform across GPUs by multi-process spawning.

\subsubsection{Software stack}
The forward stage is built on NVIDIA PhysicsNeMo
\citep{physicsnemo} with Hydra/OmegaConf
configuration, MLflow tracking of all $13$ loss components, and a distributed
launcher; a fixed seed, NaN/Inf-safe optimiser and unit-norm clipping make the
$700$-epoch run repeatable. The complete workflow, from data extraction through
joint inversion, runs as a single four-rank distributed job on $4\times$GB200.

\subsection{Inverse modelling results}
\label{sec:results}

\subsubsection{Data split and run configuration}
The Tau model ensemble pool contains $1400$ prior realisations. Of these,
\textbf{$200$} are used for the surrogate, partitioned into
\textbf{$190$ training, $5$ validation} (early-stopping) and \textbf{$5$ held-out
test} realisations; the forward-accuracy metrics below are reported on the
$5$-member test set, which the operator never sees during training. The
inversion's \textbf{$N_e=190$} prior members are a \emph{different} set of
the same size, drawn from the \emph{disjoint remaining $1200$} realisations: they
share no member with the surrogate's $200$, so the inverse results are strictly
out-of-sample, and the truth realisation is additionally excluded from the prior
ensemble. We note plainly that a $5$-member test set bounds how strongly
generalisation can be asserted; the error \emph{distribution} over a larger
held-out set is the natural strengthening (\S\ref{sec:limitations}).

\paragraph{Assimilation window and blind forecast.}
The horizon splits as $20=15+5$. The observation vector presented to both
inversions covers only the first $T_a=15$ report steps, giving an assimilated
production dimension $k_p=9\!\times\!15=135$; the last $5$ steps are withheld
from the misfit entirely. The forward model, however, is always evaluated over
the full $T=20$ steps: every ensemble member is rolled out to $2193$~days, so
the trailing $5$ steps are produced by the surrogate but never conditioned on,
and therefore constitute a genuine \textbf{blind forecast}. Assessing the
posterior over that window is the operationally relevant test of a history
match, and it is what Figs.~\ref{fig:wwpr_compare} and \ref{fig:bhp_ensemble}
display: the dashed vertical line is the boundary, everything to its left is
conditioned on, everything to its right is predicted.

\paragraph{Configuration.}
Both inversions used $N_e=190$, $N_{\mathrm{iter}}=10$, the FNO surrogate, raw
seismic ($k_s=4554$), $25\%$ production / $30\%$ seismic noise, no
localisation, grid-space parameters. The production-only run completed in
$2303$~s; the two-stage joint run in $3151$~s and selected P2.

\paragraph{Inverse crime and scope.}
The seismic ``observations'' are synthesised by the same Gassmann code, with the
same mineral moduli, MacBeth pressure coefficients and Batzle--Wang correlations
used inside the inversion; there is no petro-elastic model error, no NRMS
repeatability noise, and no time-shift, registration or depth-conversion error.
The Tau model is itself synthetic, a realistic model of a fault-dissected
clastic reservoir rather than a producing field, so the history-matching
experiment is \emph{synthetic} end to end. The consequences of
this, and the further validation the results require, are stated in
\S\ref{sec:limitations}.

\paragraph{End-to-end speed-up.}
The $\sim\!2500\times$ figure is \emph{per rollout}. The end-to-end accounting is
less dramatic: generating the $200$ full-physics runs that supply
the $190/5/5$ split costs $\mathcal{O}(200\times5\text{ min})\approx17$~h of
simulation, and training costs $\approx4$ GPU-hours; amortised against
$\approx1900$ full-physics solves
($\sim\!160$~h) for a single inversion, the end-to-end acceleration is
\emph{single-digit}. The surrogate pays off only when it is reused across many
inversions/scenarios; we therefore also report the per-call number, which is the
quantity relevant to reuse. Note too that the Peaceman heads have a
well-count-specific feature layout ($N_p=2$ producers), so the trained surrogate
does \emph{not} transfer as-is to a different well count or grid.

\subsubsection{Joint vs.\ production-only: posterior comparison}
\label{sec:jvp}
Figures~\ref{fig:perm_compare}--\ref{fig:wwpr_compare} make the central
finding visual. Figure~\ref{fig:perm_compare} places the production-only
posterior permeability (left $2\!\times\!2$ panel: MAP / best / prior / truth)
beside the joint posterior (right): the production-only MAP smears the
inter-well region and misses the \texttt{Inj1}$\to$\texttt{Tau\_2} high-permeability corridor,
whereas the joint MAP recovers the corridor and fault block in agreement with
the true model, the inter-well structure that the Fisher-information
argument~\eqr{eq:loewner} predicts seismic can resolve and production cannot.

\begin{figure}[h]\centering
\begin{subfigure}{0.49\linewidth}\includegraphics[width=\linewidth]{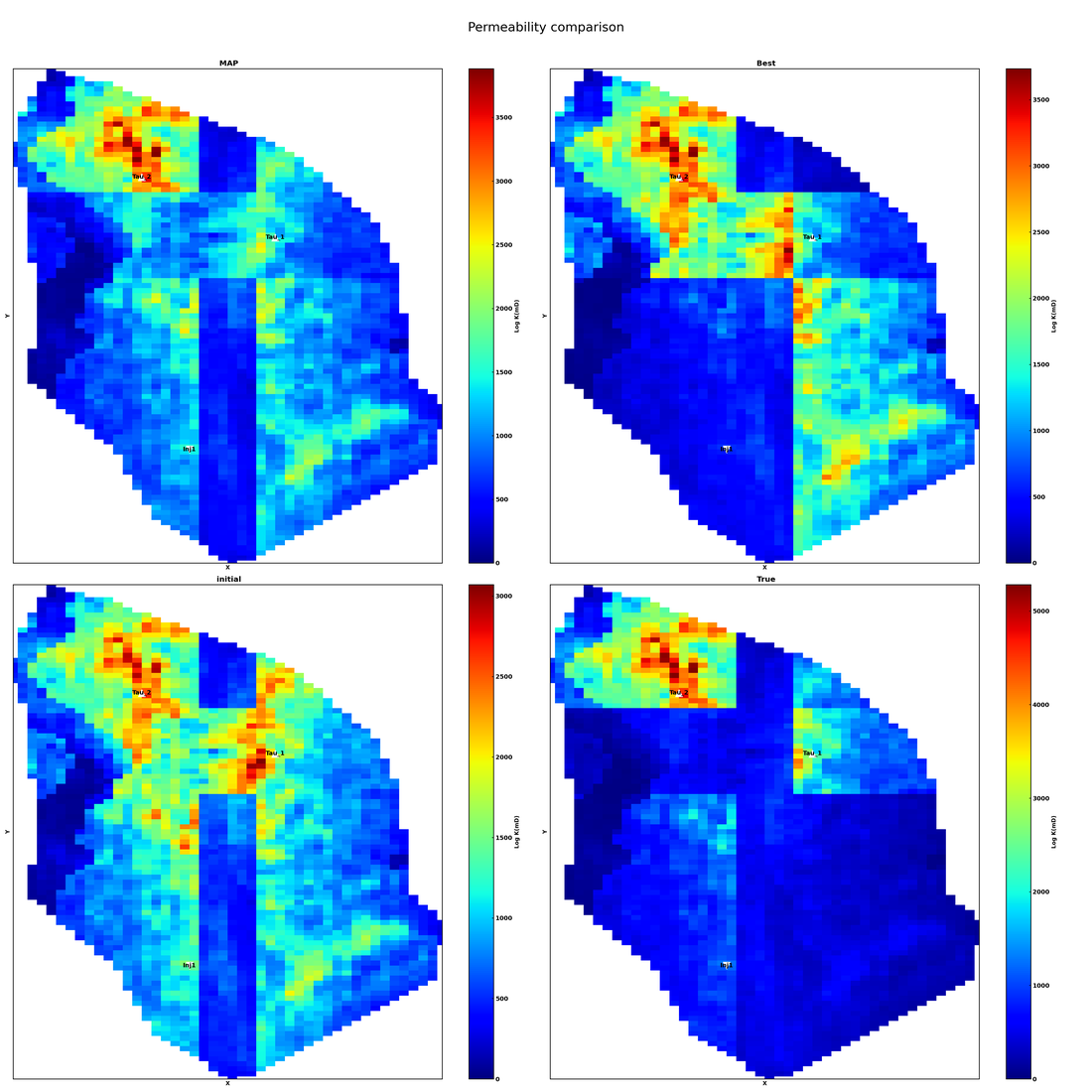}\caption{Production-only posterior}\end{subfigure}
\hfill
\begin{subfigure}{0.49\linewidth}\includegraphics[width=\linewidth]{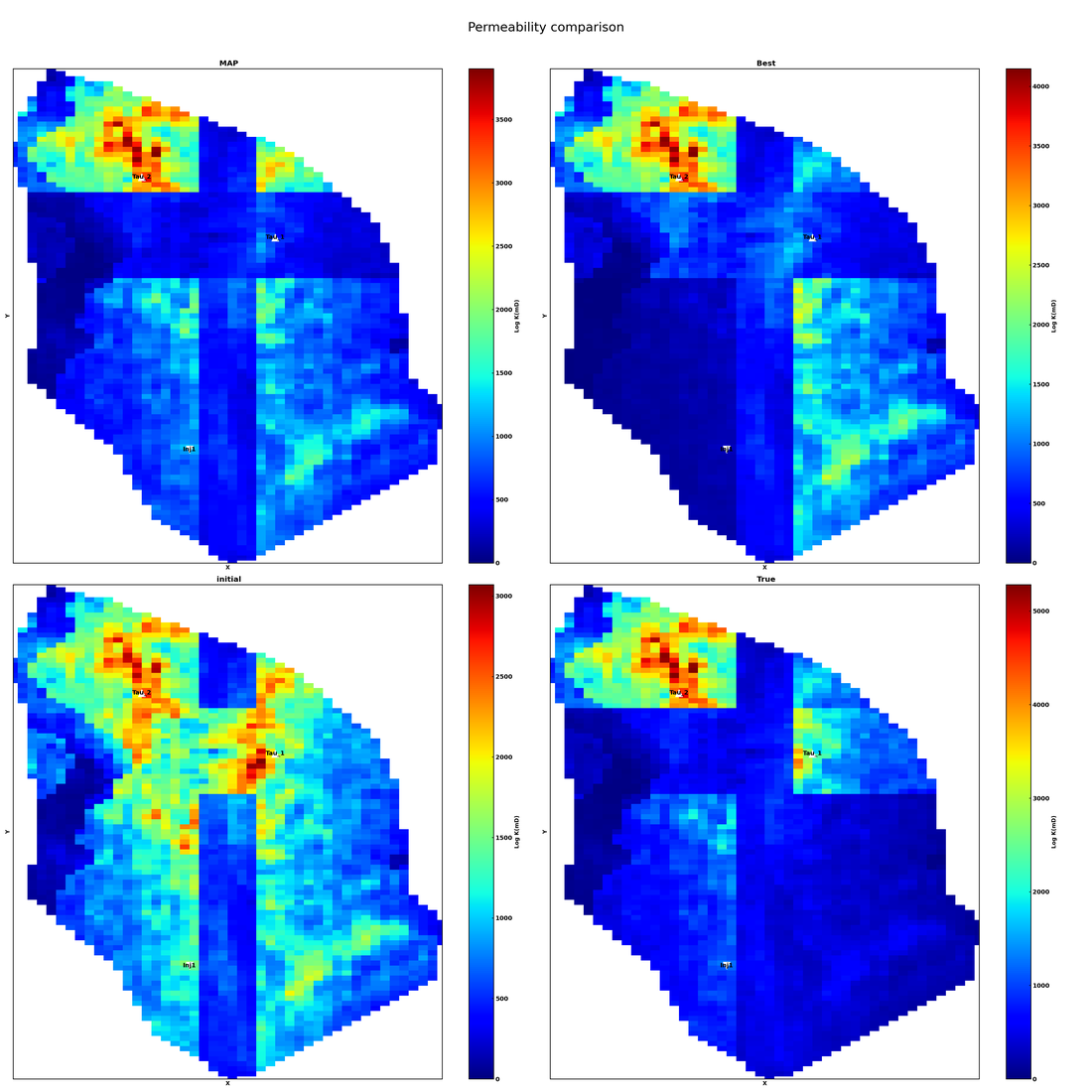}\caption{Joint seismic+production posterior}\end{subfigure}
\caption{MAP permeability vs.\ truth (each panel: MAP, best, prior, true model;
log-mD). Production-only (left) leaves the inter-well connectivity
under-determined; the joint inversion (right) recovers the high-$K$ corridor and
fault block matching the true model, the visual case for ``why joint is
better''.}
\label{fig:perm_compare}
\end{figure}

Figure~\ref{fig:wwpr_compare} contrasts the \emph{posterior} water-rate (WWPR)
ensembles. Under production-only assimilation the posterior WWPR retains a
broad spread about the truth (breakthrough timing is poorly constrained); under
joint assimilation it collapses tightly onto the observed history, because the
swept volume that governs water rate is directly imaged by the 4-D seismic
$\Delta V_p$. This is the production-space signature of the sharper joint
posterior.

\begin{figure}[h]\centering
\begin{subfigure}{0.49\linewidth}\includegraphics[width=\linewidth]{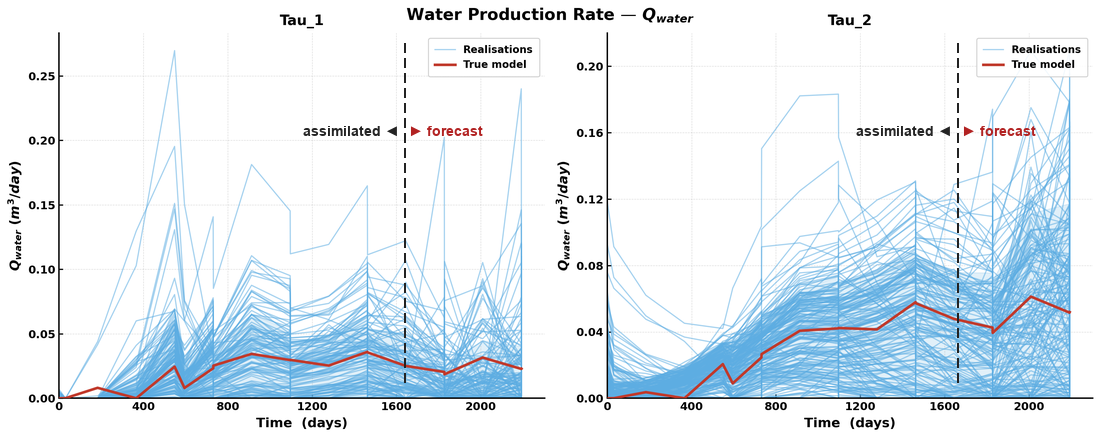}\caption{Production-only posterior WWPR}\end{subfigure}
\hfill
\begin{subfigure}{0.49\linewidth}\includegraphics[width=\linewidth]{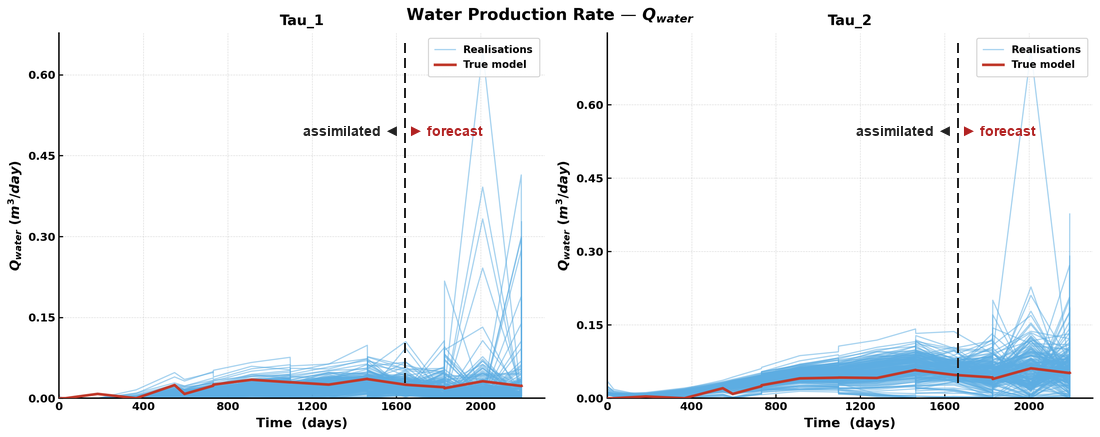}\caption{Joint posterior WWPR}\end{subfigure}
\caption{Posterior water-production-rate ensembles (red: truth). The dashed
vertical line marks the assimilation/forecast boundary: the first $15$ report
steps (left) are assimilated, the final $5$ (right) are forecast blind. The
joint posterior (right sub-panel) concentrates markedly more tightly than the
production-only posterior (left), including over the forecast window, quantifying
the seismic information gain ($s_{\mathrm{P2}}{=}0.876$ vs.\
$s_{\mathrm{P1}}{=}2.556$, Tab.~\ref{tab:scores}).}
\label{fig:wwpr_compare}
\end{figure}

\subsubsection{Producer bottom-hole pressure: prior-to-posterior collapse}
\label{sec:bhpprod}
The clearest single-response demonstration of the assimilation is the producer
bottom-hole-pressure observable WBHP\_PROD, whose ensemble is transported from
the broad geological prior onto the observed history by the two-stage schedule.
Figure~\ref{fig:bhp_ensemble} plots the full $N_e=190$-member ensemble of the
producer BHP defined by \eqr{eq:bhp}: the left panel is the \emph{prior}
ensemble, each grey trajectory a forward rollout of one prior realisation
(\S\ref{sec:prior}) through the surrogate and the Peaceman head, and the
right panel is the \emph{posterior} ensemble after the joint aREKI$\to$ES-MDA
update. Two features are diagnostic. First, the prior spread is wide and
biased: the untuned permeability realisations produce a family of drawdown
curves that bracket, but do not track, the observed BHP (red), a direct
visualisation of the prior predictive variance
$\mathrm{Var}_{\mu_0}[g_{\mathrm{BHP}}(u)]$. Second, the posterior ensemble
contracts sharply around the observations, its mean tracking the history and its
$\mathrm{P10}$--$\mathrm{P90}$ band narrowing to within the $25\%$
production-noise envelope encoded in $\Gamma_p$ \eqr{eq:gamma}. This collapse
is the L\"owner-order sharpening \eqr{eq:loewner} rendered in
observation space: WBHP\_PROD is a transmissibility-weighted average of the
near-well pressure field \eqr{eq:bhp}, so it lies substantially in
$\mathrm{row}(J_p)$ and is well constrained already by production, while the
residual spread is removed once the joint update fixes the surrounding pressure
support imaged by the 4-D seismic. The near-monotone tightening, without the
ensemble collapsing to a spurious delta, is the practical signature of the
mild inter-phase inflation $\gamma=1.01$ preserving rank while the misfit
relaxes toward the discrepancy target $\E[2J]\approx N_{\mathrm{op}}$
\eqr{eq:alpha}. These two panels were absent from the original reporting of
this study and are included here because the prior/posterior BHP envelope is the
most legible operational summary of the history match: it is the plot an asset
team reads to judge whether the calibrated model can be trusted for forecasting.

\begin{figure}[h]\centering
\begin{subfigure}{0.49\linewidth}\includegraphics[width=\linewidth]{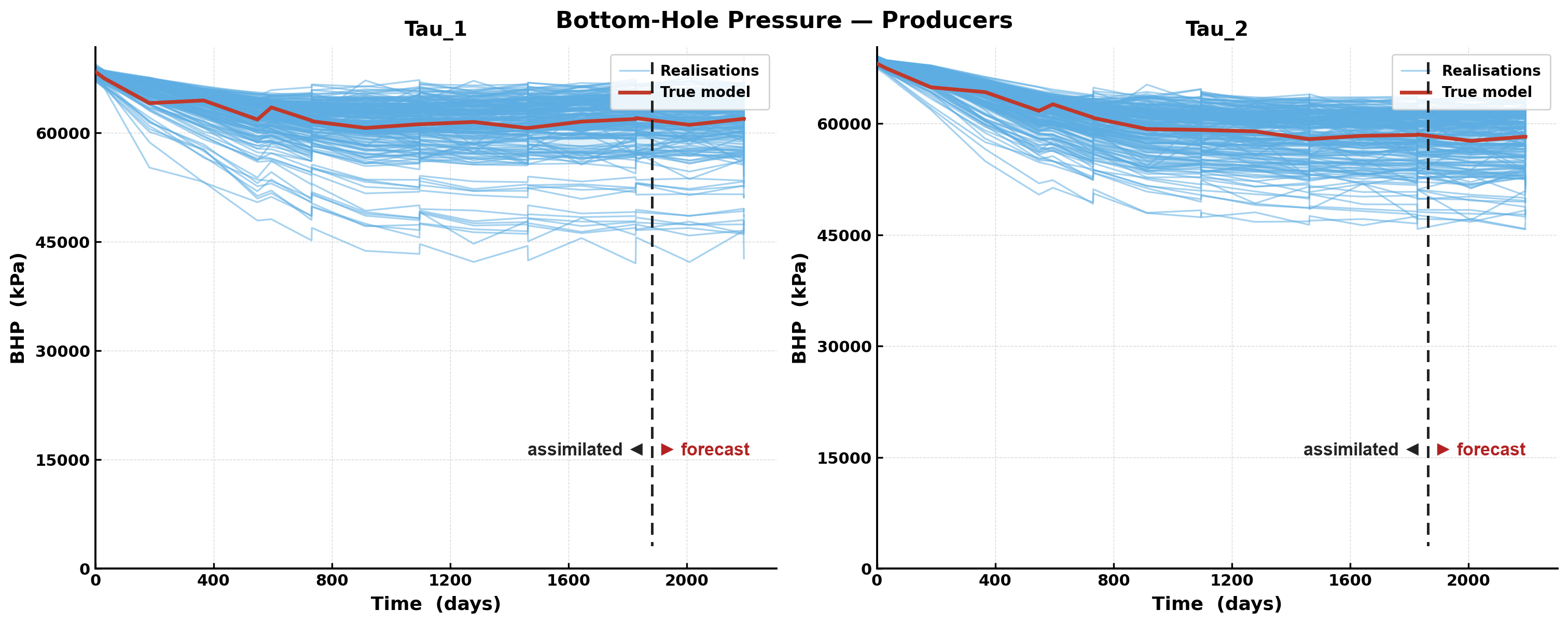}\caption{Prior ensemble WBHP\_PROD}\end{subfigure}
\hfill
\begin{subfigure}{0.49\linewidth}\includegraphics[width=\linewidth]{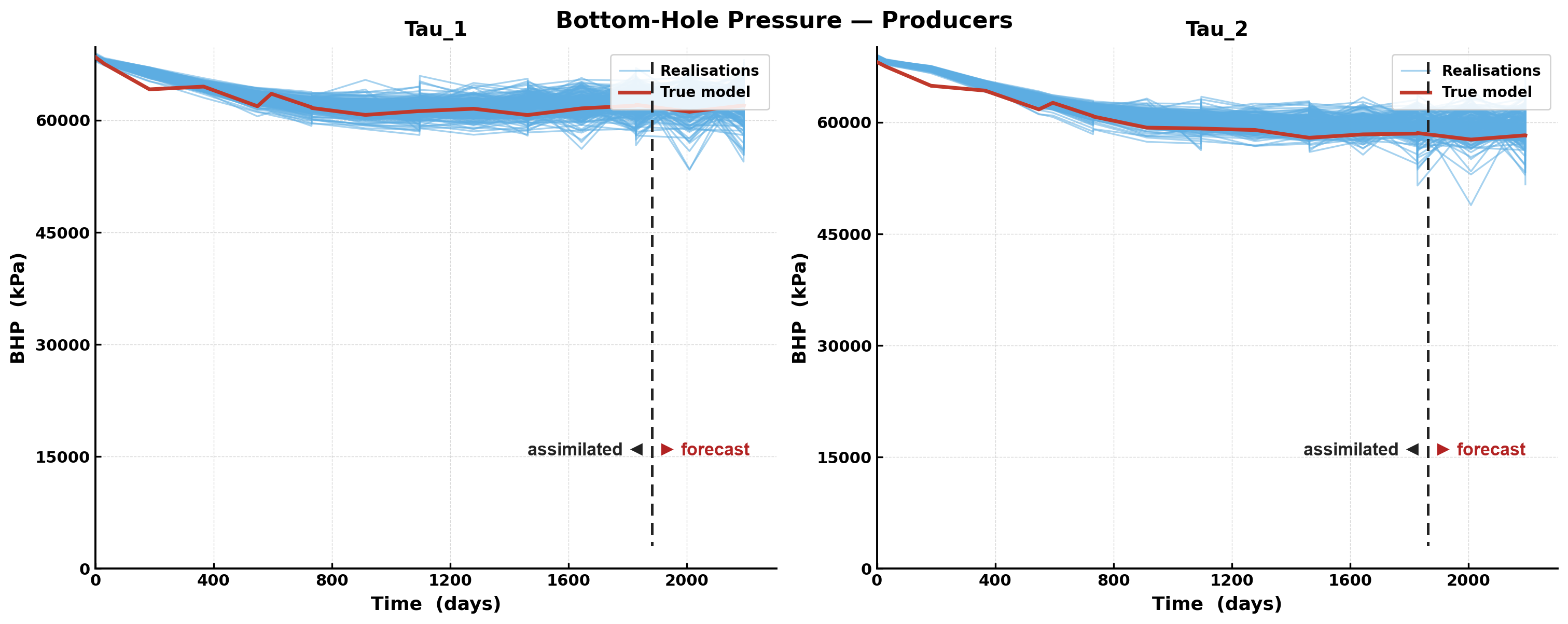}\caption{Posterior ensemble WBHP\_PROD}\end{subfigure}
\caption{Producer bottom-hole pressure (WBHP\_PROD) ensembles across the
$N_e{=}190$ members (blue: individual realisations; red: observed/true history).
The dashed vertical line is the assimilation/forecast boundary (first $15$ steps
assimilated, final $5$ forecast). \textbf{Left:} the \emph{prior} ensemble,
forward rollouts of the unconditioned prior realisations, shows a wide, biased
predictive spread. \textbf{Right:} the \emph{posterior} ensemble after the joint
aREKI$\to$ES-MDA update tracks the observed history, including across the
forecast window. We caution that the posterior $\mathrm{P10}$--$\mathrm{P90}$
band narrows to within the $25\%$ production-noise envelope of $\Gamma_p$; a
predictive band tighter than the observation noise on assimilated data is a
symptom of \emph{over-collapse} rather than an unqualified success, and should
be read together with the spread diagnostics called for in
\S\ref{sec:limitations}.}
\label{fig:bhp_ensemble}
\end{figure}

The joint posterior percentile fields (Fig.~\ref{fig:joint_perm}) confirm the
uncertainty collapse: the P10--P90 band is narrow over the
seismically-imaged corridor and wider at the unobserved periphery.

\begin{figure}[h]\centering
\includegraphics[width=0.95\linewidth]{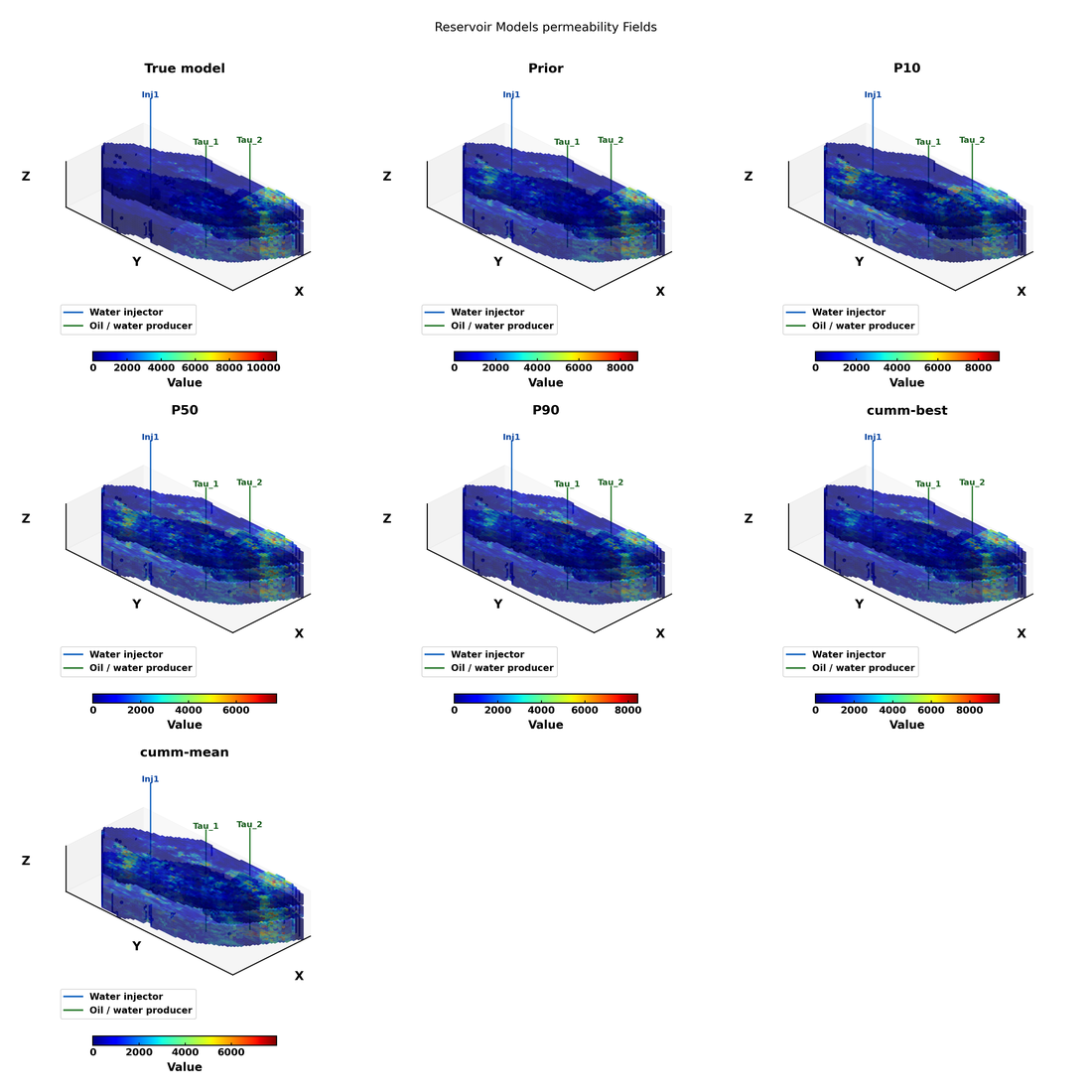}
\caption{Joint posterior permeability percentiles (true, prior, P10/P50/P90,
best, mean; log-mD). The posterior concentrates on the true high-$K$ corridor
and fault structure.}
\label{fig:joint_perm}
\end{figure}
\FloatBarrier

\subsubsection{Cost reduction}
Figure~\ref{fig:cost} summarises the cost-function reduction across the
two-stage schedule: the joint-posterior well responses (Fig.~\ref{fig:wwpr_compare})
collapse onto the observed history, and the ensemble-mean RMSE drops by nearly
two orders of magnitude.

\begin{figure}[h]\centering
\includegraphics[width=0.46\linewidth]{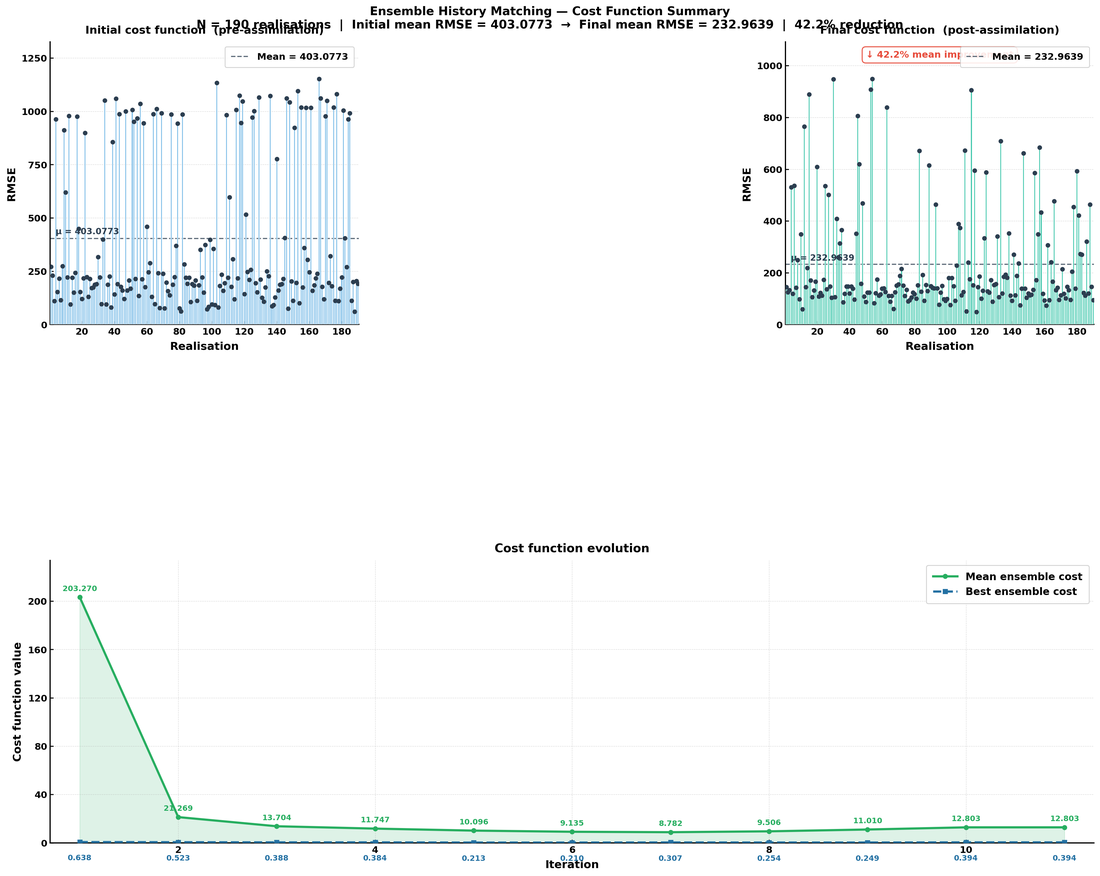}
\caption{Joint two-stage cost-function summary ($N_e{=}190$): the ensemble-mean
RMSE drops by nearly two orders of magnitude from prior to posterior.}
\label{fig:cost}
\end{figure}

\subsubsection{Per-response RMSE histograms}
Figure~\ref{fig:hist} gives posterior RMSE histograms across the ensemble,
contrasting base case, P10/P50/P90, MAP and best. The MAP/best models lie at
the low-RMSE tail of every response; the joint seismic constraint sharpens the
water-rate (WWPR) distribution most strongly, consistent with seismic
sensitivity to the water-invaded volume.

\begin{figure}[htbp]\centering
\begin{subfigure}{0.48\linewidth}\includegraphics[width=\linewidth]{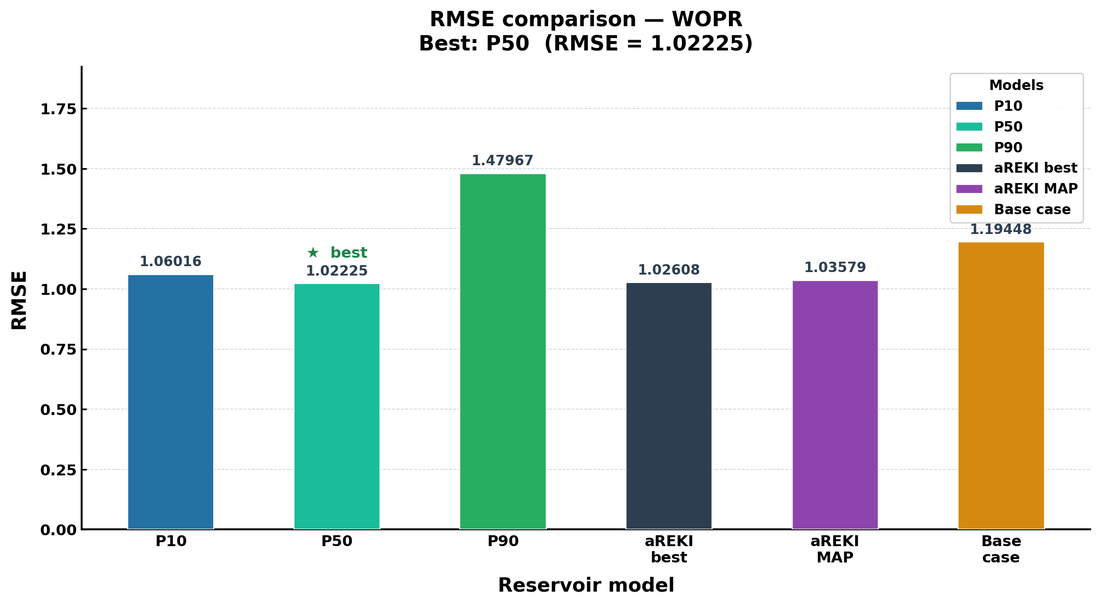}\caption{WOPR}\end{subfigure}
\hfill
\begin{subfigure}{0.48\linewidth}\includegraphics[width=\linewidth]{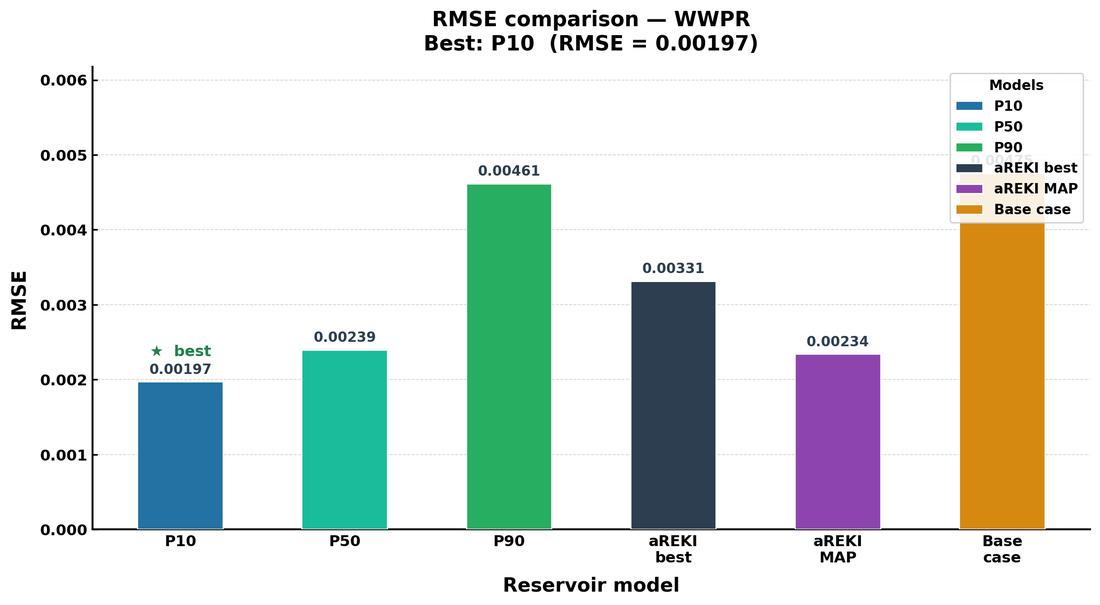}\caption{WWPR}\end{subfigure}\\[4pt]
\begin{subfigure}{0.48\linewidth}\includegraphics[width=\linewidth]{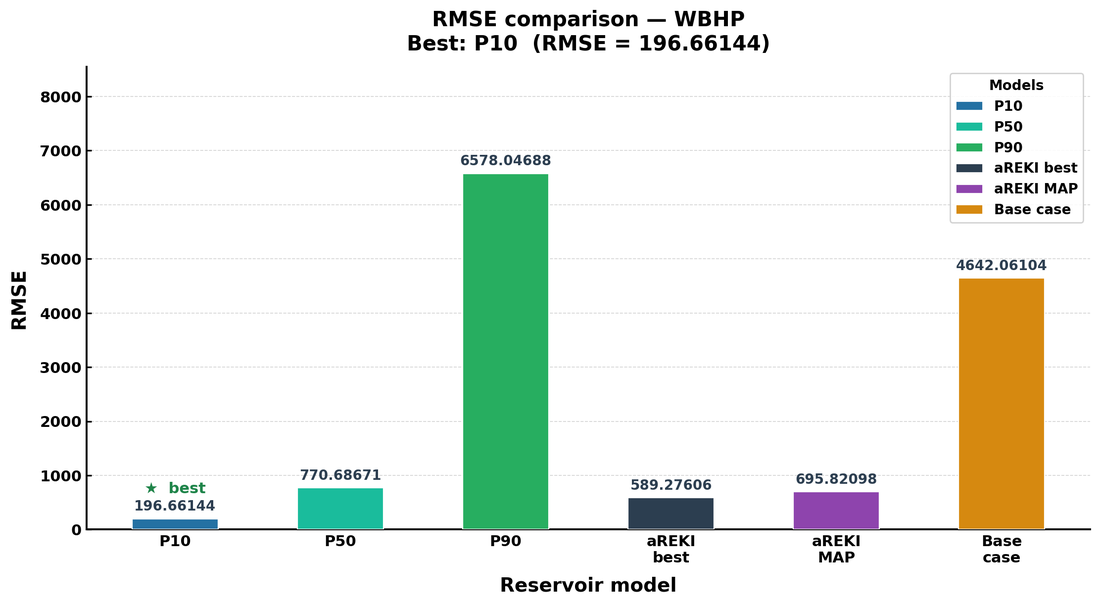}\caption{WBHP inj.}\end{subfigure}
\hfill
\begin{subfigure}{0.48\linewidth}\includegraphics[width=\linewidth]{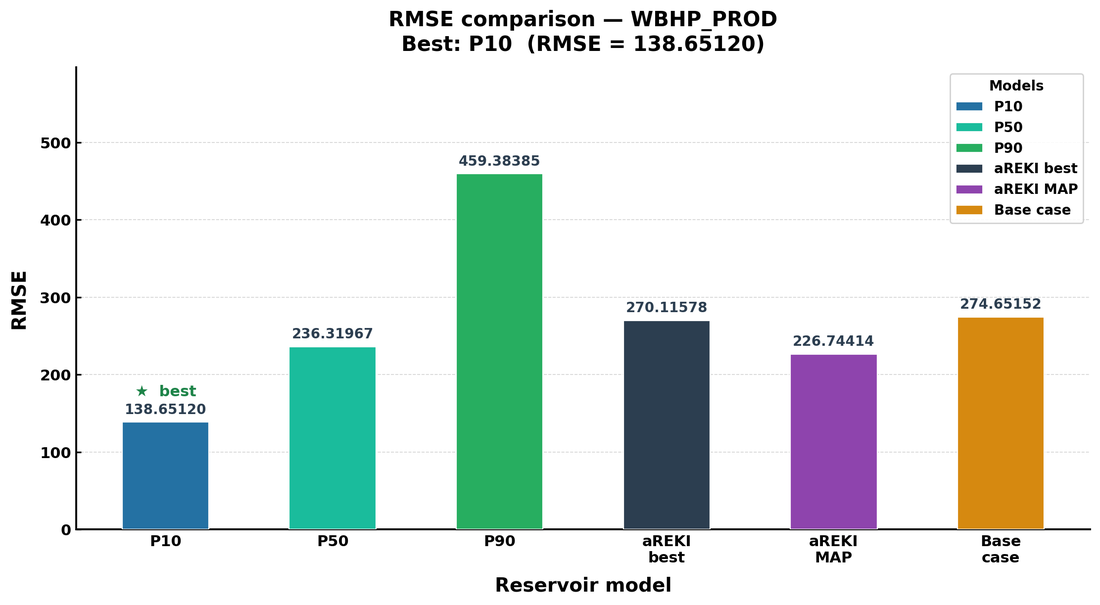}\caption{WBHP prod.}\end{subfigure}
\caption{Posterior RMSE histograms (joint inversion). MAP/best markers lie at
the low-RMSE tail across all four responses.}
\label{fig:hist}
\end{figure}

\subsubsection{Quantitative dominance and summary}
Table~\ref{tab:scores} consolidates the comparison: the joint score
($0.876<1$) means the posterior misfit is below the prior-normalised reference,
whereas production-only ($2.556$) is not. Table~\ref{tab:summary} summarises
the forward and inverse metrics.

\begin{table}[h]\centering\small
\caption{Production-only (P1) vs.\ joint seismic--production (P2). Lower score
is better; P2 selected.}
\label{tab:scores}
\begin{tabular}{p{0.42\linewidth}cc}
\toprule
\textbf{Quantity} & \textbf{P1 (production)} & \textbf{P2 (joint)}\\
\midrule
Assimilated data & $d_p$ only & $d=(d_p,d_s)$, $k_s{=}4554$\\
Algorithm & aREKI & aREKI $\to$ infl.\ $\to$ ES-MDA\\
Dimensionless score $s$ & $2.556$ & $\mathbf{0.876}$\\
Inter-phase inflation $\gamma$ &, & $1.01$\\
Wall-clock & $2303$~s & $3151$~s\\
Selected posterior & \multicolumn{2}{c}{\textbf{P2} (gap $1.68$)}\\
\bottomrule
\end{tabular}
\end{table}

\begin{table}[h]\centering\small
\caption{End-to-end summary of dynamic-modelling prediction accuracy and
inversion performance.}
\label{tab:summary}
\begin{tabular}{p{0.54\linewidth}cc}
\toprule
\textbf{Metric}&\textbf{Value}&\textbf{Ref.}\\
\midrule
\multicolumn{3}{l}{\textit{Autoregressive forward surrogate}}\\
Pressure / water / oil $R^2$ & $>\!0.999$ & Fig.~\ref{fig:r2l2}\\
WOPR $R^2$ / WBHP\_PROD $R^2$ & $>$99.99\% / 98\% & Fig.~\ref{fig:cmp}\\
Seismic $\Delta V_p$ fidelity & high & Fig.~\ref{fig:r2l2}\\
Rollout time (20 steps) & $0.579$~s & log\\
Training ($4\times$GB200, DDP) & $3$h$58$m & log\\
\midrule
\multicolumn{3}{l}{\textit{Two-stage inversion ($N_e{=}190$, $10$ iter.)}}\\
Production-only score & $2.556$ & Tab.~\ref{tab:scores}\\
Joint seismic--production score & $\mathbf{0.876}$ & Tab.~\ref{tab:scores}\\
Selected posterior & P2 (joint) & Tab.~\ref{tab:scores}\\
\bottomrule
\end{tabular}
\end{table}

\subsubsection{Posterior uncertainty quantification}
\label{sec:uq}
The deliverable of a Bayesian history match is not a single model but the
posterior \emph{distribution}, summarised by percentiles. From the final
ensemble we extract P10/P50/P90 permeability fields, the MAP (mean of
the posterior ensemble), and the best member (lowest individual misfit); these
appear in Figure~\ref{fig:joint_perm}. The percentile spread is the
calibrated forecast uncertainty: a narrow P10--P90 band where the data are
informative (the seismically-constrained inter-well corridor) and a wider band
where they are not (the aquifer periphery, unseen by either modality). This
spatial heterogeneity of the posterior variance is what the
Fisher-information decomposition~\eqr{eq:fisher} predicts, variance
collapses in $\mathrm{row}(J_p)\cup\mathrm{row}(J_s)$ and reverts to the prior
in the joint null space. Reporting the band, rather than a single MAP forecast,
is what distinguishes a defensible Bayesian match from a deterministic fit.

\section{Discussion and Conclusions}
\label{sec:discussion}
The forward pressure $R^2>0.999$ is the empirical signature of
Theorem~\thmr{thm:contract}: the elliptic block is strongly dissipative, its
rollout Jacobian contracts, and AR errors do not accumulate. The slower
oil-saturation convergence reflects the $BV$-regularity gap
in~\eqr{eq:fnoerr} and the front-direction Lipschitz amplification of
Propositions~\thmr{prop:tv},~\thmr{prop:rollout}; AR training with $K=10$ TBPTT
and noise injection is the remedy: it trains the network to damp
its own front-position errors (reducing $\tilde L$ in
Proposition~\thmr{prop:arstab}), while Proposition~\thmr{prop:tbptt} certifies
the truncated gradient is near-unbiased.

On the inverse side, the strict dominance of the joint posterior is the
L\"owner-order inequality~\eqr{eq:loewner}: production data leave a large
null space along inter-well connectivity, and the spatially-distributed seismic
Fisher information $\mathcal I_s$ fills it. The two-stage schedule is the
correct way to combine the modalities: aREKI's adaptive $\alpha$ is
well-behaved on the informative production phase, the deterministic ES-MDA
schedule is robust on the already-conditioned joint phase whose misfit variance
has collapsed, and the mild inflation $\gamma=1.01$ restores just enough spread
to keep $C_{up}$ full-rank. The petro-elastic channel makes all of this
tractable: because the Gassmann map is a closed-form, differentiable,
GPU-batched transform of the surrogate's own output, $g_s$ adds negligible cost
over $g_p$, and the entire two-stage inversion completes in under one GPU-hour.

The parametrisation machinery (\S\ref{sec:param}), unused at $N_e=190$ but
fully implemented, is the path to larger problems: at finer grids or smaller
ensembles the rank deficiency of $\hat C_{uu}$ forces either a DCT/VCAE
reduction or localisation, and the VCAE chart additionally guarantees
geologically plausible, multi-modal posteriors that the Gaussian EKI update
cannot otherwise represent. Several extensions are compatible with the present
architecture, since the seismic operator is an arbitrary differentiable
post-map of the surrogate state: a fully non-Gaussian (transport- or
score-based) sampler to remove the Jensen bias~\eqr{eq:jensen}, a learned
heteroscedastic surrogate-error term in $\Gamma$, and assimilation of the full
4-D sequence and AVO attributes beyond $\Delta V_p$.

\subsection{Limitations and scope}
\label{sec:limitations}
We state the scope of the reported study plainly; several of the following are
prerequisites for an operational claim rather than optional extensions.
\begin{enumerate}
\item \textbf{Posterior not yet verified against full physics.} The inversion
moves ensemble members off the manifold of prior realisations (the clip-to-bounds
projection is direct evidence), so the surrogate is queried under a
\emph{parameter-space} covariate shift analogous to the temporal one we treat in
\S\ref{sec:forward}. The decisive validation, running the selected P2 MAP and
best members back through the full-physics simulator and overlaying against the
observations, is outstanding; until it is done, part of the ``match'' may be a
match to surrogate-error structure.
\item \textbf{Synthetic experiment (inverse crime).} As noted in
\S\ref{sec:results}, the seismic data are generated by the same petro-elastic
code used in the inversion, with no model error, repeatability (NRMS) noise, or
time-shift/registration/depth-conversion error. The Tau model is itself a
realistic synthetic model, so the experiment is synthetic end to end.
\item \textbf{Likelihood-budget accounting in the two-stage schedule.} Phase 1
runs aREKI to budget exhaustion ($\sum_k\alpha_k^{-1}\!\ge\!1$), and Phase 2
re-assimilates the stacked $d=(d_p,d_s)$ with a fresh budget, so the production
block enters with total weight $\approx2$ rather than $1$. The consequence is
specific and bounded: the P2 posterior is tighter in $\mathrm{row}(J_p)$ than a
single-budget Bayesian update would give, so the \emph{absolute} value of
$s_{\mathrm{P2}}$ should be read as a fit diagnostic rather than a calibrated
posterior credibility statement. Two things limit the damage. First, the
comparison is like-for-like in the modality that matters: P1 and P2 differ by the
addition of $d_s$, and the seismic block is assimilated once in both accounts.
Second, the blind forecast window
(\S\ref{sec:results}) is \emph{immune} to this defect, because neither run
assimilated those $5$ steps under any budget; posterior behaviour there is
therefore uncontaminated evidence. The clean fix, which we recommend for any
operational use, is to allocate a \emph{partial} production budget to Phase 2 (or
to treat Phase 1 purely as a preconditioner) so that
$\sum_k\alpha_k^{-1}=1$ holds once across the whole schedule.
\item \textbf{Seismic information is overstated by a diagonal $\Gamma_s$.} The
depth-averaged $|\Delta V_p|$ map is spatially correlated over many cells, so its
effective degrees of freedom are $\mathcal{O}(10^2)$, not $k_s=4554$; the Fisher
bound $\mathrm{rank}(\mathcal I_s)\!\le\!k_s$ and the discrepancy target
$\E[2J]\!\approx\!N_{\mathrm{op}}$ should use the effective, not nominal, DOF.
A spatially-correlated $\Gamma_s$ (or the truncated-DCT encoding, which is the
more defensible of the two we implement) would reduce the over-confidence and
the ensemble-collapse pressure.
\item \textbf{Capacity versus training-set size.} The composite operator carries
$\approx\!2\times10^{8}$ parameters (\S\ref{sec:complexity}) trained on $190$
realisations, and the held-out test set is only $5$ members. Nothing in the
reported metrics indicates memorisation, but a $5$-member test set cannot exclude
it either. The direct evidence a reader should want is the train/validation/test
loss gap together with the error \emph{distribution} over a substantially larger
held-out set (tens of realisations), rather than metrics from a single member;
enlarging the test split costs only full-physics runs from the unused pool.
\item \textbf{Theory is not yet matched by ablation.} The PINO-induced
contractivity of Theorem~\thmr{thm:contract} assumes $\cL_{\mathrm{PINO}}(\theta^
\star)=0$, which is not attained (final loss $24.90$); Propositions~\thmr{prop:tv},
\thmr{prop:arstab},\thmr{prop:tbptt} hinge on $\E\norm{D_x\cG_\theta}$, which we
assume rather than measure. Both are testable: an AR-only / AR$+$noise /
AR$+$noise$+$PINO ablation (and a sweep of $K$ and the noise level) on
long-horizon drift and mass-balance violation, and a power-iteration estimate of
$\E\norm{D_x\cG_\theta}$ per field block to exhibit $\alpha_P\!\ll\!\alpha_O$.
\item \textbf{Over-collapse diagnostics.} A posterior predictive band narrower
than the observation noise (Fig.~\ref{fig:bhp_ensemble}) warrants spread
diagnostics, the trace of $\hat C_{uu}$ versus iteration, an effective sample
size, or rank histograms, which we do not yet report.
\item \textbf{Transferability.} The Peaceman heads use a well-count-specific
feature layout ($N_p=2$), so the trained operator does not transfer as-is to a
different well count or grid; retraining is required.
\end{enumerate}

\subsection{Conclusions}
We have presented and demonstrated a unified surrogate-and-inversion framework on
a realistic synthetic reservoir model referred to here as Tau.
(i) An \emph{autoregressive} PINO-FNO surrogate, trained
in closed loop with $K$-step TBPTT, discrete physics residuals and
covariate-shift noise, and distributed over four GB200 GPUs by DDP/NCCL,
reproduces the black-oil state and well responses at $R^2>0.999$ and rolls out
in $0.579$~s. (ii) A closed-form, GPU-batched Gassmann channel fused inside the
rollout turns the surrogate into a 4-D seismic forward model at negligible
marginal cost. (iii) A two-stage aREKI/ES-MDA inversion, with optional DCT/VCAE
reparametrisation and Furrer localisation, founded on a function-space Bayesian
posterior, assimilates production and seismic jointly. On the synthetic
reservoir-model experiment the joint posterior is sharper than production-only
matching and tracks a blind forecast window, recovering inter-well connectivity that
production data alone leave under-determined, consistent with the Löwner-order
argument. The remaining step to an operational claim, full-physics verification
of the posterior and the ablations of \S\ref{sec:limitations}, is set out
explicitly; the contribution here is the integrated architecture and its
demonstration on a realistic synthetic reservoir model.

\section*{Nomenclature}
\label{sec:nomenclature}
Mathematical symbols are collected in Table~\ref{tab:notation}; abbreviations
are listed below.

\begin{center}\small
\begin{tabular}{@{}p{0.10\linewidth}p{0.34\linewidth}@{\hspace{0.04\linewidth}}p{0.12\linewidth}p{0.33\linewidth}@{}}
\toprule
4-D & time-lapse (four-dimensional) seismic & MAP & maximum a posteriori\\
AR & autoregressive & MoE & mixture of experts\\
aREKI & adaptive regularised ensemble Kalman inversion & NCCL & NVIDIA Collective Communications Library\\
BHP & bottom-hole pressure & NRMS & normalised root-mean-square difference\\
CCR & cluster--classify--regress & PINN & physics-informed neural network\\
CFL & Courant--Friedrichs--Lewy & PINO & physics-informed neural operator\\
DCT & discrete cosine transform & PVT & pressure--volume--temperature\\
DDP & distributed data parallel & RMSE & root-mean-square error\\
EKI & ensemble Kalman inversion & TBPTT & truncated backpropagation through time\\
ELBO & evidence lower bound & TPFA & two-point flux approximation\\
ES-MDA & ensemble smoother with multiple data assimilation & VCAE & variational convolutional auto-encoder\\
FFT & fast Fourier transform & WBHP & injector bottom-hole pressure\\
FNO & Fourier neural operator & WBHP\_PROD & producer near-well drainage pressure\\
FV & finite volume & WOPR & well oil production rate\\
FVF & formation volume factor & WWCT & well water cut\\
GPU & graphics processing unit & WWPR & well water production rate\\
\bottomrule
\end{tabular}
\end{center}

\section*{Acknowledgments}
This work was supported by NVIDIA Corporation. The authors gratefully
acknowledge \textbf{Shell} for the realistic synthetic reservoir simulation data,
and in particular \textbf{David Blom} (Shell Information Technology International
B.V.) for preparing and curating the ensemble of reservoir simulation output
files encompassing grid-, well- and reservoir-level data, and the NVIDIA
PhysicsNeMo
team for framework support. Compute was provided on NVIDIA GB200 GPUs.

\section*{Data and code availability}
The Tau model data are proprietary to Shell and cannot be released, and the
software implementation of the workflow is not publicly available. The
methodology is nonetheless specified in full, so that it can be re-implemented
independently on any reservoir model and with any reservoir simulator: the
forward model, surrogate architecture and training procedure
(\S\ref{sec:forward}, Table~\ref{tab:arch}, Algorithm~\ref{alg:train}), the
petro-elastic channel (\S\ref{sec:seismic}), the inversion algorithms and their
settings (\S\ref{sec:inverse}, Algorithms~\ref{alg:areki}
and~\ref{alg:master}), and the data split and run configuration
(\S\ref{sec:results}).

{\footnotesize\setstretch{0.94}\setlength{\bibsep}{0.0ex plus 0.03ex}

}


\begin{thebibliography}{99}
\bibitem[Alpak et al.(2023)]{alpak2023lbfgs} Alpak, F.O., Gao, G., Florez, H., Shi, S., Vink, J., Blom, C., Saaf, F., Wells, T. (2023). A machine-learning-accelerated distributed LBFGS method for field development optimization: algorithm, validation, and applications. \textit{Computational Geosciences} 27(3):425--450. doi:10.1007/s10596-023-10197-3.
\bibitem[Azizzadenesheli et al.(2024)]{azizzadenesheli2024neuralop} Azizzadenesheli, K., Kovachki, N., Li, Z., Liu-Schiaffini, M., Kossaifi, J., Anandkumar, A. (2024). Neural operators for accelerating scientific simulations and design. \textit{Nature Reviews Physics} 6:320--328.
\bibitem[Baldassari et al.(2024)]{baldassari2024conditional} Baldassari, L., Siahkoohi, A., Garnier, J., S{\o}lna, K., de Hoop, M.V. (2024). Conditional score-based diffusion models for Bayesian inference in infinite dimensions. \textit{Adv.\ Neural Inf.\ Process.\ Syst.} 36. arXiv:2305.19147.
\bibitem[Batzle and Wang(1992)]{batzle1992} Batzle, M., Wang, Z. (1992). Seismic properties of pore fluids. \textit{Geophysics} 57(11):1396--1408.
\bibitem[Bernholdt et al.(2019)]{bernholdt2019ccr} Bernholdt, D.E., Cianciosa, M.R., Green, D.L., Park, J.M., Law, K.J.H., Etienam, C. (2019). Cluster, classify, regress: a general method for learning discontinuous functions. \textit{Foundations of Data Science} 1(4):491--506. doi:10.3934/fods.2019020.
\bibitem[Bengio et al.(2015)]{bengio2015scheduled} Bengio, S., Vinyals, O., Jaitly, N., Shazeer, N. (2015). Scheduled sampling for sequence prediction with RNNs. \textit{NeurIPS} 28.
\bibitem[Bl\"omker et al.(2019)]{blomker2019strongly} Bl\"omker, D., Schillings, C., Wacker, P., Weissmann, S. (2019). Well posedness and convergence analysis of the ensemble Kalman inversion. \textit{Inverse Problems} 35(8):085007.
\bibitem[Bogachev(1998)]{bogachev1998gaussian} Bogachev, V.I. (1998). \textit{Gaussian Measures}. Mathematical Surveys and Monographs 62, American Mathematical Society.
\bibitem[Brandstetter et al.(2022)]{brandstetter2022message} Brandstetter, J., Worrall, D., Welling, M. (2022). Message passing neural PDE solvers. \textit{ICLR}.
\bibitem[Canchumuni et al.(2019)]{canchumuni2019esmda} Canchumuni, S.W.A., Emerick, A.A., Pacheco, M.A.C. (2019). Towards a robust parameterization for conditioning facies models using deep variational autoencoders and ES-MDA. \textit{Computers \& Geosciences} 128:87--102.
\bibitem[Cao(2021)]{cao2021choose} Cao, S. (2021). Choose a transformer: Fourier or Galerkin. \textit{NeurIPS} 34:24924--24940.
\bibitem[Chada et al.(2020)]{chada2020tikhonov} Chada, N.K., Stuart, A.M., Tong, X.T. (2020). Tikhonov regularization within ensemble Kalman inversion. \textit{SIAM J.\ Numer.\ Anal.} 58(2):1263--1294.
\bibitem[Chandra et al.(2025)]{chandra2025co2} Chandra, A., Koch, M., Pawar, S., Panda, A., Azizzadenesheli, K., Snippe, J., Alpak, F.O., Hariri, F., Etienam, C., Devarakota, P., Anandkumar, A., Hohl, D. (2025). Fourier neural operator based surrogates for CO$_2$ storage in realistic geologies. \textit{arXiv preprint} arXiv:2503.11031. (Published in \textit{Adv.\ Theory Simul.}, 2026.)
\bibitem[Chung et al.(2023)]{chung2023dps} Chung, H., Kim, J., McCann, M.T., Klasky, M.L., Ye, J.C. (2023). Diffusion posterior sampling for general noisy inverse problems. \textit{ICLR}.
\bibitem[Cotter et al.(2013)]{cotter2013mcmc} Cotter, S.L., Roberts, G.O., Stuart, A.M., White, D. (2013). MCMC methods for functions: modifying old algorithms to make them faster. \textit{Statistical Science} 28(3):424--446.
\bibitem[Da Prato and Zabczyk(2014)]{daprato2014stochastic} Da Prato, G., Zabczyk, J. (2014). \textit{Stochastic Equations in Infinite Dimensions}, 2nd ed. Cambridge Univ.\ Press.
\bibitem[Dashti et al.(2013)]{dashti2013map} Dashti, M., Law, K.J.H., Stuart, A.M., Voss, J. (2013). MAP estimators and their consistency in Bayesian nonparametric inverse problems. \textit{Inverse Problems} 29(9):095017.
\bibitem[Dashti and Stuart(2017)]{dashti2017bayesian} Dashti, M., Stuart, A.M. (2017). The Bayesian approach to inverse problems. In \textit{Handbook of Uncertainty Quantification}, pp.\ 311--428. Springer.
\bibitem[Emerick and Reynolds(2013)]{emerick2013esmda} Emerick, A.A., Reynolds, A.C. (2013). Ensemble smoother with multiple data assimilation. \textit{Computers \& Geosciences} 55:3--15.
\bibitem[Emerick(2016)]{emerick2016review} Emerick, A.A. (2016). Analysis of the performance of ensemble-based assimilation of production and seismic data. \textit{J.\ Petroleum Sci.\ Eng.} 139:219--239.
\bibitem[Evensen(2018)]{evensen2018analysis} Evensen, G. (2018). Analysis of iterative ensemble smoothers for solving inverse problems. \textit{Computational Geosciences} 22:885--908.
\bibitem[Etienam(2019)]{etienam2019seismic} Etienam, C. (2019). 4D seismic history matching incorporating unsupervised learning. \textit{arXiv preprint} arXiv:1905.07469.
\bibitem[Etienam et al.(2024a)]{etienam2024gpexperts} Etienam, C., Law, K.J.H., Wade, S., Zankin, V. (2024). Fast deep mixtures of Gaussian process experts. \textit{Machine Learning} 113(3):1483--1508. doi:10.1007/s10994-023-06491-x.
\bibitem[Etienam et al.(2024b)]{etienam2024modulus} Etienam, C., Yang, J., Said, I., Ovcharenko, O., Tangsali, K., Dimitrov, P., Hester, K. (2024b). A novel A.I.\ enhanced reservoir characterization with a combined mixture of experts--NVIDIA Modulus based physics informed neural operator forward model. \textit{arXiv preprint} arXiv:2404.14447.
\bibitem[Etienam et al.(2024c)]{etienam2024norne} Etienam, C., Yang, J., Ovcharenko, O., Said, I. (2024c). Reservoir history matching of the Norne field with generative exotic priors and a coupled mixture of experts--physics informed neural operator forward model. \textit{arXiv preprint} arXiv:2406.00889.
\bibitem[Etienam et al.(2024d)]{etienam2024ecmor} Etienam, C., Yang, J., Said, I., Ovcharenko, O. (2024d). A reservoir history matching approach using a coupled mixture of experts--physics informed neural operator forward model. In \textit{Proc.\ ECMOR 2024}, pp.\ 1--35. EAGE.
\bibitem[Etienam et al.(2026a)]{etienam2025seqpino} Etienam, C., Yang, J., Ovcharenko, O., Luiken, N., Onishi, T., Moridis, N., Said, I. (2026a). Sequential physics-constrained neural operator forward modeling for the Norne reservoir system. \textit{arXiv preprint} arXiv:2605.28909. \url{https://arxiv.org/abs/2605.28909}.
\bibitem[Etienam et al.(2026b)]{etienam2026seqecmor} Etienam, C., Ovcharenko, O., Yang, J., Onishi, T., Luiken, N., Moridis, N., Said, I. (2026b). Sequential physics-constrained neural operator forward modelling for the Norne reservoir system. In \textit{Proc.\ ECMOR 2026}, pp.\ 1--28. EAGE. doi:10.3997/2214-4609.202637013.
\bibitem[Fahimuddin et al.(2010)]{fahimuddin2010} Fahimuddin, A., Aanonsen, S.I., Skjervheim, J.-A. (2010). Ensemble based 4D seismic history matching: integration of different levels and types of seismic data. \textit{SPE EUROPEC/EAGE}, SPE-131453.
\bibitem[Furrer and Bengtsson(2007)]{furrer2007} Furrer, R., Bengtsson, T. (2007). Estimation of high-dimensional prior and posterior covariance matrices in Kalman filter variants. \textit{J. Multivariate Anal.} 98(2):227--255.
\bibitem[Gaspari and Cohn(1999)]{gaspari1999} Gaspari, G., Cohn, S.E. (1999). Construction of correlation functions in two and three dimensions. \textit{QJRMS} 125(554):723--757.
\bibitem[Gassmann(1951)]{gassmann1951} Gassmann, F. (1951). \"Uber die Elastizit\"at por\"oser Medien. \textit{Vierteljahrsschrift Naturf.\ Ges.\ Z\"urich} 96:1--23.
\bibitem[Hao et al.(2023)]{hao2023gnot} Hao, Z., Wang, Z., Su, H., et al. (2023). GNOT: A general neural operator transformer for operator learning. \textit{ICML}, 12556--12569.
\bibitem[Iglesias et al.(2013)]{iglesias2013ensemble} Iglesias, M.A., Law, K.J.H., Stuart, A.M. (2013). Ensemble Kalman methods for inverse problems. \textit{Inverse Problems} 29(4):045001.
\bibitem[Iglesias(2016)]{iglesias2016regularising} Iglesias, M.A. (2016). A regularising iterative ensemble Kalman method for PDE-constrained inverse problems. \textit{Inverse Problems} 32(2):025002.
\bibitem[Iglesias and Yang(2021)]{iglesias2021adaptive} Iglesias, M.A., Yang, Y. (2021). Adaptive regularisation for ensemble Kalman inversion. \textit{Inverse Problems} 37(2):025008. doi:10.1088/1361-6420/abd29b.
\bibitem[Keisler(2022)]{keisler2022forecasting} Keisler, R. (2022). Forecasting global weather with graph neural networks. \textit{arXiv:2202.07575}.
\bibitem[Kovachki et al.(2021)]{kovachki2021neural} Kovachki, N., Li, Z., Liu, B., et al. (2021). Neural operator: learning maps between function spaces. \textit{JMLR} 24(89):1--97.
\bibitem[Laloy et al.(2017)]{laloy2017vae} Laloy, E., H\'erault, R., Lee, J., Jacques, D., Linde, N. (2017). Inversion using a new low-dimensional representation of complex binary geological media based on a deep neural network. \textit{Adv.\ Water Resour.} 110:387--405.
\bibitem[Laloy et al.(2018)]{laloy2018gan} Laloy, E., H\'erault, R., Jacques, D., Linde, N. (2018). Training-image based geostatistical inversion using a spatial generative adversarial neural network. \textit{Water Resour.\ Res.} 54(1):381--406.
\bibitem[Landr{\o}(2001)]{landro2001discrimination} Landr{\o}, M. (2001). Discrimination between pressure and fluid saturation changes from time-lapse seismic data. \textit{Geophysics} 66(3):836--844.
\bibitem[Lanthaler et al.(2022)]{lanthaler2022error} Lanthaler, S., Mishra, S., Karniadakis, G.E. (2022). Error estimates for DeepONets. \textit{Trans.\ Math.\ Appl.} 6(1):tnac001.
\bibitem[Li et al.(2020)]{li2020fourier} Li, Z., Kovachki, N., Azizzadenesheli, K., et al. (2020). Fourier neural operator for parametric PDEs. \textit{arXiv:2010.08895}.
\bibitem[Li et al.(2021)]{li2021physics} Li, Z., Zheng, H., Kovachki, N., et al. (2021). Physics-informed neural operator. \textit{arXiv:2111.03794}.
\bibitem[Li et al.(2022)]{li2022markov} Li, Z., Kovachki, N., Azizzadenesheli, K., et al. (2022). Learning dissipative dynamics in chaotic systems (Markov neural operator). \textit{NeurIPS} 35. arXiv:2106.06898.
\bibitem[Li et al.(2023)]{li2023geofno} Li, Z., Huang, D.Z., Liu, B., Anandkumar, A. (2023). Fourier neural operator with learned deformations for PDEs on general geometries. \textit{JMLR} 24(388):1--26.
\bibitem[Lippe et al.(2023)]{lippe2023pderefiner} Lippe, P., Veeling, B.S., Perdikaris, P., Turner, R.E., Brandstetter, J. (2023). PDE-Refiner: Achieving accurate long rollouts with neural PDE solvers. \textit{NeurIPS} 36.
\bibitem[Lu et al.(2021)]{lu2021deeponet} Lu, L., Jin, P., Pang, G., Zhang, Z., Karniadakis, G.E. (2021). Learning nonlinear operators via DeepONet based on the universal approximation theorem of operators. \textit{Nature Machine Intelligence} 3:218--229.
\bibitem[Lumley(2001)]{lumley2001timelapse} Lumley, D.E. (2001). Time-lapse seismic reservoir monitoring. \textit{Geophysics} 66(1):50--53.
\bibitem[Luo et al.(2018)]{luo2018seismic} Luo, X., Bhakta, T., N{\ae}vdal, G. (2018). Correlation-based adaptive localization with applications to ensemble-based 4D-seismic history matching. \textit{SPE Journal} 23(2):396--427.
\bibitem[MacBeth(2004)]{macbeth2004} MacBeth, C. (2004). A classification for the pressure-sensitivity properties of a sandstone rock frame. \textit{Geophysics} 69(2):497--510.
\bibitem[Mahalanobis(1936)]{mahalanobis1936} Mahalanobis, P.C. (1936). On the generalized distance in statistics. \textit{Proceedings of the National Institute of Sciences of India} 2(1):49--55.
\bibitem[Makhoul(1980)]{makhoul1980} Makhoul, J. (1980). A fast cosine transform in one and two dimensions. \textit{IEEE Trans.\ ASSP} 28(1):27--34.
\bibitem[Marzouk et al.(2016)]{marzouk2016transport} Marzouk, Y., Moselhy, T., Parno, M., Spantini, A. (2016). Sampling via measure transport: an introduction. In \textit{Handbook of Uncertainty Quantification}, pp.\ 1--41. Springer.
\bibitem[Mavko et al.(2009)]{mavko2009rock} Mavko, G., Mukerji, T., Dvorkin, J. (2009). \textit{The Rock Physics Handbook: Tools for Seismic Analysis of Porous Media}, 2nd ed. Cambridge Univ.\ Press.
\bibitem[McCabe et al.(2023)]{mccabe2023multiple} McCabe, M., R\'egaldo-Saint~Blancard, B., Parker, L.H., et al. (2023). Multiple physics pretraining for physical surrogate models. \textit{arXiv:2310.02994}.
\bibitem[Millevoi et al.(2025)]{millevoi2025seismic} Millevoi, C., Zoccarato, C., Ferronato, M. (2025). A deep learning-based surrogate model for seismic data assimilation in fault activation modeling. \textit{International Journal for Numerical Methods in Engineering}. doi:10.1002/nme.70040. arXiv:2409.01215.
\bibitem[Mosser et al.(2020)]{mosser2020stochastic} Mosser, L., Dubrule, O., Blunt, M.J. (2020). Stochastic seismic waveform inversion using generative adversarial networks as a geological prior. \textit{Mathematical Geosciences} 52:53--79.
\bibitem[PhysicsNeMo Contributors(2023)]{physicsnemo} PhysicsNeMo Contributors (2023). NVIDIA PhysicsNeMo: an open-source framework for physics-based deep learning in science and engineering. \url{https://github.com/NVIDIA/physicsnemo}.
\bibitem[Oliver et al.(2008)]{oliver2008inverse} Oliver, D.S., Reynolds, A.C., Liu, N. (2008). \textit{Inverse Theory for Petroleum Reservoir Characterization and History Matching}. Cambridge Univ.\ Press.
\bibitem[Pascanu et al.(2013)]{pascanu2013difficulty} Pascanu, R., Mikolov, T., Bengio, Y. (2013). On the difficulty of training RNNs. \textit{ICML}, 1310--1318.
\bibitem[Pathak et al.(2022)]{pathak2022fourcastnet} Pathak, J., Subramanian, S., Harrington, P., et al. (2022). FourCastNet: A global data-driven high-resolution weather model using adaptive Fourier neural operators. \textit{arXiv:2202.11214}.
\bibitem[Peaceman(1978)]{peaceman1978} Peaceman, D.W. (1978). Interpretation of well-block pressures in numerical reservoir simulation. \textit{SPE J.} 18(3):183--194.
\bibitem[Pfaff et al.(2021)]{pfaff2021learning} Pfaff, T., Fortunato, M., Sanchez-Gonzalez, A., Battaglia, P.W. (2021). Learning mesh-based simulation with graph networks. \textit{ICLR}.
\bibitem[Rahman et al.(2023)]{rahman2023uno} Rahman, M.A., Ross, Z.E., Azizzadenesheli, K. (2023). U-NO: U-shaped neural operators. \textit{Trans.\ Mach.\ Learn.\ Res.}
\bibitem[Raissi et al.(2019)]{raissi2019physics} Raissi, M., Perdikaris, P., Karniadakis, G.E. (2019). Physics-informed neural networks. \textit{J.\ Comput.\ Phys.} 378:686--707.
\bibitem[Ranzato et al.(2016)]{ranzato2016sequence} Ranzato, M., Chopra, S., Auli, M., Zaremba, W. (2016). Sequence level training with recurrent neural networks. \textit{ICLR}.
\bibitem[Reich and Cotter(2015)]{reich2015probabilistic} Reich, S., Cotter, C. (2015). \textit{Probabilistic Forecasting and Bayesian Data Assimilation}. Cambridge Univ.\ Press.
\bibitem[Ross et al.(2011)]{ross2011dagger} Ross, S., Gordon, G.J., Bagnell, J.A. (2011). A reduction of imitation learning and structured prediction to no-regret online learning. \textit{AISTATS}, 627--635.
\bibitem[Rozet and Louppe(2023)]{rozet2023scoreda} Rozet, F., Louppe, G. (2023). Score-based data assimilation. \textit{NeurIPS} 36.
\bibitem[Sanchez-Gonzalez et al.(2020)]{sanchez2020learning} Sanchez-Gonzalez, A., Godwin, J., Pfaff, T., Ying, R., Leskovec, J., Battaglia, P.W. (2020). Learning to simulate complex physics with graph networks. \textit{ICML}, 8459--8468.
\bibitem[Schillings and Stuart(2017)]{schillings2017convergence} Schillings, C., Stuart, A.M. (2017). Analysis of the ensemble Kalman filter for inverse problems. \textit{SIAM J.\ Numer.\ Anal.} 55(3):1264--1290.
\bibitem[Schillings and Stuart(2018)]{schillings2018convergence} Schillings, C., Stuart, A.M. (2018). Convergence analysis of ensemble Kalman inversion: the linear, noisy case. \textit{Applicable Analysis} 97(1):107--123.
\bibitem[Smith et al.(2003)]{smith2003gassmann} Smith, T.M., Sondergeld, C.H., Rai, C.S. (2003). Gassmann fluid substitutions: a tutorial. \textit{Geophysics} 68(2):430--440. doi:10.1190/1.1567211.
\bibitem[Song et al.(2021)]{song2021sde} Song, Y., Sohl-Dickstein, J., Kingma, D.P., Kumar, A., Ermon, S., Poole, B. (2021). Score-based generative modeling through stochastic differential equations. \textit{ICLR}.
\bibitem[Stachenfeld et al.(2022)]{stachenfeld2022learned} Stachenfeld, K., Fielding, D.B., Kochkov, D., et al. (2022). Learned simulators for turbulence. \textit{ICLR}.
\bibitem[Stuart(2010)]{stuart2010inverse} Stuart, A.M. (2010). Inverse problems: a Bayesian perspective. \textit{Acta Numerica} 19:451--559.
\bibitem[Tallec and Ollivier(2017)]{tallec2017unbiased} Tallec, C., Ollivier, Y. (2017). Unbiasing truncated backpropagation through time. \textit{arXiv:1705.08209}.
\bibitem[Tang et al.(2020)]{tang2020deep} Tang, M., Liu, Y., Durlofsky, L.J. (2020). A deep-learning-based surrogate model for data assimilation in dynamic subsurface flow problems. \textit{J.\ Comput.\ Phys.} 413:109456.
\bibitem[Trani et al.(2013)]{trani2013seismic} Trani, M., Arts, R., Leeuwenburgh, O. (2013). Seismic history matching of fluid fronts using the ensemble Kalman filter. \textit{SPE Journal} 18(1):159--171.
\bibitem[Tran et al.(2023)]{tran2023factorized} Tran, A., Mathews, A., Xie, L., Ong, C.S. (2023). Factorized Fourier neural operators. \textit{ICLR}.
\bibitem[Tripura and Chakraborty(2023)]{tripura2023wavelet} Tripura, T., Chakraborty, S. (2023). Wavelet neural operator for solving parametric PDEs. \textit{Comput.\ Methods Appl.\ Mech.\ Eng.} 404:115783.
\bibitem[Venkatraman et al.(2015)]{venkatraman2015improving} Venkatraman, A., Hebert, M., Bagnell, J.A. (2015). Improving multi-step prediction of learned time series models. \textit{AAAI}, 3024--3030.
\bibitem[Wang et al.(2021)]{wang2021theory} Wang, S., Wang, H., Perdikaris, P. (2021). Learning the solution operator of parametric PDEs with physics-informed DeepONets. \textit{Science Advances} 7(40):eabi8605.
\bibitem[Wang and Durlofsky(2024)]{wang2024co2} Wang, N., Durlofsky, L.J. (2024). Deep learning framework for history matching CO$_2$ storage with 4D seismic and monitoring well data. \textit{arXiv preprint} arXiv:2408.01575.
\bibitem[Wen et al.(2022)]{wen2022ufno} Wen, G., Li, Z., Azizzadenesheli, K., Anandkumar, A., Benson, S.M. (2022). U-FNO: An enhanced Fourier neural operator-based deep-learning model for multiphase flow. \textit{Adv.\ Water Resour.} 163:104180.
\bibitem[Wen et al.(2023)]{wen2023nested} Wen, G., Li, Z., Long, Q., Azizzadenesheli, K., Anandkumar, A., Benson, S.M. (2023). Real-time high-resolution CO$_2$ geological storage prediction using nested Fourier neural operators. \textit{Energy \& Environmental Science} 16(4):1732--1741.
\bibitem[Williams and Zipser(1989)]{williams1990efficient} Williams, R.J., Zipser, D. (1989). A learning algorithm for continually running fully recurrent neural networks. \textit{Neural Computation} 1(2):270--280.
\bibitem[Zhang et al.(2022)]{zhang2022fourier} Zhang, K., Zuo, Y., Zhao, H., et al. (2022). Fourier neural operator for solving subsurface oil/water two-phase flow in reservoir engineering. \textit{arXiv:2207.03123}.
\end{thebibliography}
\end{document}